\documentclass[twoside]{article}

\usepackage[accepted]{aistats2024}
\usepackage[round]{natbib}

\usepackage[utf8]{inputenc} %
\usepackage[T1]{fontenc}    %
\usepackage{hyperref}       %
\usepackage{url}            %
\usepackage{booktabs}       %
\usepackage{amsfonts}       %
\usepackage{nicefrac}       %
\usepackage{microtype}      %
\usepackage{xcolor}         %

\definecolor{linkcolor}{RGB}{83,83,182}
\hypersetup{
    colorlinks=true,
    citecolor=linkcolor,
    linkcolor=linkcolor
}

\newcommand{\bfu}{\mathbf{u}}
\newcommand{\bfdelta}{\mathbf{\Delta}}

\usepackage{amsmath}
\usepackage{algorithm}
\usepackage{algorithmic}
\usepackage{empheq}
\usepackage{amssymb}
\usepackage{mathtools}
\usepackage{amsthm}
\usepackage{colortbl}
\usepackage{thmtools}
\usepackage[capitalize,noabbrev]{cleveref}
\usepackage{hhline}

\newcommand{\bbE}{\mathbb{E}}

\newcommand{\bbN}{\mathbb{N}}

\newcommand{\bbR}{\mathbb{R}}

\DeclareMathOperator*{\argmin}{arg\,min}

\newcommand{\calA}{\mathcal{A}}

\newcommand{\calC}{\mathcal{C}}
\newcommand{\calD}{\mathcal{D}}

\newcommand{\calG}{\mathcal{G}}

\newcommand{\calI}{\mathcal{I}}

\newcommand{\calL}{\mathcal{L}}

\newcommand{\calN}{\mathcal{N}}
\newcommand{\calO}{\mathcal{O}}

\newcommand{\calV}{\mathcal{V}}

\newcommand{\train}{{\mathrm{train}}}

\newcommand{\val}{{\mathrm{val}}}

\newcommand{\diff}{\mathrm{d}}
\newcommand{\e}{\mathrm{e}}

\newcommand{\setcomb}[1]{[#1]}

\newcommand{\Span}{\mathrm{Span}}

\theoremstyle{plain}
\newtheorem{theorem}{Theorem}[section]
\newtheorem{proposition}[theorem]{Proposition}
\newtheorem{lemma}[theorem]{Lemma}
\newtheorem{corollary}[theorem]{Corollary}
\theoremstyle{definition}
\newtheorem{definition}[theorem]{Definition}
\newtheorem{assumption}[theorem]{Assumption}
\theoremstyle{remark}

\usepackage{thmtools}
\usepackage{thm-restate}

\declaretheoremstyle[
    shaded={bgcolor=\color{HTML}{E9F6FF}}
]{shadedtheorem}

\newcommand{\checklistans}[1]{\textbf{\textcolor{blue}{#1}}}

\newcommand{\new}[1]{#1}

\usepackage[textsize=tiny]{todonotes}

\begin{document}

\runningtitle{A Lower Bound and a Near-Optimal Algorithm for Bilevel Empirical Risk Minimization}
\twocolumn[

\aistatstitle{A Lower Bound and a Near-Optimal Algorithm \\
for Bilevel Empirical Risk Minimization}

\aistatsauthor{Mathieu Dagréou \And Thomas Moreau \And  Samuel Vaiter \And Pierre Ablin}

\aistatsaddress{Université Paris-Saclay \\ Inria, CEA \\ Palaiseau, France \And Université Paris-Saclay \\ Inria, CEA \\ Palaiseau, France \And Université de Côte d'Azur \\ CNRS, LJAD \\ Nice, France \And Apple \\ Paris, France} ]

\begin{abstract}
    Bilevel optimization problems, which are problems where two optimization problems are nested, have more and more applications in machine learning.
    In many practical cases, the upper and the lower objectives correspond to empirical risk minimization problems and therefore have a sum structure.
    In this context, we propose a bilevel extension of the celebrated SARAH algorithm.
    We demonstrate that the algorithm requires $\mathcal{O}((n+m)^{\frac12}\varepsilon^{-1})$ oracle calls to achieve $\varepsilon$-stationarity with $n+m$ the total number of samples, which improves over all previous bilevel algorithms.
    Moreover, we provide a lower bound on the number of oracle calls required to get an approximate stationary point of the objective function of the bilevel problem.
    This lower bound is attained by our algorithm, making it optimal in terms of sample complexity.
\end{abstract}

\section{Introduction}
\label{sec:intro}

\looseness=-1
In the last few years, bilevel optimization has become an essential tool for the machine learning community thanks to its numerous applications. Among them, we can cite hyperparameter selection~\citep{Bengio2000, Pedregosa2016hoag, Franceschi2017ForwardReverse, Lorraine2020MillionsofHyperparameters}, implicit deep learning \citep{Bai2019}, neural architecture search \citep{Liu2019darts, Zhang2021idarts}, data augmentation~\citep{Li2020dada, Rommel2022} or meta-learning \citep{Franceschi2018hyperparameter, Rajeswaran2019imaml}. Bilevel optimization consists in minimizing a function under the constraint that one variable minimizes another function. This can be formalized as follows
\begin{align}\label{eq:bilevel_pb}
    \begin{split}
        &\min_{x\in\bbR^d} h(x) = F(z^*(x), x)\enspace,\\
        &\text{subject to }z^*(x) \in \argmin_{z\in\bbR^p} G(z,x)\enspace .
    \end{split}
\end{align}
The function $F$ is called the outer function and the function $G$ is the inner function. Likewise, we refer to $x$ as the outer variable and $z$ as the inner variable.
\new{The function $h$ is the value function and it can be minimized using gradient descent. To compute its gradient, we use implicit differentiation which yields}
\begin{equation}\label{eq:grad_h}
    \nabla h(x) = \nabla_2 F(z^*(x),x) + \nabla_{21}^2 G(z^*(x),x)v^*(x)
\end{equation}
where $v^*(x)$ is the solution of a linear system
\begin{equation}\label{eq:lin_syst}
    v^*(x) = -\left[\nabla^2_{11}G(z^*(x),x)\right]^{-1}\nabla_1F(z^*(x),x)\enspace.
\end{equation}
When we have exact access to $z^*(x)$, solving \eqref{eq:bilevel_pb} boils down to a smooth nonconvex optimization problem which can be solved using solvers for single-level problems. However, computing exactly $z^*(x)$ and $v^*(x)$ is often too costly, and implicit differentiation-based algorithms rely on approximations of $z^*(x)$ and $v^*(x)$ rather than their exact value. Depending on the precision of the different approximations, we are not ensured that the approximate gradient used is a descent direction.
Results by \citet{Pedregosa2016hoag} characterized the approximation quality for $z^*(x)$ and $v^*(x)$ required to ensure convergence, opening the door to various algorithms to solve bilevel optimization problems~\citep{Lorraine2020MillionsofHyperparameters, Ramzi2022}.

In many applications of interest, the functions $F$ and $G$ correspond to Empirical Risk Minimization (ERM), and as a consequence have a finite sum structure
\begin{equation*}
    F(z,x) = \frac1m\sum_{j=1}^m F_j(z,x),\quad G(z,x) = \frac1n\sum_{i=1}^n G_i(z,x)\enspace.
\end{equation*}
For instance, in hyperparameter selection, $F$ is the validation loss which is an average on the validation set and $G$ is the training loss which is an average on the training set. In single-level optimization, the finite sum structure has been widely leveraged to produce fast first-order algorithms that provably converge faster than gradient descent.
Among them, we can cite stochastic methods such as stochastic gradient descent (SGD) \citep{Robbins1951, Bottou2010} and its variance-reduced variants such as SAGA~\citep{Defazio2014}, STORM \citep{Cutkosky2019} or SPIDER/SARAH \citep{Fang2018, Nguyen2017} that use only a handful of samples at a time to make progress.
To get faster methods than full-batch approaches, it is natural to extend these methods to the bilevel setting.
The main obstacle comes from the difficulty of obtaining stochastic approximations of $\nabla h(x)$ because of its structure \eqref{eq:grad_h} \new{which involves a Hessian inversion}.
Several strategies have been proposed to overcome this obstacle, and some works demonstrate that stochastic implicit differentiation-based algorithms for solving \eqref{eq:bilevel_pb}  have the same complexity as single-level analogous algorithms.
For instance, ALSET from \citet{Chen2021alset} and SOBA from \cite{Dagreou2022SABA} have the same convergence rate as SGD for nonconvex single-level problems \citep{Ghadimi2013, Bottou2018}.
Also, \citet{Dagreou2022SABA} show that SABA, an adaptation of SAGA \citep{Defazio2014}, has an analogous sample complexity to its single-level counterparts for nonconvex problems \citep{Reddi2016saga}.

Yet, in classical single-level optimization, it is known that neither of these algorithms is optimal: the SARAH algorithm~\citep{Nguyen2017} achieves a better sample complexity of $\calO(m^{\frac12}\varepsilon^{-1})$ with $m$ the number of samples.
Furthermore, this algorithm is \emph{near-optimal} (\textit{i.e.} optimal up to constant factors) because the lower bound for single-level nonconvex optimization is $\Omega(m^{\frac12}\varepsilon^{-1})$ as proved by \citet{Zhou2019}.
It is natural to ask if we can extend these results to bilevel optimization.

\textbf{Contributions}~ In \cref{sec:algo}, we introduce SRBA, an adaptation of the SARAH algorithm to the bilevel setting.
We then demonstrate in \cref{sec:theory} that, similarly to the single-level setting, $\calO\left((n+m)^\frac12\varepsilon^{-1}\vee (n+m)\right)$ oracle calls are sufficient to reach an $\varepsilon$-stationary point.
As shown in \cref{table:complexities}, it achieves the best-known complexity in the regime $n+m \lesssim \calO(\varepsilon^{-2})$.
In \cref{sec:lower_bound}, we analyze the lower bounds for such problems.
We show that we need at least $\Omega(m^{\frac12}\varepsilon^{-1})$ oracle calls to reach an $\varepsilon$-stationary point (see \cref{def:stationary_point}), hereby matching the previous upper-bound in the case where $n \asymp m$ and $\varepsilon\leq m^{-\frac12}$. 
SRBA is thus near-optimal in that regime.
Even though our main contribution is theoretical, we illustrate the numerical performances of the algorithm in \cref{sec:expe}.

\vspace{-.5em}
\begin{table*}[ht]
    \centering
    \begin{tabular}{c|c|c|c|c|}
        \hhline{~----}
                             & \cellcolor{gray!20} Sample complexity  &\cellcolor{gray!20} Stochastic setting &\cellcolor{gray!20}$F$ & \cellcolor{gray!20}$G$                      \\ \hline
        \multicolumn{1}{|c|}{StocBiO \citep{Ji2021stocbio} }      & $\tilde{\calO}(\varepsilon^{-2})$  & General expectation      & $\calC^{1,1}_L$ & SC and $\calC^{2,2}_L$       \\ \hline
        \multicolumn{1}{|c|}{AmIGO \citep{Arbel2022amigo}}         & $\calO(\varepsilon^{-2})$  & General expectation &$\calC^{1,1}_L$& SC and $\calC^{2,2}_L$    \\ \hline
        \multicolumn{1}{|c|}{MRBO \citep{Yang2021a}}          & $\tilde{\calO}(\varepsilon^{-\frac32})$  & General expectation    &$\calC^{1,1}_L$& SC and $\calC^{2,2}_L$      \\ \hline
        \multicolumn{1}{|c|}{VRBO \citep{Yang2021a}}          & $\tilde{\calO}(\varepsilon^{-\frac32})$  & General expectation  &$\calC^{1,1}_L$& SC and $\calC^{2,2}_L$       \\ \hline
        \multicolumn{1}{|c|}{SABA \citep{Dagreou2022SABA} }         & $\calO((n+m)^{\frac23}\varepsilon^{-1})$  & Finite sum   &$\calC^{2,2}_L$& SC and $\calC^{3,3}_L$       \\ \hline
        \multicolumn{1}{|c|}{\new{F$^2$SA \citep{Kwon2023f2sa} }}         & $\calO(\varepsilon^{-\frac72})$  & General expectation  &$\calC^{2,2}_L$& SC and $\calC^{2,2}_L$       \\ \hline
        \rowcolor{blue!10}
        \multicolumn{1}{|c|}{\textbf{SRBA}} & \textbf{$\calO((n+m)^{\frac12}\varepsilon^{-1})$}  & Finite sum  &$\calC^{2,2}_L$& SC and $\calC^{3,3}_L$\\ \hline
    \end{tabular}
    \vspace{.5em}
    \caption{Comparison between the sample complexities and the Assumptions of some stochastic bilevel solvers. It corresponds to the number of calls to gradient, Hessian-vector products, and Jacobian-vector product sufficient to get an $\varepsilon$-stationary point. The tilde on the $\tilde{\calO}$ hide a factor $\log(\varepsilon^{-1})$. "SC" means "strongly convex". $\calC^{p,k}_L$ means $p$-times differentiable with Lipschitz $k$th order derivatives for $k\leq p$.}
    \label{table:complexities}
    \vspace{-1.5em}
\end{table*}

\textbf{Related work}~
There are several strategies to solve \eqref{eq:bilevel_pb} with a stochastic method.
\new{The first one is the value-function-based method which consists in recasting Problem \ref{eq:bilevel_pb} as a single-level constrained optimization problem as done with F$^2$SA \citep{Kwon2023f2sa} or BOME \citep{Ye2022}.
}
\new{The second way is to use first-order methods on $h$ with \emph{approximate} gradients.
The approximate gradient of $h$ can be estimated using two approaches}: iterative differentiation (ITD) and approximate implicit differentiation (AID).
On the one hand, in ITD algorithms, the Jacobian of $z^*$ is estimated by differentiating the different steps used to compute an approximation of $z^*$.
On the other hand, AID algorithms leverage the implicit gradient given by \eqref{eq:grad_h} replacing $z^*$ and $v^*$ by some approximations $z$ and $v$.
In the class of ITD algorithms, \citet{Maclaurin2015} propose to approximate the Jacobian of the solution of the inner problem by differentiating through the iterations of SGD with momentum.
The complexity of the hypergradient computation in ITD solvers is studied in \cite{Franceschi2017ForwardReverse, Grazzi2020, Ablin2020}.
For AID algorithms, \citet{Ghadimi2018, Chen2021alset, Ji2021stocbio} propose to perform several SGD steps in the inner problem and then use Neumann approximations to approximate $v^*(x)$ defined in~\eqref{eq:lin_syst}.
A method consisting of alternating steps in the inner and outer variables was proposed in~\cite{Hong2021}.
These methods can be improved by using a warm start strategy for the inner problem \citep{Ji2021stocbio, Chen2021alset} and for the linear system \citep{Arbel2022amigo}.
Some works adapt variance reduction methods to like STORM \citep{Cutkosky2019, Khanduri2021sustain, Yang2021a} or SAGA \citep{Defazio2014,Dagreou2022SABA}.
We take a similar approach and extend the SARAH variance reduction method to the bilevel setting.
Recent works propose to approximate the Jacobian of $z^*$ by stochastic finite difference \citep{Sow2022a} or to use Bregman divergence-based methods \citep{Huang2022BilevelBregmana}. 

In single-level optimization, the problem of finding complexity lower bounds has been widely studied since the seminal work of \citet{Nemirovsky1983}. 
On the one hand, \citet{Agarwal2015} provided a lower bound to minimize strongly convex and smooth finite sums with deterministic algorithms that have access to individual gradients.
These results were extended to randomized algorithms for (strongly) convex finite sum objective by \citet{Woodworth2016}. 
On the other hand, \citet{Carmon2017lowerbound1} provided a lower bound for minimizing nonconvex functions with deterministic and randomized algorithms.
The nonconvex finite sum case is treated in \cite{Fang2018, Zhou2019}.
In the bilevel case, \citet{Ji2023accBio} showed a lower bound for deterministic algorithms.
However, this result is restricted to the case where the value function $h$ is convex or strongly convex, which is not the case with most ML-related bilevel problems.
Our results are instead in a nonconvex setting.

\textbf{Notation}~
The quantity $A_\bullet$ refers to $A_z$, $A_v$, or $A_x$, depending on the context.
If $f:\bbR^p\times\bbR^d\to\bbR$ is a twice differentiable function, we denote $\nabla_i f(z,x)$ its gradient w.r.t. its $i^\text{th}$ variable.
Its Hessian with respect to $z$ is denoted $\nabla_{11}^2f(z,x)~\in~\bbR^{p\times p}$ and its cross derivative matrix $\left(\frac{\partial^2 f}{\partial z_i\partial x_j}\right)_{\substack{i\in\setcomb{p} \\j\in\setcomb{d}}}$ is denoted $\nabla^2_{12} f(z,x)~\in~\bbR^{p\times d}$. We denote $\Pi_\calC$ the projection on a closed convex set $\calC$.

\section{SRBA: a Near-Optimal Algorithm for Bilevel Empirical Risk Minimization}
\label{sec:algo}

In this section, we introduce SRBA (Stochastic Recursive Bilevel Algorithm), a novel algorithm for bilevel empirical risk minimization which is provably near-optimal for this problem. This algorithm is inspired by the algorithms SPIDER \citep{Fang2018} and SARAH \citep{Nguyen2017, Nguyen2022} which are known for being near-optimal algorithms for nonconvex finite sum minimization problems. It relies on a recursive estimation of directions of interest, which is restarted periodically. Proofs are deferred to the appendix.

\subsection{Assumptions}\label{subsec:ass}
Before presenting our algorithm, we formulate regularity Assumptions on the functions $F$ and $G$.
\begin{assumption}\label{ass:regul_F}
    For all $j\in\setcomb{m}$, $F_j$ is twice differentiable and $L^F_0$-Lipschitz continuous. Its gradient is $L^F_1$-Lipschitz continuous and its Hessian is $L^F_2$-Lipschitz continuous.
\end{assumption}

\begin{assumption}\label{ass:regul_G}
    For all $i\in\setcomb{n}$, $G_i$ is three times differentiable. Its first, second, and third order derivatives are respectively $L^G_1$-Lipschitz continuous, $L^G_2$-Lipschitz continuous, and $L^G_3$-Lipschitz continuous. For $x\in\bbR^d$, the function $G_i(\,.\,,x)$ is $\mu_G$-strongly convex.
\end{assumption}
The strong convexity and the smoothness with respect to $z$ hold for instance when we consider an $\ell^2$-regularized logistic regression problem with non-separable data.
These regularity assumptions up to first-order for $F$ and second-order for $G$ are standard in the stochastic bilevel literature \citep{Arbel2022amigo, Ji2021stocbio, Yang2021a}.
The second-order regularity for $F$ and third-order regularity for $G$ are necessary for the analysis of the dynamics of the auxiliary variable $v$ we introduce in \cref{ssec:hyper_grad_approx}, as is the case in \cite{Dagreou2022SABA}.
As shown in \citet[Lemma 2.2]{Ghadimi2018}, these assumptions imply the smoothness of $h$, which is a fundamental property to get a descent.
\begin{proposition}\label{prop:smoothness_h}
Under Assumptions \ref{ass:regul_F} and \ref{ass:regul_G}, the function $h$ is $L^h$ smooth for some $L^h>0$ which is precised in \cref{app:sec:smoothness_h}.
\end{proposition}

Another consequence of Assumptions \ref{ass:regul_F} and \ref{ass:regul_G} is the boundedness of the function $v^*$.
\begin{proposition}\label{prop:v_star_boundedness}
    Assume that Assumptions \ref{ass:regul_F} and \ref{ass:regul_G} hold. Then, for $R = \frac{L^F_0}{\mu_G}$ it holds that for any $x\in\bbR^d$, we have $\|v^*(x)\|\leq R$.
\end{proposition}
We denote $\Gamma$ the closed ball centered in $0$ with radius $R$ and $\Pi_\Gamma$ the projection onto $\Gamma$. For $(z, v, x)\in\bbR^p\times\bbR^p\times\bbR^d$, we denote $\Pi(z,v,x) = (z, \Pi_\Gamma(v), x)$.

\subsection{Hypergradient Approximation}\label{ssec:hyper_grad_approx}
The gradient of $h$ given by \eqref{eq:grad_h} is intractable in practice because it requires the perfect knowledge of $z^*(x)$ and $v^*(x)$ which are usually costly to compute.
As classically done in the stochastic bilevel literature \citep{Ji2021stocbio, Arbel2022amigo, Li2022}, $z^*(x)$ and $v^*(x)$ are replaced by approximate surrogate variables $z$ and $v$. The variable $z$ is typically the output of one or several steps of an optimization procedure applied to $G(\,.\,, x)$. The variable $v$ can be computed by using Neumann approximations or doing some optimization steps on the quadratic $v\mapsto \frac12 v^\top \nabla^2_{11}G(z,x) v + \nabla_1 F(z,x)^\top v$. We consider the approximate hypergradient given by 
\begin{equation*}
    D_x(z, v, x) = \nabla^2_{21}G(z,x) v + \nabla_2 F(z,x)\enspace.
\end{equation*}
The motivation behind this direction is that if we take $z=z^*(x)$ and $v = v^*(x)$, we recover the true gradient, that is $D_x(z^*(x), v^*(x), x)=~\nabla h(x)$.
\cref{prop:error_gradient} from \citep[Lemma~3.4]{Dagreou2022SABA} controls the hypergradient approximation error by the distances between $z$ and $z^*(x)$ and between $v$ and $v^*(x)$.
\begin{proposition}\label{prop:error_gradient}
    Let $x~\in~\bbR^d$. Assume that $F$ is differentiable and $L^F_1$ smooth with bounded gradient, $G$ is twice differentiable with Lipschitz gradient and Hessian and $G(\,.\,, x)$ is $\mu_G$-strongly convex.
    Then there exists a constant $L_x$ such that
    $$
    \|D_x(z,v,x) - \nabla h(x)\|^2 \!\leq\! L_x^2(\|z-z^*(x)\|^2 + \|v - v^*(x)\|^2) .
    $$
\end{proposition}
Thus, it is natural to make $z$ and $v$ move towards their respective equilibrium values which are given by $z^*(x)$ and $v^*(x)$.
As a consequence, we also introduce the directions $D_z$ and $D_x$ as follows
\begin{align*}
    D_z(z,v,x) &= \nabla_1 G(z,x)\enspace,\\
    D_v(z,v,x) &= \nabla_{11}^2 G(z,x)v + \nabla_1 F(z,x)\enspace.
\end{align*}
The interest of considering the directions $D_z$ and $D_v$ is expressed in \cref{prop:directions_cancels}.
\begin{proposition}
    \label{prop:directions_cancels}
    Assume that $G$ is strongly convex with respect to its first variable. Then for any $x\in\bbR^d$, it holds $D_z(z^*(x), v^*(x), x) = 0$ and $D_v(z^*(x), v^*(x), x) = 0$.
\end{proposition}
The directions $D_z$, $D_v$, and $D_x$ can be written as sums over the samples. Hence,
following these directions enables to adapt any classical algorithm suited for single-level finite sum minimization to bilevel finite sum minimization.
In what follows, for two indices $i\in\setcomb{n}$ and $j\in\setcomb{m}$, we consider the sampled directions $D_{z,i,j}$, $D_{v,i,j}$ and $D_{x,i,j}$ defined by
\begin{align}
    \label{eq:sampled_dz}D_{z,i,j}(z,v,x) &= \nabla_1 G_i(z,x)\\
    \label{eq:sampled_dv}D_{v,i,j}(z,v,x) &= \nabla_{11}^2 G_i(z,x)v + \nabla_1 F_j(z,x) \\
    \label{eq:sampled_dx}D_{x,i,j}(z,v,x) &= \nabla_{21}^2 G_i(z,x)v + \nabla_2 F_j(z,x)\enspace .
\end{align}
When $i$ and $j$ are randomly sampled uniformly, these directions are unbiased estimators of the true directions $D_z$, $D_v$, and $D_x$. Yet, as in~\cite{Nguyen2017}, we use them to recursively build biased estimators of the directions that enable fast convergence.
\subsection{SRBA: Stochastic Recursive Bilevel Algorithm}
In \cref{alg:srba}, we present SRBA, a combination of the idea of recursive gradient coming from \citep{Fang2018, Nguyen2022} and the framework proposed in \citep{Dagreou2022SABA}. 
It relies on a recursive estimation of each direction $D_z$, $D_v$, $D_x$ which is updated following the same strategy as SARAH.
Let us denote by $\rho$ the step size of the update for the variables $z$ and $v$, and $\gamma$ the step size for the update of the variable $x$.
We use the same step size for $z$ and $v$ because the problems of minimizing the inner function $G$ and solving the linear system \eqref{eq:lin_syst} have the same conditioning driven by $\nabla^2_{11}G$.
For simplicity, we denote the joint variable $\bfu = (z, v, x)$ and the joint directions weighted by the step sizes $\bfdelta~=~(\rho D_z, \rho D_v, \gamma D_x) =(\bfdelta_z,\bfdelta_v, \bfdelta_x)$.

At iteration $t$, the estimate direction $\bfdelta$ is initialized by computing full batch directions:
\begin{align*}
    \bfdelta^{t,0} = (\rho D_z(\mathbf{\tilde u}^t), \rho D_v(\mathbf{\tilde u}^t), \gamma D_x(\mathbf{\tilde u}^t))
\end{align*}
and a first update is performed by moving from $\mathbf{\tilde u}^t$ in the direction $-\bfdelta^{t,0}$. As done in \cite{Hu2022MultiBlock}, we project the variable $v$ onto $\Gamma$ to leverage the boundedness property of $v^*$.
Then, during the $k$th iteration of an inner loop of size $q-1$, two indices $i\in\setcomb{n}$ and $j\in\setcomb{m}$ are sampled and the estimate directions are updated according to Equations~\eqref{eq:update_dz} to \eqref{eq:update_dx}
\begin{align}
    \label{eq:update_dz}\mathbf{\Delta}_z^{t,k} &= \rho (D_{z,i,j}(\mathbf{u}^{t,k}) -  D_{z,i,j}(\mathbf{u}^{t,k-1})) + \mathbf{\Delta}_z^{t,k-1}\\
    \label{eq:update_dv}\mathbf{\Delta}_v^{t,k} &= \rho (D_{v,i,j}(\mathbf{u}^{t,k}) -  D_{v,i,j}(\mathbf{u}^{t,k-1})) + \mathbf{\Delta}_v^{t ,k-1}\\
    \label{eq:update_dx}\mathbf{\Delta}_x^{t,k} &= \gamma (D_{x,i,j}(\mathbf{u}^{t,k}) - D_{x,i,j}(\mathbf{u}^{t,k-1})) + \mathbf{\Delta}_x^{t,k-1\!}
\end{align}
where the sampled directions $D_{z,i,j}$, $D_{v,i,j}$ and $D_{x,i,j}$ are defined by Equations \eqref{eq:sampled_dz} to \eqref{eq:sampled_dx}. 
        \begin{algorithm}[tb]
            \begin{algorithmic}
               \STATE {\bfseries Input:} initializations $z_0\in\bbR^p$, $x_0\in\bbR^d$, $v_0\in\bbR^p$, number of iterations $T$ and $q$, step sizes $\rho$ and $\gamma$.
        
               Set $\mathbf{\tilde{u}^0} = (z_0, v_0, x_0)$
               \FOR{$t = 0,\dots, T-1$}
                   \STATE
                    Reset $\mathbf{\Delta}$: $\mathbf{\Delta}^{t,0} = (\rho D_z(\mathbf{\tilde u}^t), \rho D_v(\mathbf{\tilde u}^t), \gamma D_x(\mathbf{\tilde u}^t))$
        
                    Update $\mathbf{u}$: $\mathbf{u}^{t,1} = \Pi(\mathbf{\tilde u}^t - \mathbf{\Delta^{t,0}})\enspace,$
        
               \FOR{$k = 1,\dots, q-1$}
                \STATE
                    Draw $i\in\{1,\dots, n\}$ and $j\in\{1,\dots,m\}$
                    
                    $\mathbf{\Delta}_z^{t,k} = \rho (D_{z,i,j}(\mathbf{u}^{t,k}) - D_{z,i,j}(\mathbf{u}^{t,k-1})) + \mathbf{\Delta}_z^{t,k-1} $
        
                    $\mathbf{\Delta}_v^{t,k} = \rho (D_{v,i,j}(\mathbf{u}^{t,k}) - D_{v,i,j}(\mathbf{u}^{t,k-1})) + \mathbf{\Delta}_v^{t ,k-1}$
            
                    $\mathbf{\Delta}_x^{t,k} = \gamma (D_{x,i,j}(\mathbf{u}^{t,k}) - D_{x,i,j}(\mathbf{u}^{t,k-1})) + \mathbf{\Delta}_x^{t,k-1}$
        
                    Update $\mathbf{u}$: $\mathbf{u}^{t, k+1} = \Pi(\mathbf{u}^{t, k} - \mathbf{\Delta}^{t,k})$
               \ENDFOR
        
               Set $\mathbf{\tilde u}^{t+1} = \mathbf{u}^{t+1,q} $
               \ENDFOR
               
               Return $(\tilde{z}^T, \tilde{v}^T, \tilde{x}^T) = \mathbf{\tilde u}^{T}$
            \end{algorithmic}
            \caption{Stochastic Recursive Bilevel Algorithm}
            \label{alg:srba}
        \end{algorithm}
Then the joint variable $\bfu$ is updated by
\begin{equation}\label{eq:u_update}
    \bfu^{t,k+1} = \Pi(\bfu^{t,k} - \bfdelta^{t,k})\enspace.
\end{equation}
Recall that the projection is only performed on the variable $v$. The other variables $z$ and $x$ remain unchanged after the projection step.
At the end of the inner procedure, we set $\mathbf{\tilde u}^{t+1} = \bfu^{t,q}$.

In \cref{alg:srba}, the variables $z$, $v$, and $x$ are updated simultaneously rather than alternatively.
From a computational perspective, this enables sharing the common computations between the different oracles and doing the update of each variable in parallel.
So there is no sub-procedure to approximate the solution of the inner problem and the solution of the linear system.

Note that in \cite{Yang2021a}, the authors propose VRBO, another adaptation of SPIDER/SARAH for bilevel problems.
VRBO has a double loop structure where the inner variable is updated by several steps in an inner loop.
In this inner loop, the estimates of the gradient of $G$ and the gradient of $h$ are also updated using SARAH's update rules.
SRBA has a different structure.
First, in SRBA, the inner variable $z$ is updated only once between two updates of the outer variable instead of several times.
Second, the solution of the linear system evolves following optimization steps whereas in VRBO a Neumann approximation is used.
Moreover, in \cite{Yang2021a}, the algorithm VRBO is analyzed in the case where the functions $F$ and $G$ are general expectations but not in the specific case of empirical risk minimization, as done in \cref{sec:theory}.
Finally, VRBO requires three more parameters than SRBA: the number of inner steps, the number of terms and the scaling parameter in the Neumann approximations.

\section{Theoretical Analysis of SRBA}\label{sec:theory}
In this section we provide the theoretical analysis of \cref{alg:srba} leading to a final sample complexity in $\calO\left((n+m)^\frac12\varepsilon^{-1}\vee (n+m)\right)$. The detailed proofs of the results are deferred to the appendix.
In \cref{def:stationary_point}, we recall what is an $\varepsilon$-stationary point.
\begin{definition}\label{def:stationary_point}
    Let $d$ a positive integer, $f:\bbR^d\to\bbR$ a differentiable function and $\varepsilon>0$.
    We say that a point $x\in\bbR^d$ is an $\varepsilon$-stationary point of $f$ if $\|\nabla f(x)\|^2\leq\epsilon$.
    In a stochastic context, we call $\varepsilon$-stationary point a random variable $x$ such that $\bbE[\|\nabla f(x)\|^2] \leq \varepsilon$.
\end{definition}

In this paper, the theoretical complexity of the algorithms is given in terms of number of calls to oracle, that is to say, the number of times the quantity
\begin{equation}\label{eq:oracles}\hskip-1.5ex
[\nabla F_j(z,x), \nabla G_i(z,x), \nabla^2_{11} G_i(z,x)v, \nabla^2_{21} G_i(z,x)v]
\end{equation}
is queried for $i\in\setcomb{n}$, $j\in\setcomb{m}$, $z\in\bbR^p$, $v\in\bbR^p$ and $x\in\bbR^d$.
Note that in practice, although the second-order derivatives of the inner functions $\nabla^2_{11} G_i(z,x)\in\bbR^{p\times p}$ and $\nabla^2_{21} G_i(z,x)\in\bbR^{d\times p}$ are involved, they are never computed or stored explicitly.
We rather work with Hessian-vector products $\nabla^2_{11} G_i(z,x)v\in\bbR^p$ and Jacobian-vector products $\nabla^2_{21} G_i(z,x)v\in\bbR^d$ which can be computed efficiently thanks to automatic differentiation with a computational cost similar to the cost of computing the gradients $\nabla_1 G_i(z,x)$ and $\nabla_2 G_i(z,x)$~\cite{pearlmutter1994}.
The cost of one query~\eqref{eq:oracles} is therefore of the same order of magnitude as that of computing one stochastic gradient.

\subsection{Mean Squared Error of the Estimated Directions}
A strength of our method is its simple expression of the estimation error of the directions coming from the bias-variance decomposition provided by \citet{Nguyen2017}. Let us denote the estimate directions $D^{t,k}_z~=~\bfdelta_z^{t,k}/\rho$, $D^{t,k}_v = \bfdelta_v^{t,k}/\rho$ and $D^{t,k}_x = \bfdelta_x^{t,k}/\gamma$. We also introduce the residuals
\vspace{-.3em}
\begin{align*}
    S_{\bullet}^{t,k} &=  \sum_{r=1}^{k}\bbE[\|D_\bullet(\bfu^{t,r}) - D_\bullet(\bfu^{t,r-1})\|^2],\\
    \tilde S_{\bullet}^{t,k} &= \sum_{r=1}^{k}\bbE[\|D^{t,r}_\bullet - D^{t,r-1}_\bullet\|^2]\enspace.
\end{align*}
We provide a link between the mean squared error $\bbE[\|D_\bullet^{t,k} - D_\bullet(\bfu^{t,k})\|^2]$ and the residuals.
\begin{proposition}[MSE of the estimate directions]\label{prop:bias_variance_decomp}
For any $t\geq 0$ and $k\in\{1,\dots,q-1\}$, the estimate $D_\bullet^{t,k}$ of the direction $D_\bullet(\bfu^{t,k})$ satisfies
    \begin{align*}
        \bbE[\|D^{t,k}_{\bullet} - D_\bullet(\bfu^{t,k})\|^2] &= \tilde S_{\bullet}^{t,k} - S_{\bullet}^{t,k}\enspace.
    \end{align*}
\end{proposition}
The above error has two components: the accumulation of the difference between two successive full batch directions and the accumulation of the difference between two successive estimate directions. 

\subsection{Fundamental Lemmas}
We establish descent lemmas which are key ingredients to get the final convergence result. \cref{lemma:descent_z} characterizes the dynamic of $\bfu$ on the inner problem. To do so, we define the function $\phi_z$ as 
\begin{equation*} 
    \phi_z(z,x) = G(z,x) - G(z^*(x), x)\enspace.
\end{equation*}
In the bilevel literature, direct control on the distance to optimum $\delta_z^{t,k}\triangleq\bbE[\|z^{t,k} - z^*(x^{t,k})\|^2]$ is established.
Here, the biased nature of the estimate direction $D^{t,k}_z$ makes it hard to upper bound appropriately the scalar product $\langle D_z(\bfu^{t,k}) - D_z^{t,k}, z^{t,k}-z^*(x^{t,k})\rangle$. Therefore, we rather consider $\phi_z^{t,k}$.
By combining the smoothness property of $\phi_z$ and the bias-variance decomposition provided in \cref{prop:bias_variance_decomp}, we can show some descent property on the sequence $\phi_z^{t,k}$ defined by $\phi_z^{t,k} = \bbE[\phi_z(z^{t,k}, x^{t,k})]$. Before stating \cref{lemma:descent_z}, let us define $\calG_v^{t,k} = \frac1\rho\left(v^{t,k} - \Pi_\Gamma(v^{t,k} - \rho D^{t,k}_v)\right)$ so that $v^{t,k+1} = v^{t,k} - \rho\calG^{t,k}_v$. This is the actual update direction of $v$. If there were no projections, we would have $\calG^{t,k}_v = D^{t,k}_v$. Hence, it acts as a surrogate of $D^{t,k}_v$ in our analysis. We also define
\begin{align*}
    V^{t,k}_z = \bbE&[\|D^{t,k}_z\|^2],\quad V^{t,k}_v = \bbE[\|\calG^{t,k}_v\|^2],\\
    &V^{t,k}_x = \bbE[\|D^{t,k}_x\|^2]
\end{align*}
the variances and their respective sums over the inner loop
\begin{align*}
    \calV^{t,k}_z = \sum_{r=1}^k&\bbE[\|D^{t,r-1}_z\|^2],\quad
    \calV^{t,k}_v = \sum_{r=1}^k\bbE[\|\calG^{t,r-1}_v\|^2],\\
    &\calV^{t,k}_x = \sum_{r=1}^k\bbE[\|D^{t,r-1}_x\|^2]\enspace.
\end{align*}
\begin{lemma}[Descent on the inner level]\label{lemma:descent_z} Assume that the step sizes $\rho$ and $\gamma$ verify 
$\gamma~\leq~C_z\rho$ for some positive constant $C_z$ specified in the appendix. Then it holds
\begin{align}\label{eq:descent_z}
    \phi_z^{t,k+1} &\leq \left(1 - \frac{\mu_G}2\rho\right)\phi_z^{t,k} - \frac\rho2\left(1 - \Lambda_z\rho\right)V_z^{t,k}\\\nonumber
        &\qquad +  \rho^3\beta_{zz}\calV_z^{t,k} + \gamma^2\rho\beta_{zv}\calV_v^{t,k} + \gamma^2\rho \beta_{zx}\calV_x^{t,k} \\\nonumber
        &\qquad   + \frac{\Lambda_z}2\gamma^2V_x^{t,k} + \frac{\gamma^2}{\rho}\overline{\beta}_{zx}\bbE[\|D_x(\bfu^{t,k})\|^2]
\end{align}
for some positive constants $\Lambda_z, \beta_{zz}$, $\beta_{zx}$ and $\overline{\beta}_{zx}$ that are specified in the appendix.
\end{lemma}
In \eqref{eq:descent_z} we recover a decrease of $\phi_z^{t,k}$ by a factor $(1 - \rho\mu_G)$. But the outer variable's movement and the noise make appear  $D_x(\bfu^{t,k})$ and the variance hindering the convergence of $z$ towards $z^*(x)$.

For the variable $v$, the quantity we consider is
\vspace{-.2em}
\begin{equation*}
    \phi_v(v,x)~=~\Psi(z^*(x), v, x) - \Psi(z^*(x), v^*(x), x)
\end{equation*}
\vspace{-.5em}
where $\Psi(z,v,x)$ is defined as
\vspace{-.5em}
\begin{equation*}
    \Psi(z,v,x) = \frac12v^\top\nabla^2_{11} G(z,x) v + \nabla_1 F(z,x)^\top v\enspace.
\end{equation*}
The intuition behind considering this quantity is that solving the linear system \eqref{eq:lin_syst} is equivalent to minimizing over $v$ the function $\Psi(z^*(x), v, x)$.
\begin{lemma}\label{lemma:descent_v}
    Assume that the step sizes $\rho$ and $\gamma$ verify $\rho\leq B_v$ and $\gamma~\leq~C_v\rho$ for some positive constants $B_v$ and $C_v$ specified in the appendix. Then it holds
    \begin{align*}
        \phi_v^{t,k+1} &\leq \left(1 - \frac{\rho\mu_G}{16}\right)\phi_v^{t,k}  - \tilde{\beta}_{vv}\rho V_v^t  + \rho^3\beta_{vz}\calV_z^{t,k}\\\nonumber
        &\qquad +2\rho^3\beta_{vv}\calV_v^{t,k} + \gamma^2\rho\beta_{vx}\calV_x^{t,k} + \rho\alpha_{vz} \phi_z^{t,k}\\\nonumber
        &\qquad+ \frac{\Lambda_v}2\gamma^2 \bbE[\|D_x^{t,k}\|^2]+ \frac{\gamma^2}\rho\overline{\beta}_{vx}\bbE[\|D_x(\bfu^{t,k})\|^2]
    \end{align*}
    \vspace{-.3em}
    for some positive constants $\Lambda_v, \beta_{vz}$, $\beta_{vx}$, $\tilde{\beta}_{vv}$ and $\overline{\beta}_{vx}$ that are specified in the appendix.
\end{lemma}
\cref{lemma:descent_v} is similar to \cref{lemma:descent_z} with a term in $\phi_z^{t,k}$ taking into account the error of $z^*(x)$'s approximation.
Its proof harnesses the generalization of Polyak-\L ojasiewicz inequality for composite functions introduced in \cite{Karimi2016PL_linear_cvg}.

The following lemma is a consequence of the smoothness of $h$. Let us denote the expected values $h^{t,k} = \bbE[h(x^{t,k})]$ and expected gradient $g^{t,k}~=~\bbE[\|\nabla h(x^{t,k})\|^2]$.
\begin{lemma}\label{lemma:descent_h}
There exist constants $\beta_{hz}$, $\beta_{hv}$, $\beta_{hx} > 0$ such that
\begin{align*}
    h^{t,k+1}&\leq h^{t,k} - \frac\gamma2 g^{t,k} + \gamma \frac{2L_x^2}{\mu_G}(\phi_z^{t,k} + \phi_v^{t,k}) \\\nonumber
    &\qquad + \gamma \rho^2\beta_{hz}\calV_z^{t,k} + \gamma \rho^2 \beta_{hv}\calV^{t,k}_v + \gamma^3 \beta_{hx}\calV^{t,k}_x \\\nonumber
    &\qquad - \frac\gamma2\left(1-L^h\gamma\right) V^{t,k}_x\enspace .
\end{align*}
\end{lemma}

This lemma shows that the control of the approximation error $\phi_\bullet$ (\cref{lemma:descent_z} and \cref{lemma:descent_v}) and the sum of variances $\calV_\bullet$ is crucial to get a decrease of $\bbE[h(x^{t,k})]$.

\subsection{Complexity Analysis of SRBA}
In \cref{th:cvg_rate}, we provide the convergence rate of SRBA towards a stationary point.
\vspace{-.2em}
\begin{restatable}{shadedtheorem}{srbarate}\label{th:cvg_rate}
    Assume that Assumptions \ref{ass:regul_F} and \ref{ass:regul_G} hold. Assume that the step sizes verify $\rho \leq \overline{\rho}$ and $\gamma \leq \min(\overline{\gamma}, \xi\rho)$ for some constants $\xi$, $\overline{\rho}$ and $\overline{\gamma}$ specified in appendix. Then it holds
    \begin{equation*}
        \frac1{Tq}\sum_{t=0}^{T-1}\sum_{k=0}^{q-1}\bbE[\|\nabla h(x^{t,k})\|^2] = \calO\left(\frac1{qT\gamma}\right)
    \end{equation*}
    where $\calO$ hides regularity constants that are independent from $n$ and $m$.
\end{restatable}
The proof combines classical proof techniques from the bilevel literature and elements from SARAH's analysis \citep{Nguyen2017, Nguyen2022}.
We introduce the Lyapunov function $\calL(\bfu^{t,k}) = h^{t,k} + \psi_z\phi_z^{t,k} + \psi_v\phi_v^{t,k}$ where $\psi_z$ and $\psi_v$ are non-negative constants chosen so that we have the inequality $\calL(\bfu^{t,k+1})\leq \calL(\bfu^{t,k}) - \frac\gamma4g^{t,k}$.
Summing and telescoping this inequality provides the result.

Note that increasing $q$ allows a faster convergence in terms of iterations but makes each iteration more expensive since the number of oracle calls per iteration is $(2n + 3m) + 2\times5(q-1)$. Thus, there is a trade-off between the convergence rate and the overall complexity. In \cref{cor:sample_complexity}, we state that the value of $q$ that gives the best sample complexity is $\calO(n+m)$.

\begin{corollary}\label{cor:sample_complexity}
    Suppose that Assumptions \ref{ass:regul_F} and \ref{ass:regul_G} hold. If we take $\rho~=~\overline{\rho}(n+m)^{-\frac12}$, $\gamma~=~\min(\overline{\gamma}, \xi\rho)(n+m)^{-\frac12}$ and $q~=~n+m$, then $\calO\left((n+m)^\frac12\varepsilon^{-1}\vee (n+m)\right)$ calls to oracles are sufficient to find an $\varepsilon$-stationary point with SRBA.
\end{corollary}
This sample complexity is analogous to the sample complexity of SARAH in the nonconvex finite-sum setting.
To the best of our knowledge, such a rate is the best known for bilevel empirical risk minimization problems in terms of dependency on the number of samples $n+m$ and the precision $\varepsilon$.
This improves by a factor $(n+m)^{-\frac16}$ the previous result which was achieved by SABA \citep{Dagreou2022SABA}.
As a comparison, VRBO \citep{Yang2021a} achieves a sample complexity in $\tilde\calO(\varepsilon^{-\frac32})$.
Note that, for large value of $n+m$ we can have actually $(n+m)^{\frac12}\varepsilon^{-1} \gtrsim \varepsilon^{-2}$.
This means that, just like single-level SARAH, the complexity of SRBA can be beaten by others when the number of samples is too high with respect to the desired accuracy (actually if $n+m = \Omega(\varepsilon^{-2})$).

\section{Lower Bound for Bilevel ERM}
\label{sec:lower_bound}
In this section, we derive a lower bound for bilevel empirical risk minimization problems. This shows that SRBA is a near-optimal algorithm for this class of problems. 

\textbf{Function and Algorithm Classes}~
We define the function and algorithm classes we consider.
\begin{definition}\label{def:function_class}
Let $n,m$ two positive integers, $L^F_1$ and $\mu_G$ two positive constants. The class of the smooth empirical risk minimization problems denoted by $\calC^{L^F_1, \mu_G}$ is the set of pairs of real-valued function families $((F_j)_{1\leq j\leq m}, (G_i)_{1\leq i\leq n})$ defined on $\bbR^p\times\bbR^d$ such that for all $j\in\setcomb{m}$, $F_j$ is $L^F_1$ smooth and for all $i\in\setcomb{n}$, $G_i$ is twice differentiable and $\mu_G$-strongly convex.
\end{definition}
Note that we consider a class of nonconvex bilevel problems.
This class contains, the functions defining the bilevel formulation of the datacleaning task.

For the algorithmic class, we consider algorithms that use approximate implicit differentiation.
\begin{definition}\label{def:algo_class}
    Given initial points $z^0, v^0, x^0$, a \emph{linear bilevel algorithm $\calA$} is a measurable mapping such that for any $((F_j)_{1\leq j\leq m}, (G_i)_{1\leq i\leq n})~\in~\calC^{L^F_1, \mu_G}$, the output of $\calA((F_j)_{1\leq j\leq m}, (G_i)_{1\leq i\leq n})$ is a sequence $\{(z^t, v^t, x^t, i_t, j_t)\}_{t\geq 0}$ of points $(z^t, v^t, x^t)$ and random variables $i_t\in\setcomb{n}$ and $j_t\in\setcomb{m}$ such that for all $t\geq 0$
    \begin{align*}
        z^{t+1} \in z^0 + &\Span\{\nabla_1 G_{i_0}(z^0, x^0),\dots,\nabla_1 G_{i_t}(z^t, x^t)\}\\
        v^{t+1} \in v^0 + &\Span\{\nabla^2_{11}G_{i_0}(z^0, x^0)v^0 + \nabla_1 F_{j_0}(z^0, x^0),\\
        &\dots,\nabla^2_{11}G_{i_t}(z^t, x^t)v^t + \nabla_1 F_{j_t}(z^t, x^t)\}\\
        x^{t+1} \in x^0 + &\Span\{\nabla^2_{21}G_{i_0}(z^0, x^0)v^0 + \nabla_2 F_{j_0}(z^0, x^0),\\
        &\dots,\nabla^2_{21}G_{i_t}(z^t, x^t)v^t + \nabla_2 F_{j_t}(z^t, x^t)\} .
    \end{align*}
\end{definition}
This algorithm class includes popular stochastic bilevel first-order algorithms, such as AmIGO \citep{Arbel2022amigo}, FSLA \citep{Li2022}, SOBA, and SABA \citep{Dagreou2022SABA}. Moreover, despite the projection step, SRBA is part of this algorithm class since the projection of a vector onto $\Gamma$ is actually just a rescaling.

\textbf{Main Theorem}~
Problem \eqref{eq:bilevel_pb} is actually a smooth nonconvex optimization problem. The lower complexity bound for nonconvex finite sum problem has been studied in \cite{Fang2018, Zhou2019}. In particular, they show that the number of gradient calls needed to get an $\varepsilon$-stationary point for a smooth nonconvex finite sum is at least $\calO(m^{\frac12}\varepsilon^{-1})$, where $m$ is the number of terms in the finite sum.

Intuitively, we expect the lower complexity bound to solve \eqref{eq:bilevel_pb} to be larger. Indeed, bilevel problems are harder than single-level problems because a bilevel problem involves the resolution of several subproblems to progress in its resolution.
\cref{th:lower_bound} formalizes this intuition by showing that the classical single-level lower bound is also a lower bound for bilevel problems.
\begin{restatable}{shadedtheorem}{lowerbound}\label{th:lower_bound}
For any linear bilevel algorithm $\calA$, and any $L^F$, $n$, $\Delta$, $\varepsilon$, $p$ such that $\varepsilon\leq(\Delta L^Fm^{-1})/10^3$, there exists a dimension $d=\calO(\Delta\varepsilon^{-1}m^{\frac12}L^F)$, an element $((F_j)_{1\leq j\leq m}, (G_i)_{1\leq i\leq n})\in\calC^{L^F_1,\mu_G}$ such that the value function $h$ defined as in \eqref{eq:bilevel_pb} satisfies $h(x^0) - \inf_{x\in\bbR^d} h(x)\leq \Delta$ and in order to find $\hat x\in\bbR^d$ such that $\bbE[\|\nabla h(\hat x)\|^2]\leq \varepsilon$, $\calA$ needs at least $\Omega(m^{\frac12}\varepsilon^{-1})$ calls to oracles of the form \eqref{eq:oracles}.
\end{restatable}
The proof is an adaptation of the proof of \citet[Theorem 4.7]{Zhou2019}. We take as outer function $F$ defined by $F(z,x) = \sum_{j=1}^m f(U^{(j)}z)$ where $f$ is the ``worst-case function'' used by \citet{Carmon2017lowerbound2}, $U = [U^{(j)}, \dots, U^{(m)}]^\top$ is an orthogonal matrix and $G(z,x) = \frac12\|z-x\|^2$. We leverage the fact that $\|\nabla f(y)\|^2 > K$ as long as the two last coordinates of $y$ are zero for some known constant $K$. Then we use the ``zero chain property'' to bound the number of indices $j$ such the two last components of $U^{(j)}x^t$ are zero at a given iteration $t$, implying $\|\nabla h(x^t)\|^2> \epsilon$ when $t$ is smaller than $\calO(m^{\frac12}\varepsilon^{-1})$.

As a comparison to the existing lower bound for bilevel optimization in \cite{Ji2023accBio}, we consider randomized algorithms and do not assume the value function $h$ to be convex or strongly convex.

\section{Numerical Experiments}\label{sec:expe}
Even though our contribution is mostly theoretical, we run several experiments to highlight to compare the proposed algorithm with state-of-the-art stochastic bilevel solvers. \new{We compare our method to AmIGO \citep{Arbel2022amigo}, F$^2$SA \citep{Kwon2023f2sa}, MRBO \citep{Yang2021a}, VRBO \citep{Yang2021a}, StocBiO \citep{Ji2021stocbio} and SABA \citep{Dagreou2022SABA}.}
They are run \new{on a synthetic problem with quadratic functions and} on a hyperparameter selection problem for $\ell^2$-regularized logistic regression with the dataset IJCNN1\footnote{\url{https://www.csie.ntu.edu.tw/~cjlin/libsvmtools/datasets/binary.html}}. 
A more detailed description of the experiments is available in \cref{app:sec:exp} and an additional experiment the datacleaning task is available in \cref{app:sec:supp_exp}.

\new{\textbf{Experiments on quadratics}}~
\new{To evaluate the performance of stochastic bilevel optimizers in a controlled setting, we perform a benchmark on quadratic loss functions described in \cref{app:sec:exp}.
Here $F$ and $G$ are quadratic jointly in $(z, x)$, allowing us to choose freely the conditioning of $F$, $G$, and $h$. We take for the Hessian and cross derivative matrices of each sample, the empirical correlation of random vectors drawn with a prescribed covariance matrix. The generation process is detailed in \cref{app:sec:exp}.
In \cref{fig:exps}, we report the norm of the gradient of the value function function with respect to time.
Our first observation is that among all the methods, SRBA and SABA converge the fastest.
These two solvers share two key ingredients: variance reduction and warm-starting.
Variance reduction makes the variance of the gradient estimate go to zero without using decreasing step sizes.
The warm-starting strategy in both the approximation of $z^*(x^t)$ and the approximation of $v^*(x^t)$ enables getting an estimator of $\nabla h(x^t)$ which is asymptotically unbiased, without requiring an increasing number of inner iterations or batch-size.
Note that solvers using Neumann iterations (VRBO, MRBO, stocBiO) fail to converge because Neumann iterations provide a biased estimate of $v^*(x)$.
Moreover, AmiIGO and stocBiO evolve slowly after some iterations because they require vanishing step sizes to converge.
Finally, SRBA is faster than SABA, which is consistent with the theory.}

\textbf{\new{Hyperparameter selection}}~
\new{We also run an experiment on hyperparameter selection problem for $\ell^2$-regularized logistic regression with the IJCNN1 dataset.}
SRBA shows good performances in the experiment, both in speed and accuracy.
It is competitive with other state-of-the-art methods AmIGO and SABA, while going faster than Amigo and requiring less memory than SABA.
VRBO --another extension of SARAH for bilevel problems-- is slower in all problems.
This is due to the burden of computing the approximate hypergradient at each inner iteration without updating the outer parameter.
We can also notice that in the experiment on IJCNN1, the slowest method are method implementing Neumann approximations to approximate $v^*(x)$.
\new{Note that this last experiment does not include F$^2$SA because we find that on this problem, the norm of the iterates of F$^2$SA goes towards infinity.}

\begin{figure}[ht]
    \centering
    \includegraphics[width=.4\textwidth]{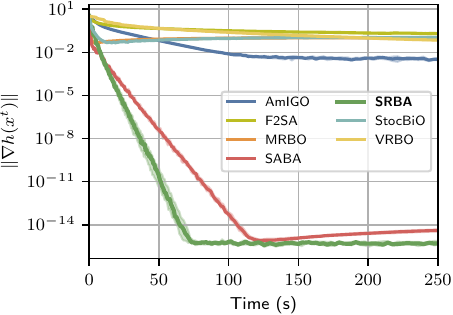}

    \includegraphics[width=.4\textwidth]{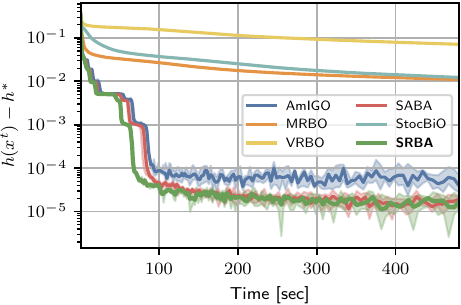}
     \caption{\new{Comparison of the behavior of SRBA with other stochastic bilevel solvers. For each experiment, the solvers are run with 10 different seeds and the median performance over these seeds is reported. The shaded area corresponds to the performances between the 20\% and the 80\% percentiles. The performances are reported with respect to wall-clock time. \textbf{Top:} Experiments on quadratic functions. We report the gradient norm of the value function. \textbf{Bottom: } Hyperparameter selection with the IJCNN1 dataset.}}
    \label{fig:exps}
\end{figure}

\section{Conclusion}
In this paper, we have introduced SRBA, an algorithm for bilevel empirical risk minimization. We have demonstrated that the sample complexity of SRBA is $\mathcal{O}((n+m)^{\frac12}\varepsilon^{-1})$ for any bilevel problem where the inner problem is strongly convex. Then, we have demonstrated that any bilevel empirical risk minimization algorithm has a sample complexity of at least $\mathcal{O}(m^{\frac12}\varepsilon^{-1})$ on some problems where the inner problem is strongly convex. This demonstrates that SRBA is optimal, up to constant factors, and that bilevel ERM is as hard as single-level nonconvex ERM.

\bibliography{biblio}

\begin{thebibliography}{50}
\providecommand{\natexlab}[1]{#1}
\providecommand{\url}[1]{\texttt{#1}}
\expandafter\ifx\csname urlstyle\endcsname\relax
  \providecommand{\doi}[1]{doi: #1}\else
  \providecommand{\doi}{doi: \begingroup \urlstyle{rm}\Url}\fi

\bibitem[Ablin et~al.(2020)Ablin, Peyr{\'e}, and Moreau]{Ablin2020}
Pierre Ablin, Gabriel Peyr{\'e}, and Thomas Moreau.
\newblock Super-efficiency of automatic differentiation for functions defined
  as a minimum.
\newblock In \emph{International {{Conference}} on {{Machine Learning}}
  ({{ICML}})}, 2020.

\bibitem[Agarwal and Bottou(2015)]{Agarwal2015}
Alekh Agarwal and L{\'e}on Bottou.
\newblock A {{Lower Bound}} for the {{Optimization}} of {{Finite Sums}}.
\newblock In \emph{International {{Conference}} on {{Machine Learning}}
  ({{ICML}})}, 2015.

\bibitem[Arbel and Mairal(2022)]{Arbel2022amigo}
Michael Arbel and Julien Mairal.
\newblock Amortized {{Implicit Differentiation}} for {{Stochastic Bilevel
  Optimization}}.
\newblock In \emph{International {{Conference}} on {{Learning Representations}}
  ({{ICLR}})}, 2022.

\bibitem[Bai et~al.(2019)Bai, Kolter, and Koltun]{Bai2019}
Shaojie Bai, J.~Zico Kolter, and Vladlen Koltun.
\newblock Deep {{Equilibrium Models}}.
\newblock In \emph{Advances in {{Neural Information Processing Systems}}
  ({{NeurIPS}})}, 2019.

\bibitem[Bengio(2000)]{Bengio2000}
Yoshua Bengio.
\newblock Gradient-{{Based Optimization}} of {{Hyperparameters}}.
\newblock \emph{Neural Computation}, 12\penalty0 (8):\penalty0 1889--1900,
  2000.
\newblock ISSN 0899-7667, 1530-888X.
\newblock \doi{10.1162/089976600300015187}.

\bibitem[Bottou(2010)]{Bottou2010}
L{\'e}on Bottou.
\newblock Large-{{Scale Machine Learning}} with {{Stochastic Gradient
  Descent}}.
\newblock In \emph{International Conference on Computational Statistics
  (COMPSTAT)}, pages 177--186, 2010.

\bibitem[Bottou et~al.(2018)Bottou, Curtis, and Nocedal]{Bottou2018}
L{\'e}on Bottou, Frank~E. Curtis, and Jorge Nocedal.
\newblock Optimization methods for large-scale machine learning.
\newblock \emph{Siam Reviews}, 60\penalty0 (2):\penalty0 223--311, 2018.

\bibitem[Carmon et~al.(2020)Carmon, Duchi, Hinder, and
  Sidford]{Carmon2017lowerbound1}
Yair Carmon, John~C. Duchi, Oliver Hinder, and Aaron Sidford.
\newblock Lower bounds for finding stationary points {{I}}.
\newblock \emph{Mathematical Programming}, 184\penalty0 (1-2):\penalty0
  71--120, 2020.

\bibitem[Carmon et~al.(2021)Carmon, Duchi, Hinder, and
  Sidford]{Carmon2017lowerbound2}
Yair Carmon, John~C. Duchi, Oliver Hinder, and Aaron Sidford.
\newblock Lower bounds for finding stationary points {{II}}: First-order
  methods.
\newblock \emph{Mathematical Programming}, 185\penalty0 (1-2):\penalty0
  315--355, 2021.

\bibitem[Chen et~al.(2021)Chen, Sun, and Yin]{Chen2021alset}
Tianyi Chen, Yuejiao Sun, and Wotao Yin.
\newblock Closing the {{Gap}}: {{Tighter Analysis}} of {{Alternating Stochastic
  Gradient Methods}} for {{Bilevel Problems}}.
\newblock In \emph{Advances in {{Neural Information Processing Systems}}
  ({{NeurIPS}})}, 2021.

\bibitem[Cutkosky and Orabona(2019)]{Cutkosky2019}
Ashok Cutkosky and Francesco Orabona.
\newblock Momentum-based variance reduction in non-convex {{SGD}}.
\newblock In \emph{Advances in Neural Information Processing Systems
  ({{NeurIPS}})}, 2019.

\bibitem[Dagr{\'e}ou et~al.(2022)Dagr{\'e}ou, Ablin, Vaiter, and
  Moreau]{Dagreou2022SABA}
Mathieu Dagr{\'e}ou, Pierre Ablin, Samuel Vaiter, and Thomas Moreau.
\newblock A framework for bilevel optimization that enables stochastic and
  global variance reduction algorithms.
\newblock In \emph{Advances in {{Neural Information Processing Systems}}
  ({{NeurIPS}})}, 2022.

\bibitem[Defazio et~al.(2014)Defazio, Bach, and {Lacoste-Julien}]{Defazio2014}
Aaron Defazio, Francis Bach, and Simon {Lacoste-Julien}.
\newblock {{SAGA}}: {{A Fast Incremental Gradient Method With Support}} for
  {{Non-Strongly Convex Composite Objectives}}.
\newblock In \emph{Advances in {{Neural Information Processing Systems}}
  ({{NeurIPS}})}, 2014.

\bibitem[Fang et~al.(2018)Fang, Li, Lin, and Zhang]{Fang2018}
Cong Fang, Chris~Junchi Li, Zhouchen Lin, and Tong Zhang.
\newblock {{SPIDER}}: {{Near-Optimal Non-Convex Optimization}} via {{Stochastic
  Path Integrated Differential Estimator}}.
\newblock In \emph{Advances in {{Neural Information Processing Systems}}
  ({{NeurIPS}})}, 2018.

\bibitem[Franceschi et~al.(2017)Franceschi, Donini, Frasconi, and
  Pontil]{Franceschi2017ForwardReverse}
Luca Franceschi, Michele Donini, Paolo Frasconi, and Massimiliano Pontil.
\newblock Forward and {{Reverse Gradient-Based Hyperparameter Optimization}}.
\newblock In \emph{International {{Conference}} on {{Machine Learning}}
  ({{ICML}})}, 2017.

\bibitem[Franceschi et~al.(2018)Franceschi, Frasconi, Salzo, Grazzi, and
  Pontil]{Franceschi2018hyperparameter}
Luca Franceschi, Paolo Frasconi, Saverio Salzo, Riccardo Grazzi, and
  Massimilano Pontil.
\newblock Bilevel {{Programming}} for {{Hyperparameter Optimization}} and
  {{Meta-Learning}}.
\newblock In \emph{International {{Conference}} on {{Machine Learning}}
  ({{ICML}})}, 2018.

\bibitem[Ghadimi and Lan(2013)]{Ghadimi2013}
Saeed Ghadimi and Guanghui Lan.
\newblock Stochastic first- and zeroth-order methods for nonconvex stochastic
  programming.
\newblock \emph{SIAM Journal on Optimization}, 23\penalty0 (4):\penalty0
  2341--2368, 2013.

\bibitem[Ghadimi and Wang(2018)]{Ghadimi2018}
Saeed Ghadimi and Mengdi Wang.
\newblock Approximation {{Methods}} for {{Bilevel Programming}}.
\newblock \emph{arXiv preprint arXiv:1802.02246}, 2018.

\bibitem[Grazzi et~al.(2020)Grazzi, Franceschi, Pontil, and Salzo]{Grazzi2020}
Riccardo Grazzi, Luca Franceschi, Massimiliano Pontil, and Saverio Salzo.
\newblock On the iteration complexity of hypergradient computation.
\newblock In Hal~Daum{\'e} III and Aarti Singh, editors, \emph{International
  {{Conference}} on {{Machine Learning}} ({{ICML}})}, 2020.

\bibitem[Hong et~al.(2023)Hong, Wai, Wang, and Yang]{Hong2021}
Mingyi Hong, Hoi-To Wai, Zhaoran Wang, and Zhuoran Yang.
\newblock A {{Two-Timescale Stochastic Algorithm Framework}} for {{Bilevel
  Optimization}}: {{Complexity Analysis}} and {{Application}} to
  {{Actor-Critic}}.
\newblock \emph{SIAM Journal on Optimization}, 33\penalty0 (1):\penalty0
  147--180, 2023.

\bibitem[Hu et~al.(2022)Hu, Zhong, and Yang]{Hu2022MultiBlock}
Quanqi Hu, Yongjian Zhong, and Tianbao Yang.
\newblock Multi-block {{Min-max Bilevel Optimization}} with {{Applications}} in
  {{Multi-task Deep AUC Maximization}}.
\newblock In \emph{Advances in {{Neural Information Processing Systems}}
  ({{NeurIPS}})}, 2022.

\bibitem[Huang et~al.(2022)Huang, Li, Gao, and Huang]{Huang2022BilevelBregmana}
Feihu Huang, Junyi Li, Shangqian Gao, and Heng Huang.
\newblock Enhanced {{Bilevel Optimization}} via {{Bregman Distance}}.
\newblock In \emph{Advances in {{Neural Information Processing Systems}}
  ({{NeurIPS}})}, 2022.

\bibitem[Ji and Liang(2023)]{Ji2023accBio}
Kaiyi Ji and Yingbin Liang.
\newblock Lower {{Bounds}} and {{Accelerated Algorithms}} for {{Bilevel
  Optimization}}.
\newblock \emph{Journal of Machine Learning Research}, 24\penalty0
  (22):\penalty0 1--56, 2023.

\bibitem[Ji et~al.(2021)Ji, Yang, and Liang]{Ji2021stocbio}
Kaiyi Ji, Junjie Yang, and Yingbin Liang.
\newblock Bilevel {{Optimization}}: {{Convergence Analysis}} and {{Enhanced
  Design}}.
\newblock In \emph{International {{Conference}} on {{Machine Learning}}
  ({{ICML}})}, 2021.

\bibitem[Karimi et~al.(2016)Karimi, Nutini, and
  Schmidt]{Karimi2016PL_linear_cvg}
Hamed Karimi, Julie Nutini, and Mark Schmidt.
\newblock Linear {{Convergence}} of {{Gradient}} and {{Proximal-Gradient
  Methods Under}} the {{Polyak-}}\textbackslash{{Lojasiewicz Condition}}.
\newblock In \emph{European {{Conference}} on {{Machine Learning}} ({{ECML}})},
  2016.

\bibitem[Khanduri et~al.(2021)Khanduri, Zeng, Hong, Wai, Wang, and
  Yang]{Khanduri2021sustain}
Prashant Khanduri, Siliang Zeng, Mingyi Hong, Hoi-To Wai, Zhaoran Wang, and
  Zhuoran Yang.
\newblock A {{Near-Optimal Algorithm}} for {{Stochastic Bilevel Optimization}}
  via {{Double-Momentum}}.
\newblock In \emph{Advances in Neural Information Processing Systems
  ({{NeurIPS}})}, 2021.

\bibitem[Kwon et~al.(2023)Kwon, Kwon, Wright, and Nowak]{Kwon2023f2sa}
Jeongyeol Kwon, Dohyun Kwon, Stephen Wright, and Robert Nowak.
\newblock A {{Fully First-Order Method}} for {{Stochastic Bilevel
  Optimization}}.
\newblock In \emph{International {{Conference}} on {{Machine Leaning}}
  ({{ICML}})}, 2023.

\bibitem[Li et~al.(2022)Li, Gu, and Huang]{Li2022}
Junyi Li, Bin Gu, and Heng Huang.
\newblock A {{Fully Single Loop Algorithm}} for {{Bilevel Optimization}}
  without {{Hessian Inverse}}.
\newblock In \emph{Proceedings of the {{Thirty-sixth AAAI Conference}} on
  {{Artificial Intelligence}}, {{AAAI}}'22}, 2022.

\bibitem[Li et~al.(2020)Li, Hu, Wang, Hospedales, Robertson, and
  Yang]{Li2020dada}
Yonggang Li, Guosheng Hu, Yongtao Wang, Timothy Hospedales, Neil~M. Robertson,
  and Yongxin Yang.
\newblock {{DADA}}: {{Differentiable Automatic Data Augmentation}}.
\newblock \emph{arXiv preprint arXiv:2003.03780}, 2020.

\bibitem[Liu et~al.(2019)Liu, Simonyan, and Yang]{Liu2019darts}
Hanxiao Liu, Karen Simonyan, and Yiming Yang.
\newblock {{DARTS}}: {{Differentiable Architecture Search}}.
\newblock In \emph{International {{Conference}} on {{Learning Representations}}
  ({{ICLR}})}, 2019.

\bibitem[Lorraine et~al.(2020)Lorraine, Vicol, and
  Duvenaud]{Lorraine2020MillionsofHyperparameters}
Jonathan Lorraine, Paul Vicol, and David Duvenaud.
\newblock Optimizing {{Millions}} of {{Hyperparameters}} by {{Implicit
  Differentiation}}.
\newblock In \emph{International {{Conference}} on {{Artificial Intelligence}}
  and {{Statistics}} ({{AISTAT}})}, pages 1540--1552, 2020.

\bibitem[Maclaurin et~al.(2015)Maclaurin, Duvenaud, and Adams]{Maclaurin2015}
Dougal Maclaurin, David Duvenaud, and Ryan~P. Adams.
\newblock Gradient-based {{Hyperparameter Optimization}} through {{Reversible
  Learning}}.
\newblock In \emph{International {{Conference}} on {{Machine Learning}}
  ({{ICML}})}, 2015.

\bibitem[Moreau et~al.(2022)Moreau, Massias, Gramfort, Ablin, Charlier,
  Dagr{\'e}ou, {la Tour}, Durif, Dantas, Klopfenstein, Larsson, Lai, Lefort,
  Mal{\'e}zieux, Moufad, Nguyen, Rakotomamonjy, Ramzi, Salmon, and
  Vaiter]{Moreau2022}
Thomas Moreau, Mathurin Massias, Alexandre Gramfort, Pierre Ablin,
  Pierre-Antoine Bannier~Benjamin Charlier, Mathieu Dagr{\'e}ou, Tom~Dupr{\'e}
  {la Tour}, Ghislain Durif, Cassio~F. Dantas, Quentin Klopfenstein, Johan
  Larsson, En~Lai, Tanguy Lefort, Benoit Mal{\'e}zieux, Badr Moufad, Binh~T.
  Nguyen, Alain Rakotomamonjy, Zaccharie Ramzi, Joseph Salmon, and Samuel
  Vaiter.
\newblock Benchopt: {{Reproducible}}, efficient and collaborative optimization
  benchmarks.
\newblock In \emph{Advances in {{Neural Information Processing Systems}}
  ({{NeurIPS}})}, 2022.

\bibitem[Nemirovsky and Yudin(1983)]{Nemirovsky1983}
Arkadi{\u \i}~Semenovich Nemirovsky and David~Berkovich Yudin.
\newblock \emph{{Problem complexity and method efficiency in optimization}}.
\newblock {Wiley-Interscience series in discrete mathematics}. {Wiley},
  {Chichester ; New York}, 1983.
\newblock ISBN 978-0-471-10345-5.

\bibitem[Nesterov(2018)]{Nesterov2018}
Yurii Nesterov.
\newblock \emph{Lectures on Convex Optimization}.
\newblock {Springer Berlin Heidelberg}, {New York, NY}, 2018.
\newblock ISBN 978-3-319-91577-7.

\bibitem[Nguyen et~al.(2017)Nguyen, Liu, Scheinberg, and Tak{\'a}{\v
  c}]{Nguyen2017}
Lam~M. Nguyen, Jie Liu, Katya Scheinberg, and Martin Tak{\'a}{\v c}.
\newblock {{SARAH}}: {{A Novel Method}} for {{Machine Learning Problems Using
  Stochastic Recursive Gradient}}.
\newblock In \emph{International {{Conference}} on {{Machine Learning}}
  ({{ICML}})}, 2017.

\bibitem[Nguyen et~al.(2022)Nguyen, {van Dijk}, Phan, Nguyen, Weng, and
  Kalagnanam]{Nguyen2022}
Lam~M. Nguyen, Marten {van Dijk}, Dzung~T. Phan, Phuong~Ha Nguyen, Tsui-Wei
  Weng, and Jayant~R. Kalagnanam.
\newblock Finite-sum smooth optimization with {{SARAH}}.
\newblock \emph{Computational Optimization and Applications}, 82\penalty0
  (3):\penalty0 561--593, 2022.
\newblock ISSN 0926-6003, 1573-2894.
\newblock \doi{10.1007/s10589-022-00375-x}.

\bibitem[Pearlmutter(1994)]{pearlmutter1994}
Barak~A. Pearlmutter.
\newblock Fast {{Exact Multiplication}} by the {{Hessian}}.
\newblock \emph{Neural Computation}, 6\penalty0 (1):\penalty0 147--160, 1994.
\newblock ISSN 0899-7667, 1530-888X.
\newblock \doi{10.1162/neco.1994.6.1.147}.

\bibitem[Pedregosa(2016)]{Pedregosa2016hoag}
Fabian Pedregosa.
\newblock Hyperparameter optimization with approximate gradient.
\newblock In \emph{International {{Conference}} on {{Machine Learning}}
  ({{ICML}})}, 2016.

\bibitem[Rajeswaran et~al.(2019)Rajeswaran, Finn, Kakade, and
  Levine]{Rajeswaran2019imaml}
Aravind Rajeswaran, Chelsea Finn, Sham Kakade, and Sergey Levine.
\newblock Meta-{{Learning}} with {{Implicit Gradients}}.
\newblock In \emph{Advances in {{Neural Information Processing Systems}}
  ({{NeurIPS}})}, 2019.

\bibitem[Ramzi et~al.(2022)Ramzi, Mannel, Bai, Starck, Ciuciu, and
  Moreau]{Ramzi2022}
Zaccharie Ramzi, Florian Mannel, Shaojie Bai, Jean-Luc Starck, Philippe Ciuciu,
  and Thomas Moreau.
\newblock {{SHINE}}: {{SHaring}} the {{INverse Estimate}} from the forward pass
  for bi-level optimization and implicit models.
\newblock In \emph{International {{Conference}} on {{Learning Representations}}
  ({{ICLR}})}, 2022.

\bibitem[Reddi et~al.(2016)Reddi, Sra, Poczos, and Smola]{Reddi2016saga}
Sashank~J. Reddi, Suvrit Sra, Barnabas Poczos, and Alex Smola.
\newblock Fast {{Incremental Method}} for {{Nonconvex Optimization}}.
\newblock In \emph{2016 {{IEEE}} 55th {{Conference}} on {{Decision}} and
  {{Control}} ({{CDC}})}, {{IEEE}}, pages 1971--1977, 2016.

\bibitem[Robbins and Monro(1951)]{Robbins1951}
Herbert Robbins and Sutton Monro.
\newblock A stochastic approximation method.
\newblock \emph{The Annals of Mathematical Statistics}, 22\penalty0
  (3):\penalty0 400--407, 1951.
\newblock \doi{10.1214/aoms/1177729586}.

\bibitem[Rommel et~al.(2022)Rommel, Moreau, Paillard, and Gramfort]{Rommel2022}
C{\'e}dric Rommel, Thomas Moreau, Joseph Paillard, and Alexandre Gramfort.
\newblock {{CADDA}}: {{Class-wise Automatic Differentiable Data Augmentation}}
  for {{EEG Signals}}.
\newblock In \emph{International {{Conference}} on {{Learning Representations}}
  ({{ICLR}})}, 2022.

\bibitem[Sow et~al.(2022)Sow, Ji, and Liang]{Sow2022a}
Daouda Sow, Kaiyi Ji, and Yingbin Liang.
\newblock On the {{Convergence Theory}} for {{Hessian-Free Bilevel
  Algorithms}}.
\newblock In \emph{Advances in {{Neural Information Processing Systems}}
  ({{NeurIPS}})}, 2022.

\bibitem[Woodworth and Srebro(2016)]{Woodworth2016}
Blake~E Woodworth and Nati Srebro.
\newblock Tight {{Complexity Bounds}} for {{Optimizing Composite Objectives}}.
\newblock In \emph{Advances in {{Neural Information Systems Processing}}
  ({{NeurIPS}})}, 2016.

\bibitem[Yang et~al.(2021)Yang, Ji, and Liang]{Yang2021a}
Junjie Yang, Kaiyi Ji, and Yingbin Liang.
\newblock Provably {{Faster Algorithms}} for {{Bilevel Optimization}}.
\newblock In \emph{Advances in {{Neural Information Processing Systems}}
  ({{NeurIPS}})}, 2021.

\bibitem[Ye et~al.(2022)Ye, Liu, Wright, Stone, and Liu]{Ye2022}
Mao Ye, Bo~Liu, Stephen Wright, Peter Stone, and Qiang Liu.
\newblock {{BOME}}! {{Bilevel Optimization Made Easy}}: {{A Simple First-Order
  Approach}}.
\newblock In \emph{Advances in {{Neural Information Processing Systems}}
  ({{NeurIPS}})}, 2022.

\bibitem[Zhang et~al.(2021)Zhang, Su, Pan, Chang, Abbasnejad, and
  Haffari]{Zhang2021idarts}
Miao Zhang, Steven~W. Su, Shirui Pan, Xiaojun Chang, Ehsan~M Abbasnejad, and
  Reza Haffari.
\newblock {{iDARTS}}: {{Differentiable Architecture Search}} with {{Stochastic
  Implicit Gradients}}.
\newblock In \emph{International {{Conference}} on {{Machine Learning}}
  ({{ICML}})}, 2021.

\bibitem[Zhou and Gu(2019)]{Zhou2019}
Dongruo Zhou and Quanquan Gu.
\newblock Lower {{Bounds}} for {{Smooth Nonconvex Finite-Sum Optimization}}.
\newblock In \emph{International {{Conference}} on {{Machine Learning}}
  ({{ICML}})}, 2019.

\end{thebibliography}
\bibliographystyle{plainnat}

\subsubsection*{Acknowledgements}
SV acknowledges the support of the ANR GraVa ANR-18-CE40-0005.
This work is supported by a public grant overseen by the French National Research Agency (ANR) through the program UDOPIA, project funded by the ANR-20-THIA-0013-01 and DATAIA convergence institute (ANR-17-CONV-0003).
\section*{Checklist}

 \begin{enumerate}

 \item For all models and algorithms presented, check if you include:
 \begin{enumerate}
   \item A clear description of the mathematical setting, assumptions, algorithm, and/or model. \checklistans{Yes} [\cref{sec:algo}]
   \item An analysis of the properties and complexity (time, space, sample size) of any algorithm. \checklistans{Yes} [\cref{sec:theory}]
   \item (Optional) Anonymized source code, with specification of all dependencies, including external libraries. \checklistans{Yes} [\cref{app:sec:exp}]
 \end{enumerate}

 \item For any theoretical claim, check if you include:
 \begin{enumerate}
   \item Statements of the full set of assumptions of all theoretical results. \checklistans{Yes} [\cref{subsec:ass}]
   \item Complete proofs of all theoretical results. \checklistans{Yes} [\cref{app:sec:cvgt_srba} and \cref{app:sec:lower}]
   \item Clear explanations of any assumptions. \checklistans{Yes} [\cref{subsec:ass}]
 \end{enumerate}

 \item For all figures and tables that present empirical results, check if you include:
 \begin{enumerate}
   \item The code, data, and instructions needed to reproduce the main experimental results (either in the supplemental material or as a URL). \checklistans{Yes} [\cref{app:sec:exp}]
   \item All the training details (e.g., data splits, hyperparameters, how they were chosen). \checklistans{Yes} [\cref{app:sec:exp} and \cref{app:sec:supp_exp}]
         \item A clear definition of the specific measure or statistics and error bars (e.g., with respect to the random seed after running experiments multiple times). \checklistans{Yes} [\cref{fig:exps}]
         \item A description of the computing infrastructure used. (e.g., type of GPUs, internal cluster, or cloud provider). \checklistans{Yes} [\cref{app:sec:exp}]
 \end{enumerate}

 \item If you are using existing assets (e.g., code, data, models) or curating/releasing new assets, check if you include:
 \begin{enumerate}
   \item Citations of the creator If your work uses existing assets. \checklistans{Yes}
   \item The license information of the assets, if applicable. \checklistans{Not applicable}
   \item New assets either in the supplemental material or as a URL, if applicable. \checklistans{Yes}
   \item Information about consent from data providers/curators. \checklistans{Not applicable}
   \item Discussion of sensible content if applicable, e.g., personally identifiable information or offensive content. \checklistans{Not Applicable}
 \end{enumerate}

 \item If you used crowdsourcing or conducted research with human subjects, check if you include:
 \begin{enumerate}
   \item The full text of instructions given to participants and screenshots. \checklistans{Not Applicable}
   \item Descriptions of potential participant risks, with links to Institutional Review Board (IRB) approvals if applicable. \checklistans{Not Applicable}
   \item The estimated hourly wage paid to participants and the total amount spent on participant compensation. \checklistans{Not Applicable}
 \end{enumerate}

 \end{enumerate}

 \onecolumn
\appendix
\counterwithin{figure}{section}
\def\thefigure{\thesection.\arabic{figure}}

\cref{app:sec:cvgt_srba} contains the necessary lemmas and proofs of Section~\ref{sec:theory}.
\cref{app:sec:lower} contains the proof of the lower bound for stochastic bilevel optimization.
\cref{app:sec:exp} details the setting of the numerical experiments.
Finally, \cref{app:sec:supp_exp} contains two more experiments on hyperparameter selection and datacleaning tasks.

\section{Convergence analysis of SRBA}
\label{app:sec:cvgt_srba}

\subsection{Proof of \cref{prop:directions_cancels}}
\begin{proof}
    Let $x\in\bbR^d$. Since $G(\,.\,, x)$ is differentiable and $z^*(x)$ minimizes $G(\,.\,, x)$, the first order optimality condition ensures $\nabla_1 G(z^*(x),x) = 0 = D_z(z^*(x),v^*(x),x)$. Since $G$ is strongly convex with respect to $z$, the Hessian $\nabla^2_{11} G(z^*(x), x)$ is invertible. As a consequence, the equation in $v$
    \begin{align}
        D_v(z^*(x), v, x) = \nabla^2_{11}G(z^*(x),x)v + \nabla_1 F(z^*(x),x) = 0
    \end{align}
    admits a unique solution given by $v^*(x)$.
\end{proof}
\subsection{Smoothness constant of $h$}\label{app:sec:smoothness_h}
We can find in \citet[Lemma 2.2]{Ghadimi2018} the following value for the smoothness constant of $h$
\begin{equation*}
  L^h = L^F_1+ \frac{2L^F_1L^G_2+(L^F_0)^2 L^G_2}{\mu_G}
+\frac{L_{11}^GL^G_1L_0^F+L^G_1L^G_2L^F_0 + (L^G_1)^2L^F_1}{\mu_G^2}
+\frac{(L^G_1)^2L^G_2L^F_0}{\mu_G^3}\enspace.  
\end{equation*}
\subsection{Proof of \cref{prop:bias_variance_decomp}}
\begin{proof}
    Let $t>0$ and $k\in\setcomb{q-1}$. For $k=0$, we directly have $\bbE[\|D_{\bullet}^{t,k} - D_{\bullet}(\bfu^{t,k})\|^2]=0$. For $k\geq 1$ and $r\in\{1,\dots,k\}$, the bias/variance decomposition of $D_{\bullet}^{t,r}$ reads
    \begin{align*}
        \bbE_{t,r}[\|D_{\bullet}^{t,r} - D_{\bullet}(\bfu^{t,r})\|^2] &= \bbE_{t,r}[\|D_{\bullet}^{t,r} - D_{\bullet}^{t,r-1} + D_{\bullet}(\bfu^{t,r-1}) - D_{\bullet}(\bfu^{t,r})\|^2] \\\nonumber
            &\qquad + \|D_{\bullet}(\bfu^{t,r}) + D_{\bullet}(\bfu^{t,r-1}) - D_{\bullet}^{t,r-1} - D_{\bullet}(\bfu^{t,r})\|^2\\
        &= \bbE_{t,r}[\|D_{\bullet}^{t,r} - D_{\bullet}^{t,r-1} - (D_{\bullet}(\bfu^{t,r-1}) - D_{\bullet}(\bfu^{t,r}))\|^2] \\\nonumber
            &\qquad + \|D_{\bullet}^{t,r-1} - D_{\bullet}(\bfu^{t,r-1})\|^2\\
    \end{align*}
    The term $\bbE_{t,r}[\|D_{\bullet}^{t,r} - D_{\bullet}^{t,r-1} - (D_{\bullet}(\bfu^{t,r-1}) - D_{\bullet}(\bfu^{t,r}))\|^2]$ is the variance of $D_{\bullet}^{t,r} - D_{\bullet}^{t,r-1}$, and then can written as 
    \begin{align*}
        \bbE_{t,r}[\|D_{\bullet}^{t,r} - D_{\bullet}^{t,r-1} - (D_{\bullet}(\bfu^{t,r-1}) - D_{\bullet}(\bfu^{t,r}))\|^2] &= \bbE_{t,r}[\|D_{\bullet}^{t,r} - D_{\bullet}^{t,r-1}\|^2] \\\nonumber
            &\qquad - \|D_{\bullet}(\bfu^{t,r}) - D_{\bullet}(\bfu^{t,r-1})\|^2
    \end{align*}
    Plugging this in the previous inequality and taking the total expectation leads to
    \begin{align*}
         \bbE[\|D_{\bullet}^{t,r} - D_{\bullet}(\bfu^{t,r})\|^2] &= \bbE[\|D_{\bullet}^{t,r} - D_{\bullet}^{t,r-1}\|^2] - \bbE[\|D_{\bullet}(\bfu^{t,r}) - D_{\bullet}(\bfu^{t,r-1})\|^2] \\
            &\qquad + \bbE[\|D^{t,r-1}_\bullet - D^{t,r-1}_{\bullet}(\bfu^{t,r-1})\|^2]
    \end{align*}
    Summing for $r\in\{1,\dots,k\}$ and telescoping gives the final result (taking into account that $D^{t,0}_\bullet~=~D_\bullet(\bfu^{t,0})$):
    \begin{equation*}
        \bbE[\|D_{\bullet}^{t,k} - D_{\bullet}(\bfu^{t,k})\|^2] = \sum_{r=1}^k\bbE[\|D_{\bullet}^{t,r} - D_{\bullet}^{t,r-1}\|^2] - \sum_{r=1}^k\bbE[\|D_{\bullet}(\bfu^{t,r}) - D_{\bullet}(\bfu^{t,r-1})\|^2]\enspace .
    \end{equation*}
\end{proof}

\subsection{Technical lemmas}

\begin{lemma}\label{app:lemma:z_star_lipschitz}
    There exists constant $L_{z^*}$ and $L_{v^*}$ such that for any $x_1$, $x_2\in\bbR^d$, we have
    \begin{align*}
        \|z^*(x_1)-z^*(x_2)\|\leq L_{z^*}\|x_1 - x_2\|\quad\text{and}\quad \|v^*(x_1)-v^*(x_2)\|\leq L_{v^*}\|x_1 - x_2\|
    \end{align*}
\end{lemma}
\begin{proof}
    The Jacobian of $z^*$ reads $\diff z^*(x) = [\nabla^2_{11} G(z^*(x), x)]^{-1}\nabla^2_{12}G(z^*(x), x)$. By $\mu_G$-strong convexity and $L^G_1$-smoothness of $G$, we have $\|\diff z^*(x)\|\leq \frac{L^G_1}{\mu_G}$ which implies that $z^*$ is $L_{z^*}$-Lipschtiz with 
    $L_{z^*} = \frac{L^G_1}{\mu_G}$.

    For $v^*$ we do the computation directly:
    \begingroup
    \allowdisplaybreaks %
    \begin{align*}
        \|v^*(x_1) - v^*(x_2)\| &= \|[\nabla^2_{11}G(z^*(x_1),x_1)]^{-1}\nabla_1F(z^*(x_1), x_1) \\\nonumber
            &\qquad - [\nabla^2_{11}G(z^*(x_2),x_2)]^{-1}\nabla_1F(z^*(x_2), x_2)\|\\
        &\leq\|[\nabla^2_{11}G(z^*(x_1),x_1)]^{-1}(\nabla_1F(z^*(x_1), x_1) - \nabla_1F(z^*(x_2), x_2))\| \\\nonumber
            &\qquad + \|([\nabla^2_{11}G(z^*(x_1),x_1) - [\nabla^2_{11}G(z^*(x_2),x_2)]^{-1}]^{-1}\nabla_1F(z^*(x_2), x_2) \|\\
        &\leq \left(\frac{L^F_1}{\mu_G} + \frac{L^G_2L^F_0}{\mu_G^2}\right)\|(z^*(x_1),x_1) - (z^*(x_2),x_2) \|\\
        &\leq \left(\frac{L^F_1}{\mu_G} + \frac{L^G_2L^F_0}{\mu_G^2}\right)(\|z^*(x_1) - z^*(x_2)\| + \|x_1 - x_2\|)\\
        &\leq  \left(1+\frac{L^G_1}{\mu_G}\right)\left(\frac{L^F_1}{\mu_G} +  \frac{L^G_2L^F_0}{\mu_G^2}\right)\|x_1 - x_2\|
    \end{align*}
    \endgroup
    Then taking $L_{v^*} = \left(1+\frac{L^G_1}{\mu_G}\right)\left(\frac{L^F_1}{\mu_G} + \frac{L^G_2L^F_0}{\mu_G^2}\right)$ concludes the proof.
\end{proof}

\begin{lemma}\label{app:lemma:bound_error_directions}
    Let us consider the update directions $D_z^{t,k} = \bfdelta_z^{t,k}/\rho$, $D_v^{t,k} = \bfdelta_v^{t,k}/\rho$ and $D_x^{t,k} = \bfdelta_x^{t,k}/\gamma$ where $\bfdelta_z^{t,k}$, $\bfdelta_v^{t,k}$ and $\bfdelta_x^{t,k}$ verify Equations \eqref{eq:update_dz} to \eqref{eq:update_dx}. Then it holds
    \begingroup
    \allowdisplaybreaks %
    \begin{align*}
        \bbE[\|D_z^{t,k} - D_z(\bfu^{t,k})\|^2]&\leq \sum_{r=1}^k L^G_1(\rho^2\bbE[\|D^{t,r-1}_z\|^2] + \gamma^2\bbE[\|D^{t,r-1}_z\|^2])\\
        \bbE[\|D_v^{t,k} - D_v(\bfu^{t,k})\|^2]&\leq4\rho^2\left((L^G_2R)^2 + (L^F_1)^2\right)\sum_{r=1}^k\bbE[\|D^{t,r-1}_z\|^2] + 4\rho^2(L^G_1)^2\sum_{r=1}^k\bbE[\|\calG^{t,r-1}_v\|^2] \\\nonumber
        &\qquad + 4\gamma^2\left((L^G_2R)^2 + (L^F_1)^2\right)\sum_{r=1}^k\bbE[\|D_x^{t,r-1}\|^2]\\
        \bbE[\|D_x^{t,k} - D_x(\bfu^{t,k})\|^2]&\leq 4\rho^2\left((L^G_2R)^2 + (L^F_1)^2\right)\sum_{r=1}^k\bbE[\|D^{t,r-1}_z\|^2] + 4\rho^2(L^G_1)^2\sum_{r=1}^k\bbE[\|\calG^{t,r-1}_v\|^2] \\\nonumber
        &\qquad + 4\gamma^2\left((L^G_2R)^2 + (L^F_1)^2\right)\sum_{r=1}^k\bbE[\|D_x^{t,r-1}\|^2]\enspace.
    \end{align*}
    \endgroup
\end{lemma}
\begin{proof}
    \textbf{Direction $D_z$}
    
    We start from \cref{prop:bias_variance_decomp}.
    \begin{align*}
    \bbE[\|D^{t,k}_z - D_z(\bfu^{t,k})\|^2]&=\bbE[\|D^{t,k}_z - \nabla_1 G(z^{t,k}, x^{t,k})\|^2]\\
    &= \sum_{r=1}^k \bbE[\|D^{t,r}_z - D^{t,r-1}_z\|^2] - \sum_{r=1}^k \bbE[\|\nabla_1G(z^{t,r}, x^{t,r}) - \nabla_1G(z^{t,r-1}, x^{t,r-1})\|^2] \\
    &\leq \sum_{r=1}^k \bbE[\|D^{t,r}_z - D^{t,r-1}_z\|^2]\\
    &\leq \sum_{r=1}^k L^G_1(\rho^2\bbE[\|D^{t,r-1}_z\|^2] + \gamma^2\bbE[\|D^{t,r-1}_z\|^2])
    \end{align*}
    where the last inequality comes from the smoothness of each $G_i$.

    \textbf{Direction $D_v$}
    For $D_v$, the proof is almost the same. \cref{prop:bias_variance_decomp} gives us
    \begin{align*}
        \bbE[\| D^{t,k}_v - D_v(\bfu^{t,k})\|^2] \leq \sum_{r=1}^k \bbE[\|D^{t,r}_v - D^{t,r-1}_v\|^2]\enspace.
    \end{align*}
    Then, using the boundedness of $v$ and regularity of each $G_i$ and $F_j$, we have
    \begingroup
    \allowdisplaybreaks %
    \begin{align*}
        \bbE[\|D^{t,r}_v - D^{t,r-1}_v\|^2] &\leq 2 (\bbE[\|\nabla^2_{11}G_i(z^{t,r},x^{t,r})v^{t,r} - \nabla^2_{11}G_i(z^{t,r-1},x^{t,r-1})v^{t,r-1}\|^2] \\\nonumber
            &\quad + \bbE[\|\nabla_2 F_j(z^{t,r}, x^{t,r}) - \nabla_2 F_j(z^{t,r-1}, x^{t,r-1})\|^2])\\
        &\leq 4( \bbE[\|\nabla^2_{11}G_i(z^{t,r},x^{t,r})(v^{t,r} - v^{t,r-1})\|^2] \\\nonumber
            &\quad + \bbE[\|(\nabla^2_{11}G_i(z^{t,r},x^{t,r}) - \nabla^2_{11}G_i(z^{t,r-1},x^{t,r-1}))v^{t,r-1}\|^2]\\\nonumber
            &\quad + (L^F_1)^2(\gamma^2 \bbE[\|D^{t,r-1}_z\|] + \rho^2 \bbE[\|D^{t,r-1}_x\|^2]))\\
        &\leq 4( (L^G_1)^2\rho^2\bbE[\|\calG^{t,r-1}_v\|^2] \\\nonumber
            &\quad + (L^G_2)^2R^2(\rho^2 \bbE[\|D^{t,r-1}_z\|] + \gamma^2 \bbE[\|D^{t,r-1}_x\|^2])\\\nonumber
            &\quad + (L^F_1)^2(\rho^2 \bbE[\|D^{t,r-1}_z\|] + \gamma^2 \bbE[\|D^{t,r-1}_x\|^2]))\\
        &\leq 4\rho^2\left((L^G_2R)^2 + (L^F_1)^2\right)\bbE[\|D^{t,r-1}_z\|^2] + 4\rho^2(L^G_1)^2\bbE[\|\calG^{t,r-1}_v\|^2]  \\\nonumber
            &\qquad + 4\gamma^2\left((L^G_2R)^2 + (L^F_1)^2\right)\bbE[\|D_x^{t,r-1}\|^2]\enspace.
    \end{align*}
    \endgroup

    \textbf{Direction $D_x$}
    The proof is the same as the proof for $D_v$.
\end{proof}
\subsection{Proof of \cref{lemma:descent_z}}
Let $\phi_z(z, x) = G(z, x) - G(z^*(x), x)$ the inner suboptimality gap. The proof of \cref{lemma:descent_z} is based on the smoothness of $\phi_z$, which is the object of the following lemma.
\begin{lemma}\label{app:lemma:phi_z_smooth}
    The function $\phi_z$ has $\Lambda_z$-Lipschitz continuous gradient on $\bbR^p\times\bbR^d$, for some constant $\Lambda_z$.
\end{lemma}
\begin{proof}
For any $(z,x)\in\bbR^p\times\bbR^d$, we have
$$
\nabla_1\phi_z(z, x) = \nabla_1 G(z, x) \text{ and } \nabla_2\phi_z(z, x) = \nabla_2 G(z, x) - \nabla_2 G(z^*(x), x)\enspace.
$$

Let us consider $(z,x)\in\bbR^p\times\bbR^d$ and $(z',x')\in\bbR^p\times\bbR^d$. Since $\nabla G$ is $L^G_1$-Lipschitz continuous, we have directly
$$
\|\nabla_1 \phi_z(z,x) - \nabla_1 \phi_z(z',x')\| \leq L^G_1 \|(z,x)-(z',x')\|\enspace.
$$
Moreover, we have
\begin{align*}
    \|\nabla_2 \phi_z(z,x) - \nabla_2 \phi_z(z',x')\| &\leq \|\nabla_2 G(z,x) - \nabla_2 G(z',x')\| \\
        &\qquad + \|\nabla_2 G(z^*(x),x) - \nabla_2 G(z^*(x'),x')\|\\
    &\leq L^G_1 \|(z,x)-(z',x')\| + L^G_1\|(z^*(x),x)-(z^*(x'),x')\|\\
    &\leq L^G_1 \|(z,x)-(z',x')\| + L^G_1(\|z^*(x)-z^*(x')\| + \|x - x'\|)\enspace.
\end{align*}
From \cref{app:lemma:z_star_lipschitz}, $z^*$ is $L_*$ Lipschitz continuous, so
\begin{align*}
    \|\nabla_2 \phi_z(z,x) - \nabla_2 \phi_z(z',x')\| &\leq L^G_1 \|(z,x)-(z',x')\| + L^G_1(\|z^*(x)-z^*(x')\| + \|x - x'\|)\\
    &\leq  L^G_1 \|(z,x)-(z',x')\| + L^G_1(L_* + 1)\|x - x'\|\\
    &\leq  L^G_1(L_{z^*} + 2) \|(z,x)-(z',x')\|\enspace.
\end{align*}
As a consequence
\begin{align*}
    \|\nabla \phi_z(z,x) - \nabla\phi_z(z',x')\| &\leq \|\nabla_1 \phi_z(z,x) - \nabla_1 \phi_z(z',x')\|  + \|\nabla_2 \phi_z(z,x) - \nabla_2 \phi_z(z',x')\|\\
    &\leq L^G_1(L_{z^*} + 3) \|(z,x)-(z',x')\|\enspace.
\end{align*}
Hence, $\phi_z$ is $\Lambda_z$ smooth with $\Lambda_z = L^G_1(L_{z^*} + 3)$.
\end{proof}

We can now turn to the proof of \cref{lemma:descent_z}.

\begin{proof}
    The smoothness of $\phi_z$ provides us the following upper bound
    \begin{align}\label{app:eq:smoothness_phi_z}
        \phi_z(z^{t,k+1}, x^{t,k+1}) &\leq \phi_z(z^{t,k}, x^{t,k})  - \rho \langle D_z^{t,k}, \nabla_1 G(z^{t,k}, x^{t,k})\rangle + \frac{\Lambda_z}2\rho^2 \|D_z^{t,k}\|^2 \\\nonumber
        &\qquad - \gamma  \langle D_x^{t,k}, \nabla_2 G(z^{t,k}, x^{t,k}) - \nabla_2 G(z^*(x^{t,k}), x^{t,k})\rangle + \frac{\Lambda_z}2\gamma^2  \|D_x^{t,k}\|^2\enspace .
    \end{align}
    Using the equality $\langle a, b\rangle = \frac12(\|a\|^2 + \|b\|^2 - \|a-b\|^2)$, we get 
     \begin{align}\label{app:eq:1}
     - \langle D_z^{t,k}, \nabla_1 G(z^{t,k}, x^{t,k})\rangle + \frac{\Lambda_z}2\rho \|D_z^{t,k}\|^2 &= \frac12(\|D^{t,k}_z - \nabla_1 G(z^{t,k}, x^{t,k})\|^2 \\\nonumber
        &\qquad - \|\nabla_1 G(z^{t,k}, x^{t,k})\|^2 - \left(1 - \Lambda_z\rho\right)\|D^{t,k}_z\|^2)\enspace.
    \end{align}
    Plugging \cref{app:eq:1} into \cref{app:eq:smoothness_phi_z} and tacking the expectation conditionally to the past iterates yields
    \begin{align}\label{app:eq:2}
        \bbE_{t,k}[\phi_z^{t,k+1}] &\leq \phi_z^{t,k} + \frac\rho2\bbE_{t,k}[\|D^{t,k}_z - \nabla_1 G(z^{t,k}, x^{t,k})\|^2] \\\nonumber
            &\qquad - \frac\rho2\|\nabla_1 G(z^{t,k}, x^{t,k})\|^2 - \frac\rho2\left(1 - \Lambda_z\rho\right)\bbE_{t,k}[\|D^{t,k}_z\|^2] \\\nonumber
            &\qquad - \gamma  \langle \bbE_{t,k}[D_x^{t,k}], \nabla_2 G(z^{t,k}, x^{t,k}) - \nabla_2 G(z^*(x^{t,k}), x^{t,k})\rangle + \frac{\Lambda_z}2\gamma^2  \bbE_{t,k}[\|D_x^{t,k}\|^2]\enspace .
    \end{align}
    From Young inequality, we have for any $c>0$
    \begin{align}\label{app:eq:3}
         \langle \bbE_{t,k}[D_x^{t,k}], \nabla_2 G(z^{t,k}, x^{t,k}) - \nabla_2 G(z^*(x^{t,k}), x^{t,k})\rangle &\leq  \frac1{2c}\|\bbE_{t,k}[D_x^{t,k}]\|^2 \\\nonumber
            &\qquad + \frac{c}2\|\nabla_2 G(z^{t,k}, x^{t,k}) - \nabla_2 G(z^*(x^{t,k}), x^{t,k})\|^2
    \end{align}
    The smoothness of $G$ and strong convexity give us
    \begin{align}\label{app:eq:4}
    \|\nabla_2 G(z^{t,k}, x^{t,k}) - \nabla_2 G(z^*(x^{t,k}), x^{t,k})\|^2\leq L^G_1\|z^{t,k} - z^*(x^{t,k})\|^2 \leq \frac{2L^G_1}{\mu_G}\phi_z(z^{t,k}, x^{t,k})
    \end{align}
    Let us denote $L' = \frac{L^G_1}{\mu_G}$. Plugging Inequalities \eqref{app:eq:3} and \eqref{app:eq:4} into \cref{app:eq:2} yields
    \begin{align}\label{app:eq:41}
       \bbE_{t,k}[\phi_z(z^{t,k+1}, x^{t,k+1})] &\leq (1 + cL'\gamma)\phi_z(z^{t,k+1}, x^{t,k+1}) - \frac\rho2\bbE_{t,k}[\|\nabla_1 G(z^{t,k}, x^{t,k})\|^2] \\\nonumber
        &\qquad +\frac\rho2\bbE_{t,k}[\|D^{t,k}_z - \nabla_1 G(z^{t,k}, x^{t,k})\|^2] - \frac\rho2\left(1 - \Lambda_z\rho\right)\bbE_{t,k}[\|D^{t,k}_z\|^2]\\\nonumber
        &\qquad + \frac\gamma{2c}\|\bbE_{t,k}[D^{t,k}_x]\|^2  + \frac{\Lambda_z}2\gamma^2\bbE_{t,k}[\|D^{t,k}_x\|^2]
    \end{align}

    From \cref{app:lemma:bound_error_directions}, we have
    \begin{align*}
    \bbE[\|D^{t,k}_z - \nabla_1 G(z^{t,k}, x^{t,k})\|^2] &\leq \sum_{r=1}^k L^G_1(\rho^2\bbE[\|D^{t,r-1}_z\|^2] + \gamma^2\bbE[\|D^{t,r-1}_z\|^2])\enspace.
    \end{align*}

    Taking the total expectation and plugging the previous inequality into \cref{app:eq:41} yields
    \begin{align}\label{app:eq:5}
       \phi_z^{t,k+1} &\leq (1 + cL'\gamma)\phi^{t,k} + \frac{L^G_1}2\sum_{r=1}^k (\rho^3\bbE[\|D^{t,r-1}_z\|^2] + \gamma^2\rho\bbE[\|D^{t,r-1}_x\|^2])\\\nonumber
        &\qquad  - \frac\rho2\bbE[\|\nabla_1 G(z^{t,k}, x^{t,k})\|^2] - \frac\rho2\left(1 - \Lambda_z\rho\right)\bbE[\|D^{t,k}_z\|^2] \\\nonumber
        &\qquad + \frac\gamma{2c}\bbE[\|\bbE[D^{t,k}_x]\|^2]  + \frac{\Lambda_z}2\gamma^2\bbE[\|D^{t,k}_x\|^2]
    \end{align}

    Since $G$ is $\mu_G$-strongly convex with respect to $z$, Polyak-\L ojasiewicz inequality holds: 
    $$
    \|\nabla_1 G(z^{t,k},x^{t,k})\|^2 \geq 2\mu_G\phi_z(z^{t,k}, x^{t,k})
    $$
    As a consequence, \cref{app:eq:5} becomes
    \begin{align*}
       \phi_z^{t,k+1} &\leq \left(1 + cL'\gamma - \mu_G\rho\right)\phi^{t,k} + \frac{L^G_1}2\sum_{r=1}^k (\rho^3\bbE[\|D^{t,r-1}_z\|^2] + \gamma^2\rho\bbE[\|D^{t,r-1}_x\|^2])  \\\nonumber
        &\qquad - \frac\rho2\left(1 - \Lambda_z\rho\right)\bbE[\|D^{t,k}_z\|^2] + \frac\gamma{2c}\bbE[\|\bbE[D^{t,k}_x]\|^2] + \frac{\Lambda_z}2\gamma^2\bbE[\|D^{t,k}_x\|^2]
    \end{align*}

    Taking $c = \frac{\mu_G\rho}{2L'\gamma}$ yields
    \begin{align*}
       \phi_z^{t,k+1} &\leq \left(1 - \frac{\mu_G}2\rho\right)\phi^{t,k} + \frac{L^G_1}2\sum_{r=1}^k (\rho^3\bbE[\|D^{t,r-1}_z\|^2] + \gamma^2\rho\bbE[\|D^{t,r-1}_x\|^2])  \\\nonumber
        &\qquad - \frac\rho2\left(1 - \Lambda_z\rho\right)\bbE[\|D^{t,k}_z\|^2] + \frac{L'}{\mu_G}\frac{\gamma^2}{\rho}\bbE[\|\bbE[D^{t,k}_x]\|^2] + \frac{\Lambda_z}2\gamma^2\bbE[\|D^{t,k}_x\|^2]
    \end{align*}

    For the term $\bbE[\|\bbE_{t,k}[D^{t,k}_z]\|^2]$, we have
    \begin{align}
        \nonumber\bbE[\|\bbE_{t,k}[D^{t,k}_x]\|^2] &= \bbE[\|D_x(z^{t,k},v^{t,k}, x^{t,k}) - D_x(z^{t,k-1},v^{t,k-1}, x^{t,k-1}) + D^{t,k-1}_x \|^2]\\
        \nonumber&= \bbE[\|D_x(z^{t,k},v^{t,k}, x^{t,k}) - D_x(z^{t,k-1},v^{t,k-1}, x^{t,k-1}) - \bbE[D^{t,k-1}_x] \|^2] \\\nonumber
            &\quad + \bbE[\| D^{t,k-1}_x -\bbE[D^{t,k-1}_x]\|^2]\\
        \label{app:eq:exp_norm_exp_d_x} &= \bbE[\|D_x(z^{t,k},v^{t,k}, x^{t,k})\|^2] \\\nonumber
            &\qquad + \bbE[\| D^{t,k-1}_x - D_x(z^{t,k-1},v^{t,k-1}, x^{t,k-1})\|^2]\enspace .
    \end{align}

    Using \cref{app:lemma:bound_error_directions}, we get
    \begin{align*}
        \bbE[\| D^{t,k-1}_x - D_x(\bfu^{t,k-1})\|^2]&\leq 4\rho^2\left((L^G_2R)^2 + (L^F_1)^2\right)\sum_{r=1}^{k-1}\bbE[\|D^{t,r-1}_z\|^2] \\\nonumber
            &\qquad + 4\rho^2(L^G_1)^2\sum_{r=1}^{k-1}\bbE[\|\calG^{t,r-1}_v\|^2]  \\\nonumber
            &\qquad + 4\gamma^2\left((L^G_2R)^2 + (L^F_1)^2\right)\sum_{r=1}^{k-1}\bbE[\|D_x^{t,r-1}\|^2]\enspace.
    \end{align*}

    Putting all together yields
    \begin{align}
       \phi_z^{t,k+1} &\leq \left(1 - \frac{\mu_G}2\rho\right)\phi^{t,k} - \frac\rho2\left(1 - \Lambda_z\rho\right)\bbE[\|D^{t,k}_z\|^2]  + \frac{\Lambda_z}2\gamma^2\bbE[\|D^{t,k}_x\|^2]\\\nonumber
        &\qquad + \frac{L'}{\mu_G}\frac{\gamma^2}{\rho}\bbE[\|D^{t,k}_x(\bfu^{t,k})\|^2] +  4(L^G_1)^2\frac{L'}{\mu_G}\gamma^2\rho\sum_{r=1}^k\bbE[\|\calG^{t,r-1}_v\|^2] \\\nonumber
        &\qquad + \rho\left[\rho^2\frac{L^G_1}2+\frac{4(L^G_2R)^2L'}{\mu_G}\gamma^2 + \frac{4(L^F_1)^2L'}{\mu_G}\gamma^2\right]\sum_{r=1}^k\bbE[\|D^{t,r-1}_z\|^2] \\\nonumber
        &\qquad + \gamma^2\left[\rho\frac{L^G_1}2+4(L^G_2R)^2\frac{L'}{\mu_G}\frac{\gamma^2}{\rho} + 4(L^F_1)^2\frac{L'}{\mu_G}\frac{\gamma^2}{\rho}\right]\sum_{r=1}^k\bbE[\|D^{t,r-1}_x\|^2] 
    \end{align}

    By assumption, $\gamma \leq C_z\rho$, with $C_z = \sqrt{\frac{\mu_GL^G_1}{8L'((L^G_2R)^2 + (L^F_1)^2)}}$ therefore

    \begin{align*}
       \phi_z^{t,k+1} &\leq \left(1 - \frac{\mu_G}2\rho\right)\phi^{t,k}_z - \frac\rho2\left(1 - \Lambda_z\rho\right)\bbE[\|D^{t,k}_z\|^2]  + \frac{\Lambda_z}2\gamma^2\bbE[\|D^{t,k}_x\|^2]\\\nonumber
            &\qquad + \frac{L'}{\mu_G}\frac{\gamma^2}{\rho}\bbE[\|D^{t,k}_x(\bfu^{t,k})\|^2] +  \rho^3L^G_1\sum_{r=1}^k\bbE[\|D^{t,r-1}_z\|^2] \\\nonumber
            &\qquad + 4(L^G_1)^2\frac{L'}{\mu_G}\gamma^2\rho\sum_{r=1}^k\bbE[\|\calG^{t,r-1}_v\|^2] + \gamma^2\rho L^G_1\sum_{r=1}^k\bbE[\|D^{t,r-1}_x\|^2]\\
        &\leq \left(1 - \frac{\mu_G}2\rho\right)\phi^{t,k} - \frac\rho2\left(1 - \Lambda_z\rho\right)V_z^{t,k}  + \frac{\Lambda_z}2\gamma^2V_x^{t,k} + \overline{\beta}_{zx}\frac{\gamma^2}{\rho}\bbE[\|D^{t,k}_x(\bfu^{t,k})\|^2]\\\nonumber
            &\qquad +  \rho^3\beta_{zz}\calV_z^{t,k} + \gamma^2\rho\beta_{zv}\calV_v^{t,k}  + \gamma^2\rho \beta_{zx}\calV_x^{t,k} 
    \end{align*}
    with $\beta_{zz} = L^G_1$, $\beta_{zv} = \frac{4(L^G_1)^2L'}{\mu_G}$, $\beta_{zx} = L^G_1$ and $\overline{\beta}_{zx} = \frac{L'}{\mu_G}$.
\end{proof}

\subsection{Proof of \cref{lemma:descent_v}}
Recall that we denote $\Psi(z,v,x) = \frac12v^\top \nabla^2_{11} G(z,x)v + \nabla_1 F(z,x)^\top v$ and $\phi_v(v,x) = \Psi(z^*(x),v,x) - \Psi(z^*(x),v^*(x),x)$. As for \cref{lemma:descent_z}, the key property we need is the smoothness of $\phi_v$. 
The derivatives of $\phi_v$ involve the third derivative of $G$. For a tensor $T\in\bbR^{p_1\times p_2\times p_3}$ and a vector $a\in\bbR^{p_3}$ we denote $(T|a)$ the matrix in $\bbR^{p_1\times p_2}$ defined by:
\begin{align*}
    (T|a) = \left[\sum_{k=1}^{p_3}T_{i,j,k}a_k\right]_{\substack{1\leq i\leq p_1\\1\leq j\leq p_2}}\enspace.
\end{align*}
\begin{lemma}
    The function $\phi_v$ has $\Lambda_v$-Lipschitz continuous gradient on $\Gamma\times\bbR^d$, for some constant $\Lambda_v$.
\end{lemma}
\begin{proof}
    For any $(v,x)\in \Gamma\times\bbR^d$, we have
    $$
    \nabla_1 \phi_v(v,x) = D_v(z^*(x), v, x)
    $$
    and
    \begin{align*}
        \nabla_2 \phi_v(v,x) &= (\diff z^*(x))^\top \left[\frac12 (\nabla^3_{111}G(z^*(x),x)|v)v  - \frac12 (\nabla^3_{111}G(z^*(x),x)|v^*(x))v^*(x) \right.\\\nonumber
            &\qquad \left. + \nabla_{11}^2 F(z^*(x),x) v - \nabla_{11}^2 F(z^*(x),x) v^*(x)\right] \\\nonumber
            &\qquad  + \left[\frac12 (\nabla^3_{211}G(z^*(x),x)|v)v  - \frac12 (\nabla^3_{211}G(z^*(x),x)|v^*(x))v^*(x) \right.\\\nonumber
            &\qquad \left. + \nabla_{21}^2 F(z^*(x),x) v - \nabla_{21}^2 F(z^*(x),x) v^*(x)\right]\enspace.
    \end{align*}

Let us consider $(v,x)\in \Gamma\times\bbR^d$ and $(v',x')\in \Gamma\times\bbR^d$.
We have
\begin{align*}
    \|\nabla_1 \phi_v(v,x) - \nabla_1 \phi_v(v',x') \| &\leq \|\nabla_{11}^2 G(z^*(x),x)v - \nabla_{11}^2 G(z^*(x'),x')v'\| \\\nonumber
    &\qquad + \|\nabla_1 F(z^*(x),x) - \nabla_1 F(z^*(x'),x')\|
\end{align*}
For the first term, 
\begin{align*}
     \|\nabla_{11}^2 G(z^*(x),x)v - \nabla_{11}^2 G(z^*(x'),x')v'\| &\leq  \|\nabla_{11}^2 G(z^*(x),x)(v-v')\| \\\nonumber
        &\qquad + \|(\nabla_{11}^2 G(z^*(x),x) - \nabla_{11}^2 G(z^*(x'),x'))v'\| \\\nonumber
        &\qquad + \|\nabla_{11}^2 G(z^*(x'),x')(v-v')\| \\
     &\leq 2L^G_1\|v-v'\| + L^G_2(L_{z^*} + 1)\|v'\|\|x - x'\|\\
     &\leq [2 L^G_1 +  L^G_2 (L_{z^*} + 1)R]\|(v,x)-(v',x')\|
\end{align*}
For the second terms, we use the smoothness of $F$ and the Lipschitz continuity of $z^*$ (\cref{app:lemma:z_star_lipschitz}):
\begin{align*}
    \|\nabla_1 F(z^*(x),x) - \nabla_1 F(z^*(x'),x')\| &\leq L^F_1 \|(z^*(x),x) - (z^*(x'),x')\|\\
    &\leq L^F_1 (\|z^*(x) - z^*(x')\| + \|x - x'\|)\\
    &\leq L^F_1(L_{z^*} + 1 )\|x - x'\|\\
    &\leq L^F_1(L_{z^*} + 1 )\|(x, v) - (x', v')\|\enspace.
\end{align*}
As a consequence
\begin{equation}
    \|\nabla_1 \phi_v(v,x) - \nabla_1 \phi_v(v',x') \| \leq \Lambda_1 \|(v,x)-(v',x')\|
\end{equation}

with 
\begin{equation}
    \Lambda_1 = L^F_1(L_{z^*} + 1 ) + 2 L^G_1 +  L^G_2 (L_{z^*} + 1)R\enspace.
\end{equation}

To prove the Lipschitz continuity of $\nabla_2\phi_v$, we remark that $\nabla^3_{111} G$, $\nabla^3_{211} G$ are Lipschitz and bounded by assumption. $(v\mapsto v)$ is Lipschitz and bounded on $\Gamma$. Also by \cref{app:lemma:z_star_lipschitz}, $z^*$ and $v^*$ are Lipschitz and bounded. Finally, $\diff z^*$ is bounded (\cref{app:lemma:z_star_lipschitz}) and Lipschitz according to \citet{Chen2021alset}[Lemma 9]. As a consequence, $\nabla_2\phi_v$ is $\Lambda_2$-Lpischitz for some constant $\Lambda_2>0$. 
Hence, $\nabla \phi_v$ is $\Lambda_v$-Lipschitz continuous with $\Lambda_v = \Lambda_1 + \Lambda_2$.

\end{proof}
\begin{lemma}\label{app:lemma:decrease_iota} 
Let $t>0$. For $k\in\setcomb{q-1}$, we have
    $$0\leq -\left\langle\frac1\rho(v^{t,k+1}-v^{t,k}) + D^{t,k}_v, v^{t,k+1}-v^{t,k}\right\rangle$$
\end{lemma}
\begin{proof}
    The function $\iota_\Gamma$ being convex (since $\Gamma$ is convex), let us consider its sub-differential
    $$
    \partial \iota_\gamma(v) = \{\eta\in\bbR^p, \forall v'\in\bbR^p, \iota_\Gamma(v')\geq \iota_\Gamma(v) + \langle \eta, v'-v\rangle\}
    $$

    By definition
    \begin{align*}
        v^{t,k+1} = \argmin_v (\iota_\Gamma(v) + \frac1{2\rho}\|v-(v^{t,k} - \rho D^{t,k}_v)\|^2)\enspace.
    \end{align*}
    
    Using Fermat's rule, we get
    \begin{align*}
        -\frac1\rho(v^{t,k+1}-v^{t,k}) - D^{t,k}_v\in\partial \iota_\Gamma(v^{t,k+1})\enspace.
    \end{align*}

    We can use the definition of the sub-differential with $\eta = -\frac1\rho(v^{t,k+1}-v^{t,k}) - D^{t,k}_v$ to get
    \begin{align*}\label{app:eq:decrease_iota}
        \underbrace{\iota_\Gamma(v^{t,k+1})}_{=0} \leq \underbrace{\iota_\Gamma(v^{t,k})}_{=0} -  \left\langle \frac1\rho(v^{t,k+1}-v^{t,k}) + D^{t,k}_v, v^{t,k+1}-v^{t,k}\right\rangle\enspace.
    \end{align*}
\end{proof}
We can now turn to the proof of \cref{lemma:descent_v}.
\begin{proof}
    The smoothness of $\phi_v$ provides us the following upper bound
    \begin{align}
        \phi_v(v^{t,k+1}, x^{t,k+1}) &\leq \phi_v(v^{t,k}, x^{t,k})  + \langle \Pi_\Gamma(v^{t,k} - \rho D_v^{t,k}) - v^{t,k}, D_v(z^*(x^{t,k}), v^{t,k}, x^{t,k})\rangle \\\nonumber
        &\qquad + \frac{\Lambda_v}2\rho^2 \|\Pi_\Gamma(v^{t,k} - \rho D_v^{t,k}) - v^{t,k}\|^2 \\\nonumber
        &\qquad - \gamma  \langle D_x^{t,k}, \nabla_2 \phi_v(v^{t,k}, x^{t,k})\rangle + \frac{\Lambda_v}2\gamma^2  \|D_x^{t,k}\|^2\enspace .
    \end{align}
Let us denote $\Delta_\Pi^{t,k} = \Pi_\Gamma(v^{t,k} - \rho D_v^{t,k}) - \Pi_\Gamma(v^{t,k} - \rho D_v(z^*(x^{t,k}), v^{t,k}, x^{t,k}))$.
Adding and subtracting 
    \begin{align*}
        &\langle \Pi_\Gamma(v^{t,k} - \rho D_v(z^*(x^{t,k}),v^{t,k}, x^{t,k}) - v^{t,k}, D_v(z^*(x^{t,k}), v^{t,k}, x^{t,k})\rangle \\\nonumber
            &\qquad + \frac{\Lambda_v}2\|\Pi_\Gamma(v^{t,k} - \rho D_v(z^*(x^{t,k}),v^{t,k}, x^{t,k}) - v^{t,k}\|^2
    \end{align*}
    yields
    \begin{align}
        \phi_v(v^{t,k+1}, x^{t,k+1}) &\leq  \phi_v(v^{t,k}, x^{t,k}) + \langle \Delta_\Pi^{t,k}, D_v(z^*(x^{t,k}), v^{t,k}, x^{t,k})\rangle\\\nonumber
            &\qquad + \frac{\Lambda_v}2 \|\Pi_\Gamma(v^{t,k}  - \rho D_v(z^*(x^t), v^{t,k}, x^{t,k}))-v^{t,k}\|^2\\\nonumber
            &\qquad + \langle \Pi_\Gamma(v^{t,k} - \rho D_v(z^*(x^{t,k}), v^{t,k}, x^{t,k}))-v^{t,k}, D_v(z^*(x^{t,k}), v^{t,k}, x^{t,k})\rangle  \\\nonumber
            &\qquad + \frac{\Lambda_v}2\|\Delta_\Pi^{t,k}\|^2 + \Lambda_v\langle \Delta_\Pi^{t,k}, \Pi_\Gamma(v^{t,k}  - \rho D_v(z^*(x^{t,k}), v^{t,k}, x^{t,k}))-v^{t,k}\rangle\\\nonumber
            &\qquad - \gamma  \langle D_x^{t,k}, \nabla_2 \phi_v(v^{t,k}, x^{t,k})\rangle + \frac{\Lambda_v}2\gamma^2  \|D_x^{t,k}\|^2\enspace.
    \end{align}
    Taking $\rho\leq\frac1\Gamma_v$ gives
    \begin{align}\label{app:eq:smoothness_phi_v}
        \phi_v(v^{t,k+1}, x^{t,k+1}) &\leq  \phi_v(v^{t,k}, x^{t,k}) + \langle \Delta_\Pi^{t,k}, D_v(z^*(x^{t,k}), v^{t,k}, x^{t,k})\rangle\\\nonumber
            &\qquad + \frac{1}{2\rho} \|\Pi_\Gamma(v^{t,k}  - \rho D_v(z^*(x^t), v^{t,k}, x^{t,k}))-v^{t,k}\|^2\\\nonumber
            &\qquad + \langle \Pi_\Gamma(v^{t,k} - \rho D_v(z^*(x^{t,k}), v^{t,k}, x^{t,k}))-v^{t,k}, D_v(z^*(x^{t,k}), v^{t,k}, x^{t,k})\rangle  \\\nonumber
            &\qquad + \frac{\Lambda_v}2\|\Delta_\Pi^{t,k}\|^2 + \Lambda_v\langle \Delta_\Pi^{t,k}, \Pi_\Gamma(v^{t,k}  - \rho D_v(z^*(x^{t,k}), v^{t,k}, x^{t,k}))-v^{t,k}\rangle\\\nonumber
            &\qquad - \gamma  \langle D_x^{t,k}, \nabla_2 \phi_v(v^{t,k}, x^{t,k})\rangle + \frac{\Lambda_v}2\gamma^2  \|D_x^{t,k}\|^2\enspace.
    \end{align}
    Let $\iota_\Gamma$ the indicator function of the convex set $\Gamma$. Similarly to \citet[Equation 13]{Karimi2016PL_linear_cvg} we define for any $\alpha>0$ and $v\in\bbR^p$
    $$
        \calD_{\iota_\Gamma}(v, x, \alpha) = -2\alpha\min_{v'\in\bbR^p}\left[\langle\nabla_1\phi_v(v,x), v'-v\rangle + \frac\alpha2\|v'-v\|^2 + \iota_\Gamma(v') - \iota_\Gamma(v)\right]\enspace.
    $$
    Hence, for $v\in\Gamma$ and $x\in\bbR^d$, we have
    \begin{align*}
        -\frac\rho2\calD_{\iota_\Gamma}\left(v, x, \frac1\rho\right) &= \langle \Pi_\Gamma(v - \rho D_v(z^*(x), v, x))-v, D_v(z^*(x), v, x)\rangle \\
            &\qquad + \frac1{2\rho}\|\Pi_\Gamma(v  - \rho D_v(z^*(x), v, x))-v\|^2\enspace.
    \end{align*}
    Therefore, \cref{app:eq:smoothness_phi_v} can be written as
    \begin{align*}
        \phi_v(v^{t,k+1}, x^{t,k+1}) &\leq  \phi_v(v^{t,k}, x^{t,k}) - \frac\rho2\calD_{\iota_\Gamma}\left(v^{t,k}, x^{t,k}, \frac1\rho\right)\\\nonumber
            &\qquad + \langle \Delta_\Pi^{t,k}, D_v(z^*(x^{t,k}), v^{t,k}, x^{t,k})\rangle\\\nonumber
            &\qquad + \frac{\Lambda_v}2\|\Delta_\Pi^{t,k}\|^2 + \Lambda_v\langle \Delta_\Pi^{t,k}, \Pi_\Gamma(v^{t,k}  - \rho D_v(z^*(x^{t,k}), v^{t,k}, x^{t,k}))-v^{t,k}\rangle\\\nonumber
            &\qquad - \gamma  \langle D_x^{t,k}, \nabla_2 \phi_v(v^{t,k}, x^{t,k})\rangle + \frac{\Lambda_v}2\gamma^2  \|D_x^{t,k}\|^2\enspace.
    \end{align*}
    By strong convexity of $\phi_v$ with respect top $v$ and smoothness, we have $\calD_{\iota_{\Gamma}}(v^{t,k},x^{t,k},\Lambda_v)~\geq~2\mu_G\phi_v(v^{t,k},x^{t,k})$. According to \citet[Lemma 1]{Karimi2016PL_linear_cvg}, $\calD_{\iota_{\Gamma}}(v^{t,k},x^{t,k},\bullet)$ is an increasing function. As a consequence, since $\Lambda_v\leq\frac1\rho$, we have $\calD_{\iota_{\Gamma}}\left(v^{t,k},x^{t,k},\frac1\rho\right)\geq 2\mu_G\phi_v(v^{t,k},x^{t,k})$. This leads to
    \begin{align}\label{app:eq:phi_v_2}
        \phi_v(v^{t,k+1}, x^{t,k+1}) &\leq  (1-\rho\mu_G)\phi_v(v^{t,k}, x^{t,k}) + \langle \Delta_\Pi^{t,k}, D_v(z^*(x^{t,k}), v^{t,k}, x^{t,k})\rangle \\\nonumber
            &\qquad + \frac{\Lambda_v}2\|\Delta_\Pi^{t,k}\|^2 + \Lambda_v\langle \Delta_\Pi^{t,k}, \Pi_\Gamma(v^{t,k}  - \rho D_v(z^*(x^{t,k}), v^{t,k}, x^{t,k}))-v^{t,k}\rangle\\\nonumber
            &\qquad - \gamma  \langle D_x^{t,k}, \nabla_2 \phi_v(v^{t,k}, x^{t,k})\rangle + \frac{\Lambda_v}2\gamma^2  \|D_x^{t,k}\|^2\enspace.
    \end{align}
    The non-expansiveness of $\Pi_\Gamma$ yields
    \begin{align}\label{app:eq:bound_delta_pi}
        \|\Delta_\Pi^{t,k}\|\leq \rho\|D^{t,k}_v - D_v(z^*(x^{t,k}),v^{t,k}, x^{t,k})\|
    \end{align}
    and 
    \begin{align}
        \nonumber\|\Pi_\Gamma(v^{t,k}-\rho D_v(z^*(x^{t,k}), v^{t,k}, x^{t,k}) - \underbrace{v^{t,k}}_{\in\Gamma}\| &= \|\Pi_\Gamma(v^{t,k}-\rho D_v(z^*(x^{t,k}), v^{t,k}, x^{t,k})) - \Pi_\Gamma(v^{t,k})\|\\
        &\leq \rho\|D_v(z^*(x^{t,k}), v^{t,k}, x^{t,k})\|\enspace.\label{app:eq:pi_minus_v}
    \end{align}
    Moreover, using \cref{app:eq:bound_delta_pi} and Young Inequality, we have for any $c>0$
    \begin{align}
        \nonumber \langle \Delta_\Pi^{t,k}, D_v(z^*(x^{t,k}), v^{t,k}, x^{t,k})\rangle &\leq \frac{c}2\|\Delta_\Pi\|^2 + \frac1{2c}\|D_v(z^*(x^{t,k}), v^{t,k}, x^{tk,})\|^2\\\nonumber
        &\leq \frac{c\rho^2}2\|D^{t,k}_v - D_v(z^*(x^{t,k}), v^{t,k}, x^{t,k})\|^2 \\\nonumber
            &\qquad + \frac1{2c}\|D_v(z^*(x^{t,k}), v^{t,k}, x^{tk,}) - \underbrace{D_v(z^*(x^{t,k}), v^*(x^{t,k}), x^{tk,})}_{=0}\|^2\\
        \label{app:eq:delta_D}&\leq \frac{c\rho^2}2\|D^{t,k}_v - D_v(z^*(x^{t,k}), v^{t,k}, x^{t,k})\|^2 \\\nonumber
            &\qquad + \frac{L^G_1}{\mu_G c}\phi_v(v^{t,k},x^{t,k})
    \end{align}
    Plugging \cref{app:eq:delta_D} into \cref{app:eq:phi_v_2} with $c = \frac{2L^G_1}{\mu_G^2\rho}$ yields
    \begin{align}\label{app:eq:phi_v_3}
        \phi_v(v^{t,k+1}, x^{t,k+1}) &\leq  \left(1-\frac{\rho\mu_G}2\right)\phi_v(v^{t,k}, x^{t,k}) + \frac{L^G_1\rho}{\mu_G^2}\|D^{t,k}_v - D_v(z^*(x^{t,k}), v^{t,k}, x^{t,k})\|^2 \\\nonumber
            &\qquad + \frac{\Lambda_v}2\|\Delta_\Pi^{t,k}\|^2 + \Lambda_v\langle \Delta_\Pi^{t,k}, \Pi_\Gamma(v^{t,k}  - \rho D_v(z^*(x^{t,k}), v^{t,k}, x^{t,k}))-v^{t,k}\rangle\\\nonumber
            &\qquad - \gamma  \langle D_x^{t,k}, \nabla_2 \phi_v(v^{t,k}, x^{t,k})\rangle + \frac{\Lambda_v}2\gamma^2  \|D_x^{t,k}\|^2\enspace.
    \end{align}
    Using \cref{app:eq:bound_delta_pi}, \cref{app:eq:pi_minus_v} and Young Inequality for $d>0$ yields
    \begin{align}
        \nonumber\langle \Delta_\Pi^{t,k}, \Pi_\Gamma(v^{t,k} - \rho D_v(z^*(x^{t,k}), v^{t,k}, x^{t,k}))-v^{t,k}\rangle &\leq \frac{d}2\|\Delta_\Pi^{t,k}\|^2 \\\nonumber
            &\qquad + \frac1{2d}\|\Pi_\Gamma(v^{t,k} - \rho D_v(z^*(x^{t,k}), v^{t,k}, x^{t,k}))-v^{t,k}\|^2\\
        &\leq \frac{d\rho^2}2\|D^{t,k}_v - D_v(z^*(x^{t,k}), v^{t,k}, x^{t,k})\|^2 \\\nonumber
            &\qquad + \frac{\rho^2}{2d}\|D_v(z^*(x^{t,k}), v^{t,k}, x^{t,k})\|^2\\
        &\leq \frac{d\rho^2}2\|D^{t,k}_v - D_v(z^*(x^{t,k}), v^{t,k}, x^{t,k})\|^2 \label{app:eq:delta_pi_minus_v}\\\nonumber
            &\qquad + \frac{L_1^G\rho^2}{\mu_Gd}\phi_v(v^{t,k}, x^{t,k})\enspace.
    \end{align}
    Plugging \cref{app:eq:delta_pi_minus_v} into \cref{app:eq:phi_v_3} with $d = \frac{4L^G_1\Lambda_v)\rho}{\mu_G^2}$ gives
    \begin{align}
        \phi_v(v^{t,k+1}, x^{t,k+1}) &\leq  \left(1-\frac{\rho\mu_G}4\right)\phi_v(v^{t,k}, x^{t,k}) \\\nonumber
            &\qquad + \left[\frac{L^G_1\rho}{\mu_G^2} + \frac{2L^G_1\Lambda_v^2\rho^3}{\mu_G^2}\right]\|D^{t,k}_v - D_v(z^*(x^{t,k}), v^{t,k}, x^{t,k})\|^2 \\\nonumber
            &\qquad + \frac{\Lambda_v}2\|\Delta_\Pi^{t,k}\|^2 \\\nonumber
            &\qquad - \gamma  \langle D_x^{t,k}, \nabla_2 \phi_v(v^{t,k}, x^{t,k})\rangle + \frac{\Lambda_v}2\gamma^2  \|D_x^{t,k}\|^2\enspace.
    \end{align}
    Using once again \eqref{app:eq:bound_delta_pi}, we get
    \begin{align}\label{app:eq:phi_v_4}
        \phi_v(v^{t,k+1}, x^{t,k+1}) &\leq  \left(1-\frac{\rho\mu_G}4\right)\phi_v(v^{t,k}, x^{t,k}) \\\nonumber
            &\qquad + \left[\frac{L^G_1\rho}{\mu_G^2} + \frac{2L^G_1\Lambda_v^2\rho^3}{\mu_G^2}+\frac{\Lambda_v\rho^2}2\right]\|D^{t,k}_v - D_v(z^*(x^{t,k}), v^{t,k}, x^{t,k})\|^2 \\\nonumber
            &\qquad - \gamma  \langle D_x^{t,k}, \nabla_2 \phi_v(v^{t,k}, x^{t,k})\rangle + \frac{\Lambda_v}2\gamma^2  \|D_x^{t,k}\|^2\enspace.
    \end{align}
    By \cref{app:lemma:decrease_iota}, we have for any $\alpha>0$
    \begin{align*}
        0\leq -\alpha\left\langle\frac1\rho(v^{t,k+1} - v^{t,k}) + D^{t,k}_v, v^{t,k+1}-v^{t,k}\right\rangle\enspace.
    \end{align*}
    By adding this to \cref{app:eq:phi_v_4}, we get
    \begin{align}\label{app:eq:phi_v_10}
        \phi_v(v^{t,k+1}, x^{t,k+1}) &\leq  \left(1-\frac{\rho\mu_G}4\right)\phi_v(v^{t,k}, x^{t,k}) \\\nonumber
            &\qquad - \frac\alpha\rho\|v^{t,k+1}-v^{t,k}\|^2 - \alpha \langle D^{t,k}_v, v^{t,k+1}-v^{t,k}\rangle \\\nonumber
            &\qquad + \left[\frac{L^G_1\rho}{\mu_G^2} + \frac{2L^G_1\Lambda_v^2\rho^3}{\mu_G^2}+\frac{\Lambda_v\rho^2}2\right]\|D^{t,k}_v - D_v(z^*(x^{t,k}), v^{t,k}, x^{t,k})\|^2 \\\nonumber
            &\qquad - \gamma  \langle D_x^{t,k}, \nabla_2 \phi_v(v^{t,k}, x^{t,k})\rangle + \frac{\Lambda_v}2\gamma^2  \|D_x^{t,k}\|^2\enspace.
    \end{align}

    We can control $-\left\langle D^{t,k}_v, v^{t,k+1}-v^{t,k}\right\rangle$ by Cauchy-Schwarz and Young for some $c,d,e,f>0$
    \begin{align*}
         - \left\langle D^{t,k}_v, v^{t,k+1}-v^{t,k}\right\rangle &=  -  \left\langle D_v(z^*(x^{t,k}),v^{t,k}, x^{t,k}), \Pi_\Gamma(v^{t,k}-\rho D_v(z^*(x^{t,k}),v^{t,k}, x^{t,k})-v^{t,k}\right\rangle \\\nonumber
             &\qquad -  \left\langle D_v(z^*(x^{t,k}),v^{t,k}, x^{t,k}), \Delta_\Pi^{t,k}\right\rangle\\\nonumber
             &\qquad -  \left\langle D^{t,k}_v - D_v(z^*(x^{t,k}),v^{t,k}, x^{t,k}), \Pi_\Gamma(v^{t,k}-\rho D_v(z^*(x)))-v^{t,k}\right\rangle\\\nonumber
             &\qquad -  \left\langle D^{t,k}_v - D_v(z^*(x^{t,k}),v^{t,k}, x^{t,k}), \Delta_\Pi^{t,k}\right\rangle\\
         &\leq \frac{c}2\|D_v(z^*(x^{t,k}),v^{t,k}, x^{t,k})\|^2 \\\nonumber
            &\qquad +\frac1{2c}\| \Pi_\Gamma(v^{t,k}-\rho D_v(z^*(x^{t,k}),v^{t,k}, x^{t,k}))-v^{t,k}\|^2 \\\nonumber
            &\qquad + \frac{d}2 \| D_v(z^*(x^{t,k}),v^{t,k}, x^{t,k})\|^2 + \frac1{2d}\|\Delta_\Pi^{t,k}\|^2\\\nonumber
            &\qquad + \frac{e}{2}\| D^{t,k}_v - D_v(z^*(x^{t,k}),v^{t,k}, x^{t,k})\|^2 \\\nonumber
            &\qquad + \frac1{2e}\|\Pi_\Gamma(v^{t,k}-\rho D_v(z^*(x^{t,k}),v^{t,k}, x^{t,k}))-v^{t,k}\|^2\\\nonumber
            &\qquad + \frac{f}2\| D^{t,k}_v - D_v(z^*(x^{t,k}),v^{t,k}, x^{t,k})\|^2+\frac1{2f} \|\Delta_\Pi^{t,k}\|^2\\
         &\leq \left(\frac{c+d}2 + \rho^2\left(\frac1{2c}+\frac1{2e}\right)\right)\|D_v(z^*(x^{t,k}),v^{t,k}, x^{t,k})\|^2 \\\nonumber
            &\qquad + \left(\frac{e+f}2 + \rho^2\left(\frac1{2d}+\frac1{2f}\right)\right)\|D^{t,k}_v - D_v(z^*(x^{t,k}),v^{t,k}, x^{t,k})\|^2\\
         &\leq \left(\frac{c+d}2 + \rho^2\left(\frac1{2c}+\frac1{2e}\right)\right)\frac{2L^G_1}{\mu_G}\phi_v(v^{t,k},x^{t,k}) \\\nonumber
            &\qquad + \left(\frac{e+f}2 + \rho^2\left(\frac1{2d}+\frac1{2f}\right)\right)\|D^{t,k}_v - D_v(z^*(x^{t,k}),v^{t,k}, x^{t,k})\|^2
    \end{align*}

    Let us take $c=d=e=f=\rho$. We get
     \begin{align}
         - \left\langle D^{t,k}_v, v^{t,k+1}-v^{t,k}\right\rangle
             &\leq \frac{4L^G_1}{\mu_G}\rho\phi_v(v^{t,k},x^{t,k}) + 2\rho\|D^{t,k}_v - D_v(z^*(x^{t,k}),v^{t,k}, x^{t,k})\|^2 \enspace.
    \end{align}

    Then, by plugging the last Inequality in \cref{app:eq:phi_v_10} and setting $\alpha = \frac{\mu_G^2}{32L^G_1}$, we end up with

    \begin{align*}
        \nonumber\phi_v(v^{t,k+1}, x^{t,k+1}) &\leq  \left(1-\frac{\mu_G}8\rho\right)\phi_v(v^{t,k}, x^{t,k}) - \frac\alpha\rho\|v^{t,k+1}-v^{t,k}\|^2 \\\nonumber
            &\qquad + \rho\left[\frac{L^G_1}{\mu_G^2} +  \frac{\mu_G^2}{16L^G_1} + \frac{\Lambda_v\rho}2 + \frac{2L^G_1\Lambda_v^2\rho^2}{\mu_G^2}\right]\|D^{t,k}_v - D_v(z^*(x^{t,k}), v^{t,k}, x^{t,k})\|^2 \\\nonumber
            &\qquad - \gamma  \langle D_x^{t,k}, \nabla_2 \phi_v(v^{t,k}, x^{t,k})\rangle + \frac{\Lambda_v}2\gamma^2  \|D_x^{t,k}\|^2\\
        &\leq  \left(1-\frac{\mu_G}8\rho\right)\phi_v(v^{t,k}, x^{t,k}) - \frac{\mu_G^2}{32L^G_1}\rho\|\calG^{t,k}_v\|^2 \\\nonumber
            &\qquad + \rho\left[\frac{L^G_1}{\mu_G^2} +  \frac{\mu_G^2}{16L^G_1} + \frac{\Lambda_v\rho}2 + \frac{2L^G_1\Lambda_v^2\rho^2}{\mu_G^2}\right]\|D^{t,k}_v - D_v(z^*(x^{t,k}), v^{t,k}, x^{t,k})\|^2 \\\nonumber
            &\qquad - \gamma  \langle D_x^{t,k}, \nabla_2 \phi_v(v^{t,k}, x^{t,k})\rangle + \frac{\Lambda_v}2\gamma^2  \|D_x^{t,k}\|^2\enspace.
    \end{align*}

    Since $\rho\leq B_v \triangleq\left[\frac{L^G_1}{\mu_G^2} +  \frac{\mu_G^2}{16L^G_1}\right]\min\left(\frac{2}{\Lambda_v}, \frac{\mu_G}{\sqrt{2L^G_1}\Lambda_v}\right)$ yields
    \begin{align}
        \nonumber\phi_v(v^{t,k+1}, x^{t,k+1}) &\leq  \left(1-\frac{\mu_G}8\rho\right)\phi_v(v^{t,k}, x^{t,k}) - \frac{\mu_G^2}{32L^G_1}\rho\|\calG^{t,k}_v\|^2 \\\nonumber
            &\qquad + 3\rho\left[\frac{L^G_1}{\mu_G^2} +  \frac{\mu_G^2}{16L^G_1}\right]\|D^{t,k}_v - D_v(z^*(x^{t,k}), v^{t,k}, x^{t,k})\|^2 \\\nonumber
            &\qquad - \gamma  \langle D_x^{t,k}, \nabla_2 \phi_v(v^{t,k}, x^{t,k})\rangle + \frac{\Lambda_v}2\gamma^2  \|D_x^{t,k}\|^2\\
        &\leq \left(1-\frac{\mu_G}8\rho\right)\phi_v(v^{t,k}, x^{t,k}) - \frac{\mu_G^2}{32L^G_1}\rho\|\calG^{t,k}_v\|^2 \\\nonumber
            &\qquad + 6\rho\left[\frac{L^G_1}{\mu_G^2} +  \frac{\mu_G^2}{16L^G_1}\right]\|D^{t,k}_v - D_v(\bfu^{t,k})\|^2 \\\nonumber
            &\qquad + 6\rho\left[\frac{L^G_1}{\mu_G^2} +  \frac{\mu_G^2}{16L^G_1}\right]\|D_v(\bfu^{t,k}) - D_v(z^*(x^{t,k}), v^{t,k}, x^{t,k})\|^2 \\\nonumber
            &\qquad - \gamma  \langle D_x^{t,k}, \nabla_2 \phi_v(v^{t,k}, x^{t,k})\rangle + \frac{\Lambda_v}2\gamma^2  \|D_x^{t,k}\|^2\enspace.
    \end{align}

    Tacking the expectation conditionally to the past iterates yields
    \begin{align}\label{app:eq:phi_v_5}
        \bbE_{t,k}[\phi_v(v^{t,k+1}, x^{t,k+1})] &\leq \left(1-\frac{\mu_G}8\rho\right)\phi_v(v^{t,k}, x^{t,k}) - \frac{\mu_G^2}{32L^G_1}\rho\bbE_{t,k}[\|\calG^{t,k}_v\|^2] \\\nonumber
            &\qquad + 6\rho\left[\frac{L^G_1}{\mu_G^2} +  \frac{\mu_G^2}{16L^G_1}\right]\bbE_{t,k}[\|D^{t,k}_v - D_v(\bfu^{t,k})\|^2]\\\nonumber
            &\qquad + 6\rho\left[\frac{L^G_1}{\mu_G^2} +  \frac{\mu_G^2}{16L^G_1}\right]\bbE_{t,k}[\|D_v(\bfu^{t,k}) - D_v(z^*(x^{t,k}), v^{t,k}, x^{t,k})\|^2]\\\nonumber
            &\qquad - \gamma  \langle \bbE_{t,k}[D_x^{t,k}], \nabla_2 \phi_v(v^{t,k}, x^{t,k})\rangle + \frac{\Lambda_v}2\gamma^2  \bbE_{t,k}[\|D_x^{t,k}\|^2]\enspace.
    \end{align}

    From Young inequality, we have for any $c>0$
    \begin{align}\label{app:eq:8}
         \langle \bbE_{t,k}[D_x^{t,k}], \nabla_2 \phi_v(v^{t,k}, x^{t,k})\rangle \leq c^{-1}\|\bbE_{t,k}[D_x^{t,k}]\|^2 + c \|\nabla_2 \phi_v(v^{t,k}, x^{t,k})\|^2\enspace .
    \end{align}

    Moreover, using the Lipschitz continuity of $z^*$, of $\nabla^2_{11} G$ and $\nabla_ F$ and the fact that $v$ and $v^*$ are bounded, we have
    \begingroup
    \allowdisplaybreaks %
    \begin{align*}
    \|\nabla_2 \phi_v(v,x)\| &\leq \|\diff z*(x)\| \left[\left\|\frac12 (\nabla^3_{111}G(z^*(x),x)|v)v  - \frac12 (\nabla^3_{111}G(z^*(x),x)|v^*(x))v^*(x)\right\| \right.\\\nonumber
        &\qquad \left. +\| \nabla_{11}^2 F(z^*(x),x) v - \nabla_{11}^2 F(z^*(x),x) v^*(x)\|\right] \\\nonumber
        &\qquad  + \|\frac12 (\nabla^3_{211}G(z^*(x),x)|v)v  - \frac12 (\nabla^3_{211}G(z^*(x),x)|v^*(x))v^*(x)\|\\\nonumber
        &\qquad  + \|\nabla_{21}^2 F(z^*(x),x) v - \nabla_{21}^2 F(z^*(x),x) v^*(x)\|\\
    &\leq L_* \left[\left\|\frac12 (\nabla^3_{111}G(z^*(x),x)|v-v^*(x))v  - \frac12 (\nabla^3_{111}G(z^*(x),x)|v^*(x))(v-v^*(x))\right\| \right.\\\nonumber
        &\qquad \left. + L^F_2\| v - v^*(x)\|\right] \\\nonumber
        &\qquad  + \left\|\frac12 (\nabla^3_{211}G(z^*(x),x)|v-v^*(x))v  - \frac12 (\nabla^3_{211}G(z^*(x),x)|v^*(x))(v-v^*(x))\right\|\\\nonumber
        &\qquad  + L^F_2\|v - v^*(x)\|\\
    &\leq L_* \left[\left\|\frac12 (\nabla^3_{111}G(z^*(x),x)|v-v^*(x))v\right\| \right. \\\nonumber
        &\qquad \left. + \left\|\frac12 (\nabla^3_{111}G(z^*(x),x)|v^*(x))(v-v^*(x))\right\| + L^F_2\| v - v^*(x)\|\right] \\\nonumber
        &\qquad + \left\|\frac12 (\nabla^3_{211}G(z^*(x),x)|v-v^*(x))v\right\| \\\nonumber
        &\qquad + \left\|\frac12 (\nabla^3_{211}G(z^*(x),x)|v^*(x))(v-v^*(x))\right\| + L^F_2\|v - v^*(x)\|\\
    &\leq L_* \left[\frac{L^G_2}2 (\|v\|  + \|v^*(x)\|)\|v-v^*(x))\|  + L^F_2\| v - v^*(x)\|\right] \\\nonumber
        &\qquad  +\frac{L^G_2}2 (\|v\|  + \|v^*(x)\|)\|v-v^*(x))\|  + L^F_2\|v - v^*(x)\|\\
    &\leq (L_* + 1)\left[L^G_2R  + L^F_2\right]\| v - v^*(x)\|\enspace.
    \end{align*}
    \endgroup
    
    On the other hand, we have by strong convexity
    \begin{align*}
        \| v - v^*(x)\|^2 \leq \frac{2}{\mu_G}\phi_v(v,x)\enspace.
    \end{align*}
    As a consequence, we have
    \begin{align}\label{app:eq:9}
         \|\nabla_2 \phi_v(v^{t,k}, x^{t,k})\|^2 \leq L''\phi_v(v^{t,k}, x^{t,k})
    \end{align}
    with $L'' = \frac{2(L_* + 1)^2\left[L^G_2R+ L^F_2\right]^2}{\mu_G}$.

    Plugging Inequalities \eqref{app:eq:8} and \eqref{app:eq:9} into \eqref{app:eq:phi_v_5} yields
    \begin{align*}
        \bbE_{t,k}[\phi_v(v^{t,k+1}, x^{t,k+1})] &\leq \left(1-\frac{\mu_G}8\rho + cL''\gamma\right)\phi_v(v^{t,k}, x^{t,k}) - \frac{\mu_G^2}{32L^G_1}\rho\bbE_{t,k}[\|\calG^{t,k}_v\|^2] \\\nonumber
            &\qquad + 6\rho\left[\frac{L^G_1}{\mu_G^2} +  \frac{\mu_G^2}{16L^G_1}\right]\bbE_{t,k}[\|D^{t,k}_v - D_v(\bfu^{t,k})\|^2]\\\nonumber
            &\qquad + 6\rho\left[\frac{L^G_1}{\mu_G^2} +  \frac{\mu_G^2}{16L^G_1}\right]\bbE_{t,k}[\|D_v(\bfu^{t,k}) - D_v(z^*(x^{t,k}), v^{t,k}, x^{t,k})\|^2]\\\nonumber
            &\qquad + \frac{\gamma}c\|\bbE_{t,k}[D^{t,k}_x]\|^2 + \frac{\Lambda_v}2\gamma^2  \bbE_{t,k}[\|D_x^{t,k}\|^2]\enspace.
    \end{align*}

    The Lipschitz continuity of $\nabla^2_{11}G$ and $\nabla_1 F$ and the boundedness of $v$ give us
    \begin{align*}
        \|D_v(\bfu^{t,k}) - D_v(z^*(x^{t,k}), v^{t,k}, x^{t,k})\|^2 &\leq \left(\|\nabla^2_{11}G(z^{t,k}, x^{t,k})v^{t,k} - \nabla^2_{11}G(z^*(x^{t,k}), x^{t,k})v^{t,k}\|\right. \\\nonumber
        &\qquad \left.+ \|\nabla_1 F(z^{t,k},x^{t,k}) - \nabla_1 F(z^*(x^{t,k}), x^{t,k})\|\right)^2\\
        &\leq (L^G_2R + L^F_1)^2\|z^{t,k} - z^*(x^{t,k})\|^2\\
        &\leq \frac{2(L^G_2R + L^F_1)^2}{\mu_G}\phi_z(z^{t,k}, x^{t,k})\enspace.
    \end{align*}

    As a consequence
    \begin{align}\label{app:eq:phi_v_6}
        \bbE_{t,k}[\phi_v(v^{t,k+1}, x^{t,k+1})] &\leq \left(1-\frac{\mu_G}8\rho + cL''\gamma\right)\phi_v(v^{t,k}, x^{t,k}) - \frac{\mu_G^2}{32L^G_1}\rho\bbE_{t,k}[\|\calG^{t,k}_v\|^2] \\\nonumber
            &\qquad + 6\rho\left[\frac{L^G_1}{\mu_G^2} +  \frac{\mu_G^2}{16L^G_1}\right]\bbE_{t,k}[\|D^{t,k}_v - D_v(\bfu^{t,k})\|^2]\\\nonumber
            &\qquad + 6\rho\left[\frac{L^G_1}{\mu_G^2} +  \frac{\mu_G^2}{16L^G_1}\right]\frac{2(L^G_2R + L^F_1)^2}{\mu_G}\phi_z(z^{t,k}, x^{t,k})\\\nonumber
            &\qquad + \frac{\gamma}c\|\bbE_{t,k}[D^{t,k}_x]\|^2 + \frac{\Lambda_v}2\gamma^2  \bbE_{t,k}[\|D_x^{t,k}\|^2]\enspace.
    \end{align}
    
    From \cref{app:lemma:bound_error_directions}, we have
    \begin{align*}
        \bbE[\|D_v^{t,k} - D_v(\bfu^{t,k})\|^2] &\leq 4\rho^2\left((L^G_2R)^2 + (L^F_1)^2\right)\sum_{r=1}^k\bbE[\|D^{t,r-1}_z\|^2] + 4\rho^2(L^G_1)^2\sum_{r=1}^k\bbE[\|\calG^{t,r-1}_v\|^2] \\\nonumber
        &\qquad + 4\gamma^2\left((L^G_2R)^2 + (L^F_1)^2\right)\sum_{r=1}^k\bbE[\|D_x^{t,r-1}\|^2]
    \end{align*}

    Taking the total expectation and plugging the previous inequality in \cref{app:eq:phi_v_6} yields
    \begin{align*}
        \phi_v^{t,k+1} &\leq \left(1-\frac{\mu_G}8\rho + cL''\gamma\right)\phi_v^{t,k} - \frac{\mu_G^2}{32L^G_1}\rho\bbE_{t,k}[\|\calG^{t,k}_v\|^2] \\\nonumber
            &\qquad + 24\rho^3\left((L^G_2R)^2 + (L^F_1)^2\right)\left(\frac{L^G_1}{\mu_G^2} +  \frac{\mu_G^2}{16L^G_1}\right)\sum_{r=1}^k\bbE[\|D^{t,r-1}_z\|^2] \\\nonumber
            &\qquad + 24\rho^3(L^G_1)^2\left(\frac{L^G_1}{\mu_G^2} +  \frac{\mu_G^2}{16L^G_1}\right)\sum_{r=1}^k\bbE[\|\calG^{t,r-1}_v\|^2] \\\nonumber
            &\qquad + 24\rho\gamma^2\left((L^G_2R)^2 + (L^F_1)^2\right)\left(\frac{L^G_1}{\mu_G^2} +  \frac{\mu_G^2}{16L^G_1}\right)\sum_{r=1}^k\bbE[\|D_x^{t,r-1}\|^2]\\\nonumber
            &\qquad  + \left[\frac{L^G_1}{\mu_G^2} +  \frac{\mu_G^2}{16L^G_1}\right]\frac{12(L^G_2R + L^F_1)^2}{\mu_G}\rho\phi_z^{t,k}   \\\nonumber
            &\qquad + \frac{\gamma}c\bbE[\|[\bbE_{t,k}D_x^{t,k}]\|^2] + \frac{\Lambda_v}2\gamma^2 \bbE[\|D_x^{t,k}\|^2]\enspace .
    \end{align*}

    Taking $c = \frac{\mu_G\rho}{16L''\gamma}$ yields
    \begin{align*}
        \phi_v^{t,k+1} &\leq \left(1-\frac{\mu_G}{16}\rho\right)\phi_v^{t,k} - \frac{\mu_G^2}{32L^G_1}\rho\bbE_{t,k}[\|\calG^{t,k}_v\|^2] \\\nonumber
            &\qquad + 24\rho^3\left((L^G_2R)^2 + (L^F_1)^2\right)\left(\frac{L^G_1}{\mu_G^2} +  \frac{\mu_G^2}{16L^G_1}\right)\sum_{r=1}^k\bbE[\|D^{t,r-1}_z\|^2] \\\nonumber
            &\qquad + 24\rho^3(L^G_1)^2\left(\frac{L^G_1}{\mu_G^2} +  \frac{\mu_G^2}{16L^G_1}\right)\sum_{r=1}^k\bbE[\|\calG^{t,r-1}_v\|^2] \\\nonumber
            &\qquad + 24\rho\gamma^2\left((L^G_2R)^2 + (L^F_1)^2\right)\left(\frac{L^G_1}{\mu_G^2} +  \frac{\mu_G^2}{16L^G_1}\right)\sum_{r=1}^k\bbE[\|D_x^{t,r-1}\|^2]\\\nonumber
            &\qquad  + \left[\frac{L^G_1}{\mu_G^2} +  \frac{\mu_G^2}{16L^G_1}\right]\frac{12(L^G_2R + L^F_1)^2}{\mu_G}\rho\phi_z^{t,k}   \\\nonumber
            &\qquad + \frac{16L''}{\mu_G}\frac{\gamma^2}{\rho}\bbE[\|[\bbE_{t,k}D_x^{t,k}]\|^2] + \frac{\Lambda_v}2\gamma^2 \bbE[\|D_x^{t,k}\|^2]\enspace .
    \end{align*}

    Combining \cref{app:eq:exp_norm_exp_d_x} and \cref{app:lemma:bound_error_directions} yields
    \begin{align*}
        \phi_v^{t,k+1} &\leq \left(1-\frac{\mu_G}{16}\rho\right)\phi_v^{t,k} - \frac{\mu_G^2}{32L^G_1}\rho\bbE_{t,k}[\|\calG^{t,k}_v\|^2] \\\nonumber
            &\qquad + 8\rho\left((L^G_2R)^2 + (L^F_1)^2\right)\left[3\left(\frac{L^G_1}{\mu_G^2} +  \frac{\mu_G^2}{16L^G_1}\right)\rho^2 + \frac{8L''}{\mu_G}\gamma^2\right]\sum_{r=1}^k\bbE[\|D^{t,r-1}_z\|^2] \\\nonumber
            &\qquad + 8\rho(L^G_1)^2\left[3\left(\frac{L^G_1}{\mu_G^2} +  \frac{\mu_G^2}{16L^G_1}\right)\rho^2 + \frac{8L''}{\mu_G}\gamma^2\right]\sum_{r=1}^k\bbE[\|\calG^{t,r-1}_v\|^2] \\\nonumber
            &\qquad + 8\gamma^2\left((L^G_2R)^2 + (L^F_1)^2\right)\left[3\left(\frac{L^G_1}{\mu_G^2} +  \frac{\mu_G^2}{16L^G_1}\right)\gamma + \frac{8L''}{\mu_G}\frac{\gamma^2}\rho\right]\sum_{r=1}^k\bbE[\|D_x^{t,r-1}\|^2]\\\nonumber
            &\qquad  + \left[\frac{L^G_1}{\mu_G^2} +  \frac{\mu_G^2}{16L^G_1}\right]\frac{12(L^G_2R + L^F_1)^2}{\mu_G}\rho\phi_z^{t,k}   \\\nonumber
            &\qquad + \frac{16L''}{\mu_G}\frac{\gamma^2}{\rho}\bbE[\|D_x(\bfu^{tk,})\|^2] + \frac{\Lambda_v}2\gamma^2 \bbE[\|D_x^{t,k}\|^2]\enspace .
    \end{align*}

     By assumption, $\gamma\leq C_v\rho$ with $C_v = \sqrt{\frac{\mu_G}{8L''}\left(\frac{L^G_1}{\mu_G^2} +  \frac{\mu_G^2}{16L^G_1}\right)}$, therefore
    \begin{align*}
        \phi_v^{t,k+1} &\leq \left(1-\frac{\mu_G}{16}\rho\right)\phi_v^{t,k} - \frac{\mu_G^2}{32L^G_1}\rho\bbE_{t,k}[\|\calG^{t,k}_v\|^2] \\\nonumber
            &\qquad + 32\rho^3\left((L^G_2R)^2 + (L^F_1)^2\right)\left(\frac{L^G_1}{\mu_G^2} +  \frac{\mu_G^2}{16L^G_1}\right) \sum_{r=1}^k\bbE[\|D^{t,r-1}_z\|^2] \\\nonumber
            &\qquad + 32\rho^3(L^G_1)^2\left(\frac{L^G_1}{\mu_G^2} +  \frac{\mu_G^2}{16L^G_1}\right)\sum_{r=1}^k\bbE[\|\calG^{t,r-1}_v\|^2] \\\nonumber
            &\qquad + 32\gamma^2\rho\left((L^G_2R)^2 + (L^F_1)^2\right)\left(\frac{L^G_1}{\mu_G^2} +  \frac{\mu_G^2}{16L^G_1}\right)\sum_{r=1}^k\bbE[\|D_x^{t,r-1}\|^2]\\\nonumber
            &\qquad + \left[\frac{L^G_1}{\mu_G^2} +  \frac{\mu_G^2}{16L^G_1}\right]\frac{12(L^G_2R + L^F_1)^2}{\mu_G}\rho\phi_z^{t,k}   \\\nonumber
            &\qquad + \frac{16L''}{\mu_G}\frac{\gamma^2}{\rho}\bbE[\|D_x(\bfu^{tk,})\|^2] + \frac{\Lambda_v}2\gamma^2 \bbE[\|D_x^{t,k}\|^2]\enspace .
    \end{align*}

    We get finally

    \begin{align*}
        \phi_v^{t,k+1} &\leq \left(1 - \frac{\rho\mu_G}{16}\right)\phi_v^{t,k}  - \tilde{\beta}_{vv}\rho V_v^t  + \rho^3\beta_{vz}\calV_z^{t,k} +2\rho^3\beta_{vv}\calV_v^{t,k} + \gamma^2\rho\beta_{vx}\calV_x^{t,k}\\\nonumber
        &\qquad+ \rho\alpha_{vz} \phi_z^{t,k}  + \frac{\Lambda_v}2\gamma^2 \bbE[\|D_x^{t,k}\|^2]+ \frac{\gamma^2}\rho\overline{\beta}_{vx}\bbE[\|D_x(\bfu^{t,k})\|^2]
    \end{align*}
     with $\beta_{vz} = \beta_{vx} = 32\left((L^G_2R)^2 + (L^F_1)^2\right)\left(\frac{L^G_1}{\mu_G^2} +  \frac{\mu_G^2}{16L^G_1}\right)$, $\beta_{vv} = (L^G_1)^2\left(\frac{L^G_1}{\mu_G^2} +  \frac{\mu_G^2}{16L^G_1}\right)$, $\overline{\beta}_{vx}~=~\frac{16L''}{\mu_G}$, $\tilde{\beta}_{vv} = \frac{\mu_G^2}{32L^G_1}$\\ and $\alpha_{vz} = \left[\frac{L^G_1}{\mu_G^2} +  \frac{\mu_G^2}{16L^G_1}\right]\frac{12(L^G_2R + L^F_1)^2}{\mu_G}$.
\end{proof}

\subsection{Proof of \cref{lemma:descent_h}}
\begin{proof}
    The smoothness of $h$ (\cref{prop:smoothness_h}) gives us
    \begin{align*}
        h(x^{t,k+1})&\leq h(x^{t,k}) - \gamma\langle \nabla h(x^{t,k}), D_x^{t,k}\rangle + \gamma^2 \frac{L^h}2\|D^{t,k}_x\|^2\enspace.
    \end{align*}
    Then, we use the identity $\langle a,b\rangle = \frac12(\|a\|^2 + \|b\|^2 - \|a - b\|^2)$ to get
    \begin{align*}
       h(x^{t,k+1}) &\leq h(x^{t,k}) - \frac\gamma2\|\nabla h(x^{t,k})\|^2 - \frac\gamma2\|D^{t,k}_x\|^2 + \frac\gamma2 \|\nabla h(x^{t,k})-D^{t,k}_x\|^2 + \gamma^2 \frac{L^h}2\|D^{t,k}_x\|^2\\
        &\leq h(x^{t,k}) - \frac\gamma2\|\nabla h(x^{t,k})\|^2 - \frac\gamma2\|D^{t,k}_x\|^2 + \gamma \|\nabla h(x^{t,k})-D_x(\bfu^{t,k})\|^2 \\\nonumber
        &\qquad + \gamma \|D_x(\bfu^{t,k})-D^{t,k}_x\|^2 + \gamma^2 \frac{L^h}2\|D^{t,k}_x\|^2\enspace.
    \end{align*}
    
    Then taking the expectation gives and using \cref{prop:error_gradient} yields
    \begin{align*}
        h^{t,k+1}&\leq h^{t,k} - \frac\gamma2 g^{t,k} +\gamma\bbE[\|\nabla h(x^{t,k})-D_x(\bfu^{t,k})\|^2] \\
            &\qquad + \gamma \bbE[\|D_x(\bfu^{t,k})-D^{t,k}_x\|^2]- \frac\gamma2\left(1-L^h\gamma\right) \bbE[\|D^{t,k}_x\|^2]\\
        &\leq h^{t,k} - \frac\gamma2 g^{t,k} +\gamma L_x^2(\bbE[\|z^{t,k} - z^*(x^{t,k})\|^2] + \bbE[\|v^{t,k} - v^*(x^{t,k})\|^2])  \\\nonumber
            &\qquad + \gamma \bbE[\|D_x(\bfu^{t,k})-D^{t,k}_x\|^2] - \frac\gamma2\left(1-L^h\gamma\right) \bbE[\|D^{t,k}_x\|^2]\enspace.
    \end{align*}
    The $\mu_G$-strong convexity of $G(\,.\,,x)$ ensures that $\|z-z^*(x)\|^2 \leq \frac2{\mu_G}\phi_z(z, x)$ and $\|v~-~v^*(x)\|^2~\leq~\frac2{\mu_G}\phi_v(v, x)$. As a consequence
    \begin{align*}
        h^{t,k+1}&\leq h^{t,k} - \frac\gamma2 g^{t,k} + \gamma \frac{2L_x^2}{\mu_G}(\phi_z^{t,k} + \phi_v^{t,k}) + \gamma \bbE[\|D_x(z^{t,k}, v^{t,k}, x^{t,k})-D^{t,k}_x\|^2]\\
            &\qquad - \frac\gamma2\left(1-L^h\gamma\right) \bbE[\|D^{t,k}_x\|^2]\enspace.
    \end{align*}

    From \cref{app:lemma:bound_error_directions}, we have
    \begin{align*}
        \bbE[\|D_x^{t,k} - D_x(\bfu^{t,k})\|^2]&\leq 4\rho^2\left((L^G_2R)^2 + (L^F_1)^2\right)\sum_{r=1}^k\bbE[\|D^{t,r-1}_z\|^2] + 4\rho^2(L^G_1)^2\sum_{r=1}^k\bbE[\|\calG^{t,r-1}_v\|^2] \\\nonumber
        &\qquad + 4\gamma^2\left((L^G_2R)^2 + (L^F_1)^2\right)\sum_{r=1}^k\bbE[\|D_x^{t,r-1}\|^2]\enspace.
    \end{align*}

    As a consequence
    \begin{align*}
        h^{t,k+1}&\leq h^{t,k} - \frac\gamma2 g^{t,k} + \gamma \frac{2L_x^2}{\mu_G}(\phi_z^{t,k} + \phi_v^{t,k}) \\\nonumber
            &\qquad + 4\gamma \rho^2\left((L^G_2R)^2 + 2(L^F_1)^2\right)\sum_{r=1}^k\bbE[\|D^{t,r-1}_z\|^2] + 4\gamma\rho^2(L_1^G)^2\sum_{r=1}^k \bbE[\|\calG^{t,r-1}_v\|^2]\\\nonumber
            &\qquad + 4\gamma^3\left((L^G_2R)^2 + 2(L^F_1)^2\right)\sum_{r=1}^k \bbE[\|D^{t,r-1}_x\|^2] - \frac\gamma2\left(1-L^h\gamma\right) \bbE[\|D^{t,k}_x\|^2]\\
        &\leq h^{t,k} - \frac\gamma2 g^{t,k} + \gamma \frac{2L_x^2}{\mu_G}(\phi_z^{t,k} + \phi_v^{t,k})  + \gamma \rho^2\beta_{hz}\calV_z^{t,k} + \gamma\rho^2\beta_{hv}\calV_v^{t,k}\\
            &\qquad + \gamma^3\beta_{hx}\calV_x^{t,k} - \frac\gamma2\left(1-L^h\gamma\right) \bbE[\|D^{t,k}_x\|^2]
    \end{align*}
    with $\beta_{hz} = 4\left((L^G_2R)^2 + 2(L^F_1)^2\right)$, $\beta_{hv} = 4(L_1^G)^2$ and $\beta_{hx} = 4\left((L^G_2R)^2 + 2(L^F_1)^2\right)$.
\end{proof}

\subsection{Proof of \cref{th:cvg_rate} and \cref{cor:sample_complexity}}
The constants involved in \cref{th:cvg_rate} are 
\begin{align*}
    \psi_z = \frac1{16\overline{\beta}_{zx}},\quad \psi_v = \min\left[\frac1{16\overline{\beta}_{vx}}, \frac{\alpha_{zv}\mu_G}{12}\psi_z\right]
\end{align*}
\begin{align*}
    \overline{\rho} = \min\left[\sqrt{\frac{\psi_z}{12q(\psi_z\beta_{zz}+\psi_v\beta_{zv})}},\sqrt{\frac1{6\Lambda_z}}, \sqrt{\frac1{12q\beta_{vv}}}, B_v\right]\enspace,
\end{align*}
\begin{align*}
    \xi =\min\left[C_z, C_v, 1, \frac{\psi_v\mu_G^2}{16L_x^2}, \sqrt{\frac{\mu_G}{8\overline{\beta}_{vx}}}, \frac{\psi_z\mu_G^2}{24L_x^2}, \sqrt{\frac{\mu_G}{12\overline{\beta}_{zx}}}\right]\enspace,
\end{align*}
\begin{align*}
    \overline{\gamma} = \min\left[\sqrt{\frac1{12q(\psi_z\beta_{zx} + \psi_v\beta_{vx})}}, \sqrt{\frac1{12q\beta_{hx}}},\frac1{6(L^h + \psi_z\Lambda_z + \psi_v\Lambda_v)},\sqrt{\frac{\psi_v\tilde{\beta}_{vv}}{6q(\beta_{hv}+\psi_z\beta_{vz})}}, \sqrt{\frac{\psi_z}{12q\beta_{hz}}}\right]\enspace.
\end{align*}
\begin{proof}
The proof is a classical Lyapunov analysis.
Consider the following Lyapunov function $\calL^{t,k} = h^{t,k} + \psi_z\phi^{t,k}_z + \psi_v\phi^{t,k}_v$ for some positive constants $\psi_z$ and $\psi_v$.
We use use Lemmas \ref{lemma:descent_z} to \ref{lemma:descent_h} to upper bound $\calL^{t,k} - \calL^{t,k+1}$.

We have 
\begin{align}\label{app:eq:lyapunov_ineq}
    \calL^{t,k+1} - \calL^{t,k} &\leq -\frac\gamma2 g^{t,k} + (\psi_z \overline{\beta}_{zx} + \psi_v\overline{\beta}_{vx})\frac{\gamma^2}\rho\bbE[\|D_x(\bfu^{t,k})\|^2]\\\nonumber
        &\qquad + \left(\frac{2L_x^2}{\mu_G}\gamma - \psi_z\frac{\mu_G}2\rho + \psi_v \alpha_{zv}\rho\right)\phi_z^{t,k} + \left(\frac{2L_x^2}{\mu_G}\gamma - \psi_v\frac{\mu_G}{16}\rho\right)\phi_v^{t,k}\\\nonumber
        &\qquad + \left(\psi_z\frac{\Lambda_z}2\rho^2 - \psi_z\frac12\rho\right)V_z^{t,k}  - \psi_v\tilde{\beta}_{vv}\rho V_v^{t,k}\\\nonumber
        &\qquad + \left(\frac{L^h}2\gamma^2 + \psi_z\frac{\Lambda_z}2\gamma^2 + \psi_v\frac{\Lambda_v}2\gamma^2 - \frac\gamma2\right)V_x^{t,k} \\\nonumber
        &\qquad + \left(\beta_{hz}\rho\gamma^2 + \psi_z\beta_{zz}\rho^3 + \psi_v\beta_{zv}\rho^3\right)\calV_z^{t,k}  \\\nonumber
        &\qquad + \left(\beta_{hv}\rho\gamma^2 + \psi_z\beta_{vz}\gamma^2\rho + \psi_v\beta_{vv}\rho^3\right)\calV_v^{t,k} \\\nonumber
        &\qquad + \left(\beta_{hx}\gamma^3 + \psi_z\beta_{zx}\gamma^2\rho + \psi_v\beta_{vx}\rho^3\right)\calV_x^{t,k}\enspace.
\end{align}

We bound $\bbE[\|D_x(\bfu^{t,k})\|^2]$ crudely by using \cref{prop:error_gradient}
\begin{align*}
    \bbE[\|D_x(\bfu^{t,k})\|^2] &\leq 2\bbE[\|\nabla h(x^{t,k})\|^2] + 2\bbE[\|D_x(\bfu^{t,k}) - \nabla h(x^{t,k})\|^2]\\
    &\leq 2g^{t,k} + 2(\bbE[\|z^{t,k} - z^*(x^{t,k})\|^2] + \bbE[\|v^{t,k} - v^*(x^{t,k})\|^2])\\
    &\leq 2g^{t,k} + \frac4{\mu_G}(\phi_z^{t,k}+ \phi_v^{t,k})\enspace.
\end{align*}

Summing in \eqref{app:eq:lyapunov_ineq} for $k=0,\dots,q-1$ yields
\begin{align}
    \calL^{t,q} - \calL^{t,0} &\leq  -\left(\frac\gamma2 - 2\psi_z\overline{\beta}_{zx}\frac{\gamma^2}{\rho} - 2\psi_v\overline{\beta}_{vx}\frac{\gamma^2}{\rho}\right) \sum_{k=0}^{q-1}g^{t,k} \\\nonumber
        &\qquad + \left(\frac{2L_x^2}{\mu_G}\gamma - \psi_z\frac{\mu_G}2\rho + \psi_v \alpha_{zv}\rho + \psi_z\overline{\beta}_{zx}\frac{\gamma^2}\rho\right)\sum_{k=0}^{q-1}\phi_z^{t,k}\\\nonumber
        &\qquad + \left(\frac{2L_x^2}{\mu_G}\gamma - \psi_v\frac{\mu_G}{16}\rho + \psi_v\overline{\beta}_{vx}\frac{\gamma^2}\rho\right)\sum_{k=0}^{q-1}\phi_v^{t,k} + \left(\psi_z\frac{\Lambda_z}2\rho^2 - \psi_z\frac12\rho\right)\sum_{k=0}^{q-1}V_z^{t,k} \\\nonumber
        &\qquad-  \psi_v\tilde{\beta}_{vv}\rho\sum_{k=0}^{q-1}V_v^{t,k}  + \left(\frac{L^h}2\gamma^2 + \psi_z\frac{\Lambda_z}2\gamma^2 + \psi_v\frac{\Lambda_v}2\gamma^2 - \frac\gamma2\right)\sum_{k=0}^{q-1}V_x^{t,k}\\\nonumber
        &\qquad + \left(\beta_{hz}\rho\gamma^2 + \psi_z\beta_{zz}\rho^3 + \psi_v\beta_{zv}\rho^3\right)\sum_{k=0}^{q-1}\calV_z^{t,k} \\\nonumber
        &\qquad + \left(\beta_{hv}\rho\gamma^2 + \psi_z\beta_{vz}\gamma^2\rho + \psi_v\beta_{vv}\rho^3\right)\sum_{k=0}^{q-1}\calV_v^{t,k}\\\nonumber
        &\qquad + \left(\beta_{hx}\gamma^3 + \psi_z\beta_{zx}\gamma^2\rho + \psi_v\beta_{vx}\rho^3\right)\sum_{k=0}^{q-1}\calV_x^{t,k}\enspace.
\end{align}

Since we have
\begin{align*}
    \sum_{k=0}^{q-1}\calV_\bullet^{t,k} &= \sum_{k=0}^{q-1}\sum_{r=1}^k\bbE[\|D_\bullet^{t,r-1}\|^2] = \sum_{r=1}^{q-1}\sum_{k=r}^{q-1}\bbE[\|D_\bullet^{t,r-1}\|^2] \\\nonumber
        &\qquad = \sum_{r=1}^{q-1}(q-r)\bbE[\|D_\bullet^{t,r-1}\|^2] \leq q\sum_{k=1}^{q-1}\bbE[\|D_\bullet^{t,k-1}\|^2]\enspace
\end{align*}
we get 
\begin{align}\label{app:eq:lyap_ineq_2}
    \calL^{t,q} - \calL^{t,0} &\leq  -\left(\frac\gamma2 - 2\psi_z\overline{\beta}_{zx}\frac{\gamma^2}{\rho} - 2\psi_v\overline{\beta}_{vx}\frac{\gamma^2}{\rho}\right) \sum_{k=0}^{q-1}g^{t,k}\\\nonumber
         &\qquad + \left(\frac{2L_x^2}{\mu_G}\gamma - \psi_z\frac{\mu_G}2\rho + \psi_v \alpha_{zv}\rho + \psi_z\overline{\beta}_{zx}\frac{\gamma^2}\rho\right)\sum_{k=0}^{q-1}\phi_z^{t,k}\\\nonumber
        &\qquad + \left(\frac{2L_x^2}{\mu_G}\gamma - \psi_v\frac{\mu_G}2\rho + \psi_v\overline{\beta}_{vx}\frac{\gamma^2}\rho\right)\sum_{k=0}^{q-1}\phi_v^{t,k}\\\nonumber
        &\qquad + \left(\psi_z\frac{\Lambda_z}2\rho^2 - \psi_z\frac12\rho + q \left(\beta_{hz}\rho\gamma^2 + \psi_z\beta_{zz}\rho^3 + \psi_v\beta_{zv}\rho^3\right)\right)\sum_{k=0}^{q-1}V_z^{t,k}\\\nonumber
        &\qquad + \left( - \psi_v\tilde{\beta}_{vv}\rho + q\left(\beta_{hv}\rho\gamma^2 + \psi_z\beta_{vz}\gamma^2\rho + \psi_v\beta_{vv}\rho^3\right)\right)\sum_{k=0}^{q-1}V_v^{t,k}\\\nonumber
        &\qquad + \left(\frac{L^h}2\gamma^2 + \psi_z\frac{\Lambda_z}2\gamma^2 + \psi_v\frac{\Lambda_v}2\gamma^2 - \frac\gamma2 + q\left(\beta_{hx}\gamma^3 + \psi_z\beta_{zx}\gamma^2\rho + \psi_v\beta_{vx}\rho\gamma^2\right)\right)\sum_{k=0}^{q-1}V_x^{t,k}\enspace.
\end{align}

Since $\rho\leq \overline{\rho}\leq \min\left[\sqrt{\frac{\psi_z}{12q(\psi_z\beta_{zz}+\psi_v\beta_{zv})}},\sqrt{\frac1{6\Lambda_z}}\right]$ and $\gamma\leq\overline{\gamma} \leq \sqrt{\frac{\psi_z}{12q\beta_{hz}}}$, we have 
\begin{align}\label{app:eq:ineq_1}
    \psi_z\frac{\Lambda_z}2\rho^2 - \psi_z\frac12\rho + q \left(\beta_{hz}\rho\gamma^2 + \psi_z\beta_{zz}\rho^3 + \psi_v\beta_{zv}\rho^3\right) < 0\enspace.
\end{align}

Moreover, the conditions $\rho\leq \overline{\rho}\leq \sqrt{\frac{\tilde{\beta}_{vv}}{6q\beta_{vv}}}$ and $\gamma\leq\overline{\gamma} \leq \sqrt{\frac{\psi_v\tilde{\beta}_{vv}}{6q(\beta_{hv}+\psi_z\beta_{vz})}}$, ensure that 
\begin{align}\label{app:eq:ineq_2}
    \psi_v\frac{\Lambda_v}2\rho^2 - \psi_v\tilde{\beta}_{vv}\rho + q\left(\beta_{hv}\rho\gamma^2 + \psi_z\beta_{vz}\gamma^2\rho + \psi_v\beta_{vv}\rho^3\right) < 0\enspace.
\end{align}

The conditions  $\rho\leq \overline{\rho}\leq \sqrt{\frac1{12q(\psi_z\beta_{zx} + \psi_v\beta_{vx})}}$ and $\gamma\leq \overline{\gamma}\leq \min\left[\sqrt{\frac1{12q(\psi_z\beta_{zx} + \psi_v\beta_{vx})}}, \sqrt{\frac1{12q\beta_{hx}}},\frac1{6(L^h + \psi_z\Lambda_z + \psi_v\Lambda_v)}\right]$ yield
\begin{align}\label{app:eq:ineq_3}
    \frac{L^h}2\gamma^2 + \psi_z\frac{\Lambda_z}2\gamma^2 + \psi_v\frac{\Lambda_v}2\gamma^2 - \frac\gamma2 + q\left(\beta_{hx}\gamma^3 + \psi_z\beta_{zx}\gamma^2\rho + \psi_v\beta_{vx}\rho\gamma^2\right) < 0\enspace.
\end{align}

The condition $\gamma\leq \xi\rho\leq \min\left[\frac{\psi_v\mu_G^2}{16L_x^2}, \sqrt{\frac{\mu_G}{8\overline{\beta}_{vx}}}\right]\rho$ ensures
\begin{align}\label{app:eq:ineq_4}
    \frac{2L_x^2}{\mu_G}\gamma - \psi_v\frac{\mu_G}2\rho + \psi_v\overline{\beta}_{vx}\frac{\gamma^2}\rho \leq 0
\end{align}

By definition, we have $\psi_v \leq \frac{\alpha_{zv}\mu_G}{12} \psi_z$ and by assumptions $\gamma\leq \xi\rho\leq \min\left[\frac{\psi_z\mu_G^2}{24L_x^2}, \sqrt{\frac{\mu_G}{12\overline{\beta}_{zx}}}\right]\rho$. As a consequence
\begin{align}\label{app:eq:ineq_5}
    \frac{2L_x^2}{\mu_G}\gamma - \psi_z\frac{\mu_G}2\rho + \psi_v \alpha_{zv}\rho + \psi_z\overline{\beta}_{zx}\frac{\gamma^2}\rho <0\enspace.
\end{align}

Plugging Inequalities \eqref{app:eq:ineq_1} to \eqref{app:eq:ineq_5} into \cref{app:eq:lyap_ineq_2} gives
\begin{align*}
    \calL^{t,q} - \calL^{t,0} &\leq  -\left(\frac\gamma2 - 2\psi_z\overline{\beta}_{zx}\frac{\gamma^2}{\rho} - 2\psi_v\overline{\beta}_{vx}\frac{\gamma^2}{\rho}\right) \sum_{k=0}^{q-1}g^{t,k}\enspace.
\end{align*}

Since $\psi_z = \frac{\rho}{16\overline{\beta}_{zx}}$ and $\psi_v \leq \frac{\rho}{16\overline{\beta}_{vx}}$ and $\frac{\gamma^2}{\rho}\leq\xi\leq 1$, we get 
\begin{align*}\label{app:eq:lyap_ineq_2}
    \underbrace{\calL^{t,q} - \calL^{t,0}}_{\calL^{t+1,0} - \calL^{t,0}} &\leq  -\frac\gamma4  \sum_{k=0}^{q-1}g^{t,k}\enspace.
\end{align*}

Summing, telescoping and dividing by $Tq$ gives
\begin{equation*}
    \frac{1}{Tq}\sum_{t=0}^{T-1}\sum_{k=0}^{q-1}g^{t,k} \leq \frac4{Tq\gamma}\underbrace{\left( h^{0,0} - h^* + \psi_z\phi^{0,0} + \psi_v\phi^{0,0}]\right)}_{\Gamma^0}\enspace.
\end{equation*}
\end{proof}

From \cref{th:cvg_rate} we deduce \cref{cor:sample_complexity}.
\begin{proof}
    Let us take $\rho = \overline{\rho}(n+m)^{-\frac12}$, $\gamma = \min(\xi\rho,\overline{\gamma})$ and $q = n + m$. Then \cref{th:cvg_rate} holds: 
    \begin{equation*}
        \frac{1}{Tq}\sum_{t=0}^{T-1}\sum_{k=0}^{q-1}g^{t,k} \leq \frac4{Tq\gamma}\Gamma^0\enspace.
\end{equation*}
    with $\Gamma_0 = \calO(1)$.
    To get an $\varepsilon$-stationary solution, we set $T \geq \frac4{q\gamma}\Gamma^0\varepsilon^{-1}\vee 1 = \calO\left(\frac1{q\gamma\varepsilon}\vee1\right)$. One iteration has $\Theta(q) = \Theta(n+m)$ oracle complexity. As a consequence, the sample complexity to get an $\varepsilon$-stationary point is $\calO\left((n+m)^{\frac12}\epsilon^{-1}\vee(n+m)\right)$.
\end{proof}

\section{Lower bound for bilevel problems (proof of Theorem~\ref{th:lower_bound})}\label{app:sec:lower}
The proof of \cref{th:lower_bound} is an adaptation of the proof of \cite[Theorem 4.7]{Zhou2019} from single-level to bilevel problems. We build the outer function from the worst-case instance of  \cite[Theorem 4.7]{Zhou2019} and we add a bilevel component by using as inner function the function $G$ defined by $G(z,x) = \frac{\mu_G}2\|z-x\|^2$.
We start by introducing the different tools used in this proof.

\subsection{Preliminary results}
In what follows, we provide the building blocks of our worst-case instance.
The proof uses the following quadratic function presented by \citep{Nesterov2018}.
\begin{definition}\label{app:def:Q}
    Let $d\in\bbN_{>0}$, $\xi\in[0\,,\,+\infty)$ and $\zeta\leq 1$. We define $\mathbf{Q}(. ; \xi, d):\bbR^d\to\bbR$ by
    $$
    \mathbf{Q}(x;\xi,d) = \frac\xi2(x_1-1)^2 + \frac12\sum_{k=1}^{d-1}(x_{k+1} - x_k)^2 \enspace.
    $$
\end{definition}

\cref{app:prop:span} proposition comes directly from \cite[Proposition 3.5]{Zhou2019}. The first part of the proposition gives us the regularity of $\mathbf{Q}$. In the second part shows that a function defined as $\mathbf{Q}(U \times\,\cdot\, ; \xi, d) + \sum_{p=1}^q g(\langle u_p, \,\cdot\,\rangle)$ verifies the so-called "zero-chain property" \cite{Carmon2017lowerbound1}: if $Ux~\in~\Span(u_1,\dots,u_k)$, then we gain a non zero coordinate by calling the gradient $\nabla~[\mathbf{Q}~(U~\times~\,\cdot\,; \xi, d) + \sum_{p=1}^q g(\langle u_p, \,\cdot\,\rangle)](x)$. In other words, that makes us progress in the problem resolution.
\begin{proposition}\label{app:prop:span}
    For $d\in\bbN_{>0}$, $\xi\in[0\,,\,+\infty)$ and $\zeta\leq 1$. The following holds:
    \begin{enumerate}
        \item $\mathbf{Q}(\,\cdot\, ; \xi, d)$ is convex and 4-smooth.
        \item Let $q\in\bbN_{>0}$, $U = [u_1,\dots,u_d]^\top\in\bbR^{d\times q}$ such that $UU^\top = I$ and for $k\leq d$, $U^{(\leq k)}~=~ [u_1,\dots,u_k,0,\dots,0]^\top\in\bbR^{d\times q}$. Let $g:\bbR\to\bbR$ differentiable such that $g'(0) = 0$. Then for any $x\in\bbR^{q}$ such that $Ux= U^{(\leq k)}x$, then 
        $$
        \nabla \left[\mathbf{Q}(U \times\,\cdot\, ; \xi, d) + \sum_{p=1}^d g(\langle u_p, \,\cdot\,\rangle)\right](x) \in \Span(u_1,\dots,u_k, u_{k+1})\enspace.
        $$
    \end{enumerate}
\end{proposition}

\begin{proof}
    Let $x\in\bbR^q$ such that $Ux = U^{(\leq k)}x$. For $0\geq k\leq d$, we denote
    $$
        \bbR^{k,d} = \{v\in\bbR^d, v_{k+1} = \cdots = v_d = 0\}\enspace.
    $$
    Let us write $\mathbf{Q}(x;\xi,d) = \frac12 x^\top A x + b^\top x + c$ with 
    $$
        A = \begin{bmatrix}
                1+\xi & -1 & 0 & \cdots & 0 \\
                -1 & 2 & -1 & \ddots & \vdots \\
                0 & \ddots & \ddots & \ddots & 0 \\
                \vdots & \ddots & -1 & 2 & -1 \\
                0 & \cdots & 0 & -1 & 1 
            \end{bmatrix} 
            \in\bbR^{d\times d}\enspace,
    $$
    $b = \xi(1,0,\dots, 0)^\top$ and $c = \frac\xi2(1,0,\dots,0)^\top$.
    
    On the one hand it is known from \cite[Lemma 2.5.1]{Nesterov2018} that if $v\in\bbR^{k,d}$,
    \begin{align*}
        \nabla \mathbf{Q}(v ; \xi, d) \in\bbR^{k+1,d}
    \end{align*}
    As a consequence, 
    $$
        \nabla \mathbf{Q}(Ux ; \xi, d) = \nabla \mathbf{Q}(\underbrace{U^{(\leq k)}x}_{\in\bbR^{k,d}} ; \xi, d) \in\bbR^{k+1,d}
    $$
    and
    $$
        \nabla [\mathbf{Q}(U\times\,\cdot\, ; \xi, d)](x) = U^\top \nabla \mathbf{Q}(Ux ; \xi, d)\in\Span(u_1,\dots,u_{k+1})\enspace.
    $$
    On the other hand,
    $$
        \nabla\left[\sum_{p=1}^d g(\langle u_p, \,\cdot\,\rangle)\right](x) = \sum_{p=1}^d g'(\langle u_p, x\rangle)u_p = \sum_{p=1}^k g'(\langle u_p, x\rangle)u_p\in\Span(u_1,\dots,u_{k+1})\enspace.
    $$
    Thus
    $$
        \nabla \left[\mathbf{Q}(U \times\,\cdot\, ; \xi, d) + \sum_{p=1}^d g(\langle u_p, \,\cdot\,\rangle)\right](x) \in \Span(u_1,\dots,u_k, u_{k+1})\enspace.
    $$
\end{proof}

However, the function $\mathbf{Q}$ is convex. That is why we also use the function $\Gamma$ introduced in~\cite{Carmon2017lowerbound2}. As explained in \cite{Carmon2017lowerbound2}, this function is essential to lower bound the gradient of our worst wase instance.
\begin{definition}\label{app:def:Gamma}
    Let $d\in\bbN_{>0}$. We define $\Gamma(\,\cdot\,;d) : \bbR^{d+1} \to \bbR$ by
    $$
    \Gamma(x;d) = 120\sum_{k=1}^d\int_1^{x_k}\frac{t^2(t-1)}{1+t^2}\diff t\enspace. 
    $$
\end{definition}
An important property of $\Gamma$ shown in \cite[Lemma 2]{Carmon2017lowerbound2} is the smoothness of the function $\Gamma$.
\begin{proposition}
    There exists a constant $c>0$ such that $\Gamma(\,\cdot\,;d)$ is $c$-smooth.
\end{proposition}

Now we introduce the function $f_\mathrm{nc}$ we use to build our worst-case instance. This function comes from \cite[Definition 3.5]{Zhou2019}. It is the sum of the quadratic function defined by \ref{app:def:Q} and the nonconvex component given by \cref{app:def:Gamma}.
\begin{definition}
    For $\alpha>0$ and $d\in\bbN_{>0}$, $f_\mathrm{nc}(\,\cdot\,;\alpha,d):\bbR^{d+1}\to\bbR$ is defined a
    $$
    f_\mathrm{nc}(x;\alpha,d) = \mathbf{Q}(x;\sqrt{\alpha}, d+1) + \alpha\Gamma(x)\enspace.
    $$
\end{definition}

The essential properties of $f_{\mathrm{nc}}$ come from \cite[Lemmas 2, 3, 4]{Carmon2017lowerbound2}. The first part provides the regularity properties of $f_{\mathrm{nc}}$. The second part bounds the distance between $f_{\mathrm{nc}}(\,\cdot\,;\alpha,d)$ and the optimal value of the function. The third part will be key to the overall proof. In words, it states that as long $x\in\bbR^{d+1}$ has its two last components equal to zero, the norm of the gradient of $f_{\mathrm{nc}}(\,\cdot\,;\alpha,d)$ is higher than a constant controlled by $\alpha$. As a consequence, if $\alpha$ is properly chosen, as soon as $x_d = x_{d+1} = 0$, we are ensured that $\|\nabla f_\mathrm{nc}(x;\alpha,d)\|\geq \epsilon$.
\begin{proposition}\label{ass:prop:prop_fnc}
    For $\alpha\in[0\,,\,1]$, it holds
    \begin{enumerate}
        \item $-\alpha c\preceq\nabla^2 f_{\mathrm{nc}}\preceq 4+\alpha c$.
        \item $f_{\mathrm{nc}}(0;\alpha,d) - \inf_x f_\mathrm{nc}(x;\alpha,d) \leq \frac{\sqrt{\alpha}}2 + 10\alpha d$.
        \item For $x\in\bbR^{d+1}$ such that $x_d =x_{d+1} = 0$, $\|\nabla f_\mathrm{nc}(x;\alpha,d)\|\geq \frac{\alpha^{\frac34}}4$.
    \end{enumerate}
\end{proposition}

From now we denote
$$
\calO(a,b) = \{U \in\bbR^{a\times b}, UU^\top = I_a \}\enspace.
$$
The following Lemma adapted from \cite{Zhou2019} is fundamental for our lower bound proof.
\begin{lemma}\label{lemma:lower_bound_grad_h}
    Let $d,m\in\bbN_{>0}$ and $U\in\calO((d+1)m,(d+1)m)$. We denote $U = \begin{bmatrix}U^{(1)} \\ \vdots \\ U^{(m)}\end{bmatrix}$ with $U^{(i)}~\in~\calO(d+1, (d+1)m)$. Let $\{h_j\}_{j\in\setcomb{m}}$ with $h_j(x) = f_\mathrm{nc}(U^{(j)}x;\alpha,d)$ and $h = \frac1m \sum_{j=1}^m h_j$. Let $x\in\bbR^{(d+1)m}$ and $y^{(j)} = U^{(j)}x\in\bbR^{d+1}$. Let $\calI = \{j\in\setcomb{m}, y^{(i)}_d = y^{(i)}_{d+1} = 0\}$. Then it holds
    $$
    \|\nabla h(x)\|^2 \geq \frac{\alpha^{\frac32}|\calI|}{16m^2}\enspace.
    $$
\end{lemma}
\begin{proof}
    We have
    \begin{align*}
        \|\nabla h(x)\|^2 &= \left\|\frac1m\sum_{j=1}^d\nabla h_j(x)\right\|^2\\
        &=\left\|\frac1m\sum_{j=1}^m(U^{(j)})^\top\nabla f_{\mathrm{nc}}(U^{(j)}x;\alpha,d)\right\|^2\\
        &=\frac1{m^2}\left\|\sum_{j=1}^m(U^{(j)})^\top\nabla f_{\mathrm{nc}}(U^{(j)}x;\alpha,d)\right\|^2 \\
            &\qquad + \frac2{m^2}\sum_{\substack{j,l=1\\j\neq l}}^m\nabla f_{\mathrm{nc}}(U^{(j)}x;\alpha,d)^\top U^{(l)}(U^{(j)})^\top\nabla f_{\mathrm{nc}}(U^{(j)}x;\alpha,d)\\
        &= \frac1{m^2}\sum_{j=1}^m\left\|(U^{(j)})^\top\nabla f_{\mathrm{nc}}(U^{(j)}x;\alpha,d)\right\|^2
    \end{align*}
    where the last equality comes from the fact that for $j\neq l$, $U^{(l)}(U^{(j)})^\top=0$ since $U~\in~\calO~((d~+~1)m,(d+1)m)$. Now, using the third part of \cref{ass:prop:prop_fnc}, we get
    \begin{align*}
        \|\nabla h(x)\|^2 &\geq \frac1{m^2}\sum_{j\in\calI}\left\|(U^{(j)})^\top\nabla f_{\mathrm{nc}}(U^{(j)}x;\alpha,d)\right\|^2\\
        &\geq \frac1{m^2}\sum_{j\in\calI}\left\|\nabla f_{\mathrm{nc}}(U^{(j)}x;\alpha,d)\right\|^2\\
        &\geq \frac{\alpha^{\frac32}|\calI|}{16m^2}\enspace.
    \end{align*}
\end{proof}
\subsection{Main proof}
Now we are ready to prove \cref{th:lower_bound}.
\begin{proof}
     We consider $U\in\calO((T+1)m,(T+1)m)$ and we denote 
    $$
    U = \begin{bmatrix}U^{(1)} \\ \vdots \\ U^{(m)}\end{bmatrix}
    $$ 
    with $U^{(j)} = (u^{(j)}_1,\dots,u^{(j)}_{T+1})^\top~\in~\calO(T+1, (T+1)m)$.

    For $j\in\setcomb{m}$, we choose $\overline{F}_j:\bbR^{(T+1)m+(T+1)m}\to\bbR$ defined by
    $$
    \overline{F}_j(z,x) = f_\mathrm{nc}(U^{(j)}z;\alpha,T)
    $$
    and we set $\overline{F} = \frac1m\sum_{j=1}^m \overline{F}_j$.
    We also define for $i\in\setcomb{n}$ $\overline{G}_i(z,x) = \frac12\|z-x\|^2$, $\overline{G} = \frac1n\sum_{i=1}^n\overline{G}_i$, $\overline{z}^*(x) = \argmin_z \overline{G}(z,x)$ and $\overline{h}(x) = \overline{F}(\overline{z}^*(x),x)=f_\mathrm{nc}(U^{(j)}x;\alpha,T)$. By \cref{ass:prop:prop_fnc}, $\overline{F}_j$ is $4+\frac{\alpha c}m$ smooth, and $\overline{G}_i$ is $1$-smooth and $1$-strongly convex.

    We have
    $$
    \overline{h}(0) - \inf_x \overline{h}(x) \leq \sqrt{\alpha} + 10\alpha T\enspace.
    $$

    We finally consider $F_j(z,x) = \lambda_F \overline{F}_j(z/\beta, x/\beta)$, $G_i(z,x) = \lambda_G \overline{G}_i(z/\beta, x/\beta)$. As a consequence, we have $z^*(x)=\argmin G = \overline{z}^*(x)$ and $h(x) = F(z^*(x),x)$. We also consider a \emph{fixed} indices sequence $(i_t, j_t)$. We set
    \begin{align*}
        \alpha = \min\left\{1,\frac{m}c\right\},\quad 
        \lambda_F = \frac{160m\epsilon}{L^F_1\alpha^{3/2}},\quad
        \beta = \sqrt{5\lambda_F/L^F_1},\quad
        \lambda_G = \beta^2 \mu_G,\quad
        T = \frac{\Delta L^F_1}{1760m\epsilon}\sqrt{\alpha}\enspace.
    \end{align*}
    We can check that each $F_j$ is $L^F_1$-smooth, and each $G_i$ is $\mu_G$-strongly convex. Assuming $\epsilon~\leq~\Delta L^F_1\alpha/(1760m)$ ensures that $h(0) - \inf_x h(x) \leq \Delta$ (we can check that $h(0) = \lambda_F\overline{h}(0)$ and $\inf h = \lambda_F \inf \overline{h}$).
    
    Let us assume without loss of generality that the algorithm at initialization we have $z^0 = v^0 = x^0 = 0$ and consider $(z^t,v^t,x^t)$ the output of an algorithm with the known sequence $(i_t, j_t)$.

    Given our inner function and the fact that $\nabla_2 F(z,x) = 0$ for any $(z,x)\in\bbR^{(m+1)d+(m+1)d}$, we have 
    \begin{align}
        \label{app:eq:iter_z}z^{t+1} &\in \Span(z^0-x^0,\dots,z^t-x^t)\\
        \label{app:eq:iter_v}v^{t+1} &\in \Span(v^0 + \nabla_1 F_{j_0}(z^0,x^0),\dots,v^t + \nabla_1 F_{j_t}(z^t,x^t))\\
        \label{app:eq:iter_x}x^{t+1} &\in \Span(v^0,\dots,v^t)\enspace.
    \end{align}
    Since $v^0=0$, we have by \cref{app:eq:iter_v} $v^1\in\Span(\nabla_1 F_{j_0}(z^0,x^0))$ and by induction 
    $$
        v^{t+1}~\in~\Span(\nabla_1 F_{j_0}(z^0,x^0),\dots,\nabla_1 F_{j_t}(z^t,x^t))\enspace.
    $$
    Therefore, using \cref{app:eq:iter_x}, we have 
    $$
        x^{t+1} \in \Span(\nabla_1 F_{j_0}(z^0,x^0),\dots,\nabla_1 F_{j_t}(z^t,x^t))\enspace.
    $$
    Since $z^0=0$, by \cref{app:eq:iter_z}, $z^1\in\Span(x^0)$ and by induction 
    $$
        z^{t+1}\in\Span(x^0, \dots, x^t)\enspace.
    $$
    As a consequence, 
    \begin{align*}
        x^t\in\Span(\nabla_1 F_{j_0}(\Span(x^0),x^0),\dots,\nabla_1 F_{j_t}(\Span((x^s)_{s\leq t}),x^t))\enspace.
    \end{align*}

    Let us denote $y^{(j,t)}= U^{(j)}x^t$. Since $x^0 = 0$, $y^{(j_0,0)}=0$ and by the second part of \cref{app:prop:span}, $x^1\in\Span(u^{(j_0)}_{1})$. 
    
    Now we assume that for all $s\leq t$ we have 
    $$
        x^s~\in~\Span(u^{(j_0)}_{1},\dots, u^{(j_0)}_{s},\dots,u^{(j_{s-1})}_{1},\dots, u^{(j_{s-1})}_{s})\enspace.
    $$
    There exist scalars $\alpha_1, \dots, \alpha_r, \beta_{1,1},\beta_{2,1}, \beta_{2,2},\dots,\beta_{t,1},\dots,\beta_{t,t}$ such that
    \begin{align*}
        x^{t+1} = \sum_{r=1}^t\alpha_r\nabla_1 F_{j_r}\left(\sum_{s=1}^r \beta_{r,s}x^s, x^r\right)\enspace.
    \end{align*}
    Let $X^r = \sum_{s=1}^r \beta_{r,s}x^s$.
   For $r\in\{1,\dots, t\}$, we have by induction hypothesis
    $$
        X^r\in\Span(u^{(j_0)}_{1},\dots, u^{(j_0)}_{r},\dots,u^{(j_{r-1})}_{1},\dots, u^{(j_{r-1})}_{r})\enspace.
    $$
    By orthogonality, we have
    $$
    \Span(u^{(j_0)}_{1},\dots, u^{(j_0)}_{r},\dots,u^{(j_{r-1})}_{1},\dots, u^{(j_{r-1})}_{r})\perp\Span(u^{(j_r)}_{r+1},\dots,u^{(j_r)}_{T+1})\enspace.
    $$
    As a consequence
    $$
        U^{(j_r)}X^r = (\langle u^{(j_r)}_{1}, X^r\rangle,\dots, \langle u^{(j_r)}_{r}, X^r\rangle,0,\dots,0)\enspace.
    $$
    We can use \cref{app:prop:span} to say
    \begin{align*}
         \nabla_1 F_{j_r}\left(X^r, x^r\right)&\in\Span(u^{(j_r)}_{1},\dots, u^{(j_{r})}_{r+1})\subset \Span(u^{(j_0)}_{1},\dots, u^{(j_0)}_r, u^{(j_0)}_{r+1},\dots,u^{(j_{r})}_{1},\dots, u^{(j_{r})}_{r+1})\enspace.
    \end{align*}

    And we get finally
    $$
        x^{t+1}\in\Span(u^{(j_0)}_{1},\dots, u^{(j_0)}_{t}, u^{(j_0)}_{t+1},\dots,u^{(j_{t})}_{1},\dots, u^{(j_{t})}_{t+1})\enspace.
    $$
    
    By induction, for any $t$, we have 
    $$
        x^t\in\Span(\underbrace{u^{(j_0)}_{1},\dots,u^{(j_0)}_{t},\dots,u^{(j_t)}_{1},\dots,u^{(j_t)}_{t}}_{\text{at most $mt$ vectors}})
    $$
    and so 
    $$
        x^t\perp \Span((u^{(j)}_{1},\dots,u^{(j)}_{T+1})_{j\in\setcomb{m}\setminus\{j_0,\dots,j_t\}},(u^{(j)}_{t+1},\dots,u^{(j)}_{T+1})_{j\in\{j_0,\dots,j_t\}})\enspace.
    $$

    As a consequence, for $t\leq \frac{m}2 T$, let $\calI = \{j, y^{(j,t)}_T = y^{(j,t)}_{T+1} = 0\}$ with $y^{(j,t)} = U^{(j)}x^t$. Since $t\leq \frac{m}2 T$, we have $|\calI|\leq \frac{m}2$ and by \cref{lemma:lower_bound_grad_h}, we have 
    $$
    \|\nabla h(x^t)\|\geq \epsilon\enspace.
    $$
    If we define $T((x^t)_t, h) = \inf\{t\in\bbN, \|\nabla h(x^t)\|^2\leq \epsilon\}$, we just showed that for the fixed sequence $(i_t, j_t)$, we have
    $$
    T((x^t)_t, h)\geq \frac{m}2 T = \Omega(\sqrt{m}\epsilon^{-1})\enspace.
    $$
    The right-hand side being independent from the sequence $(i_t, j_t)$, for $t\leq \frac{m}2T$, we have
    $$
        \bbE[\|\nabla h(x^t)\|^2]>\epsilon
    $$
    where the expectation is taken over the random choice of $i_0, \dots, i_{t-1}, j_0,\dots,j_{t-1}$.
\end{proof}

\section{Details on the experiments}\label{app:sec:exp}
We performed the experiments with the Python package Benchopt \citep{Moreau2022}\footnote{The code of the benchmark is available at \url{https://github.com/benchopt/benchmark_bilevel} and the results are displayed in \url{ https://benchopt.github.io/results/benchmark_bilevel.html}.}. For each experiment, we use minibatches instead of single samples to estimate oracles because it is more efficient in practice. We use a batch size of 64 for the stochastic inner and outer oracles. \new{All the experiments were performed on processors AMD EPYC 7742 (4 CPUs/experiment).}

\new{\subsection{Benchmark on quadratics}
\label{app:sec:quadratic}
For this benchmark, we consider
$$
F(z,x) = \frac1m\sum_{j=1}^m F_j(z,x), \quad G(z,x) = \frac1n\sum_{i=1}^n G_i(z,x)\enspace.
$$
The functions $F_j$ and $G_i$ are defined as
\begin{align*}
    F_j(z,x) &= \frac12 z^\top A^{F_j}_z z + \frac12 x^\top A^{F_j}_x + xB^{F_j}z + (d^{F_j}_z)^\top z + (d^{F_j}_x)^\top x\\
    G_i(z,x) &= \frac12 z^\top A^{G_i}_z z + \frac12 x^\top A^{G_i}_x + xB^{G_i}z + (d^{G_i}_z)^\top z + (d^{G_i}_x)^\top x
\end{align*}
with $A^{F_j}_z,A^{G_i}_z\in\bbR^{p\times p}$, $A^{F_j}_x,A^{G_i}_x\in\bbR^{d\times d}$, $B^{F_j}, B^{G_i}\in\bbR^{d\times p}$, $d^{F_j}_z, d^{G_i}_z \in\bbR^p$ and $d^{F_j}_x, d^{G_i}_x \in\bbR^d$.
The vectors $d^{F_j}_x$, $d^{G_i}_x$ are drawn randomly according to a normal distribution $\calN(0, I_d)$. The vectors $d^{F_j}_z$, $d^{G_i}_z$ are drawn randomly according to a normal distribution $\calN(0, I_p)$. For the Hessian matrices with respect to $z$, we generate $A^{G_i}_z$ so that $\frac1n\sum_{i=1}^n A^{G_i}_z = A$ for a symetric positive definite matrix $A$ with spectrum in $[0.1, 1]$. To do so, we generate $x_i\sim\calN(0, I_p)$ and set $A^{G_i}_z = \sqrt{A}x_i(\sqrt{A}x_i)^\top$. We proceed similarly for $A^{F_j}_z$, $A^{G_i}_z$, $A^{F_j}_x$. For $B^{G_i}$, we want $\frac1n\sum_{i=1}^n B^{G_i} = B$ for a prescribed matrix $B\in\bbR^{d\times p}$ such that $\|B\| = 0.1$. Let $B = U\Sigma V^\top$ the singular values decomposition of $B$. To get $B^{G_i}$, we generate $x_i\sim\calN(0, I_p)$ and set $B^{G_i} = (V\Sigma x_i)(U x_i)^\top$. We proceed similarly for $B^{F_j}$. In our experiment, we take $n = 32768$ and $m=1024$.
To select the parameters of the solvers, we perform a grid search. More precisely, for each solver, we take the inner step size in the form of $\alpha t^{-a}$ where $a$ is the theoretical decrease rate of each solver and $\alpha$ is chosen in $\{0.01, 0.1\}$. The outer step size is taken as $\frac\alpha{r}t^{-b}$ where $b$ is the theoretical decrease rate and $r$ is chosen in $\{0.1,1,10,100\}$.
For the two-loops algorithms (\emph{i.e.} StocBiO, VRBO, AmIGO), the number of inner steps is set to $10$ after a manual search. In the methods implementing Neumann approximations (MRBO, VRBO, StocBiO), the number of terms in the Neumann series is also set to $10$ and the scaling parameter $\eta$ is set to $0.5$.
To get the fastest convergence, we keep for each solver the set of parameters that give the best decrease of $h$ on the 100 first epochs.  The period of full batch computation of VRBO and SRBA $q$ is parametrized as $q=a\frac{n+m}b$ where $b=64$ is the batch size and $a$ is chosen in $\{2^{-6},2^{-3},2^{-1}, 2^3, 2^6\}$.
For F$^2$SA, we take $\lambda_0 = 1$ and $\delta_t = \alpha t^{-\frac17}$ with $\alpha$ chosen in $\{0.01,0.1,1\}$.}

\subsection{Hyperparameter selection with IJCNN1}\label{app:sec:ijcnn1}
We solve a regularization selection problem for an $\ell^2$-regularized logistic regression problem. Here, we assume that we have a regularization parameter per feature. We have $n_\train=49,990$ training samples $(d^\train_i, y_i^\train)_{i\in\setcomb{n_\train}}$ and $n_\val=91,701$ validation samples $(d^\val_i, y_i^\val)_{i\in\setcomb{n_\train}}$ coming from the IJCNN1\footnote{\url{1https://www.csie.ntu.edu.tw/~cjlin/libsvmtools/datasets/binary.html}} dataset. Mathematically, it boils down to solve Problem \eqref{eq:bilevel_pb} with $F$ and $G$ given by
\begin{align*}
    F(\theta, \lambda) &= \frac1{n_\val}\sum_{j=1}^{n_\val}\varphi(y_j^\val\langle d_j^\val,\theta\rangle)\\
    G(\theta, \lambda) &= \frac1{n_\train}\sum_{i=1}^{n_\train}\varphi(y_i^\train\langle d_i^\train,\theta\rangle) + \frac12\sum_{k=1}^p\e^{\lambda_k}\theta_k^2
\end{align*}
where $\varphi$ is the logistic loss defined by $\varphi(u) = \log(1 + \e^{-u})$. The inner and outer step sizes are set to 0.05. 

To make our comparison, we select the parameters of each solver with an extensive grid search. More precisely, for each solver, we take the inner step size in the form of $\alpha t^{-a}$ where $a$ is the theoretical decrease rate of each solver and $\alpha$ is chosen in $\{2^{-5},2^{-4},2^{-3},2^{-2}\}$. The outer step size is taken as $\frac\alpha{r}t^{-b}$ where $b$ is the theoretical decrease rate and $r$ is chosen in $\{10^{-2},10^{-1.5},10^{-1},10^{-0.5}, 10^0\}$.
For the two-loops algorithms (\emph{i.e.} StocBiO, VRBO, AmIGO), the number of inner steps is set to $10$ after a manual search. In the methods implementing Neumann approximations (MRBO, VRBO, StocBiO), the number of terms in the Neumann series is also set to $10$ and the scaling parameter $\eta$ is set to $0.5$.
To get the fastest convergence, we keep for each solver the set of parameters that give the best decrease of $h$ on the 100 first epochs.  The period of full batch computation of VRBO and SRBA $q$ is parametrized as $q=a\frac{n+m}b$ where $b=64$ is the batch size and $a$ is chosen in $\{2^{-6},2^{-3},2^{-1}, 2^3, 2^6, 2^9\}$.
\new{For F$^2$SA, we take $\lambda_0 = 1$ and $\delta_t = \alpha t^{-\frac17}$ with $\alpha$ chosen in $\{0.01,0.1,1\}$.}

\section{Additional experiment: Datacleaning task}\label{app:sec:supp_exp}
We run an additional experiment. For each experiment, the parameters of the solvers are chosen by an extensive grid search. Then we select the curve that gives the best validation accuracy for each solver and finally plot the corresponding test error on \cref{app:fig:exp}.
\begin{figure}
    \centering
    \includegraphics[width=.48\textwidth]{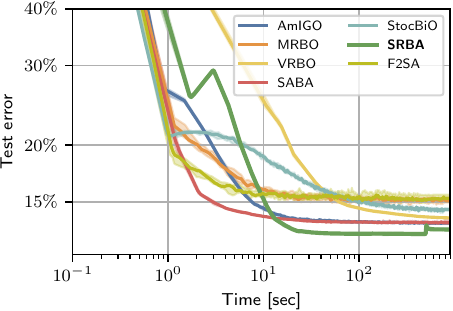}
    \caption{Comparison of stochastic bilevel solvers. Each solver is run on 10 random seeds and the lines show the median performances.  The shaded area corresponds to the performances between the 20\% and the 80\% percentiles. Test error on the datacleaning task with the MNIST dataset with a corruption rate $0.5$.}
    \label{app:fig:exp}
\end{figure}

The third experiment is the datacleaning task. It aims to train a multiclass classifier while having some training samples with noisy labels. On the one hand we have $n_\train = 20,000$ training labelled samples $(d_i^\train,y_i^\train)_{i\in\setcomb{n_\train}}$ with potentially corrupted labels with probability $p_c$ (in the experiments $p_c=0.5$). On the other hand, we have a validation set $(d_j^\val,y_j^\val)_{j\in\setcomb{n_\val}}$ of $n_\val = 5,000$ samples where all the samples are clean. We also have $10,000$ clean test samples. The datacleaning problem consists in learning a classifier on all these samples by giving less weight to corrupted labels. It can be cast as a bilevel optimization problem like \eqref{eq:bilevel_pb} where the function $F$ and $G$ are given by
\begin{align*}
    F(\theta, \lambda) &= \frac1{n_\val}\sum_{j=1}^{n_\val}\ell(\theta d_j^\val, y_j^\val)\\
    G(\theta, \lambda) &= \frac1{n_\train}\sum_{i=1}^{n_\train}\sigma(\lambda_i)\ell(\theta d_i^\train, y_i^\train) + C_r\|\theta\|^2
\end{align*}
where $\theta\in\bbR^{C\times p}$, $\lambda\in\bbR^{n_\train}$, $\ell$ is the cross entropy loss and $\sigma$ is the sigmoid function defined by $\sigma(\lambda) = \frac1{1+\e^{-\lambda}}\in (0\,,\,1]$.

We run this experiment on the MNIST dataset. We used $20,000$ training samples, $5,000$ validation samples, and $10,000$ test samples. The parameter $C_r$ is set to $0.2$ after a manual search to get the best performance. For the tuning of the step sizes of each method, we set $(\rho_t, \gamma_t) = (\alpha t^{-a}, \beta t^{-b})$ where $(a, b)$ are the rate provided by the analysis of each method, $\alpha$ is chosen among 4 values between $10^{-3}$ and $10^0$ spaced on a logarithmic scale. The scaling parameter $\beta$ is set to $\frac\beta{r}$ where $r$ is chosen among 6 values between $10^{-5}$ and $10^0$ spaced on a logarithmic scale. The other parameters are chosen in the same way as the IJCNN1 experiments (see \cref{app:sec:ijcnn1}).

We plot the test error on the \cref{app:fig:exp} (right). On the one hand, SRBA reaches the best final value. On the other hand, in terms of speed, it is the second fastest after SABA. The other methods are slower and reach a worse final accuracy.

\end{document}